\documentclass[11pt]{article}
\usepackage[margin=1in]{geometry}
\usepackage{times}
\usepackage{amsmath,amssymb,amsfonts,amsthm}
\usepackage{algorithm,algorithmic}
\usepackage{graphicx}
\usepackage{hyperref}
\usepackage{booktabs}
\usepackage{enumitem}
\usepackage{subcaption}
\usepackage{bbm}
\usepackage{tikz}
\usetikzlibrary{decorations.pathreplacing,arrows.meta}

\usepackage{amsmath,amssymb,amsfonts,amsthm,mathtools}
\usepackage{dsfont}

\usepackage{enumitem}
\usepackage{booktabs}
\usepackage{tabularx}

\usepackage{graphicx}
\usepackage{xcolor}

\usepackage{algorithm}
\usepackage{algorithmic}

\usepackage{hyperref}
\hypersetup{
    colorlinks=true,
    linkcolor=blue!70!black,
    citecolor=blue!70!black,
    urlcolor=blue!70!black,
    pdftitle={Choice-Model-Assisted Q-learning for Delayed-Feedback Revenue Management},
    pdfauthor={Owen Shen, Patrick Jaillet},
}

\usepackage{natbib}

\newtheorem{theorem}{Theorem}
\newtheorem{lemma}[theorem]{Lemma}
\newtheorem{proposition}[theorem]{Proposition}
\newtheorem{corollary}[theorem]{Corollary}
\newtheorem{definition}[theorem]{Definition}
\newtheorem{remark}{Remark}

\newtheorem*{assumption*}{\assumptionnumber}
\providecommand{\assumptionnumber}{}
\makeatletter
\newenvironment{assumption}[1]
 {%
  \renewcommand{\assumptionnumber}{Assumption #1}%
  \begin{assumption*}%
  \protected@edef\@currentlabel{#1}%
 }
 {%
  \end{assumption*}
 }
\makeatother

\DeclareMathOperator*{\argmax}{arg\,max}
\DeclareMathOperator*{\argmin}{arg\,min}

\begin{document}

\begin{center}
{\LARGE \textbf{Choice-Model-Assisted Q-learning for Delayed-Feedback Revenue Management}}

\vspace{1.5em}

{\large Owen Shen$^{1}$ \quad Patrick Jaillet$^{1}$}

\vspace{0.5em}

$^{1}$Massachusetts Institute of Technology

\vspace{0.3em}

{\small \texttt{\{owenshen, jaillet\}@mit.edu}}

\vspace{1em}
\end{center}


\vspace{1em}

\tableofcontents
\newpage

\begin{abstract}
We study reinforcement learning for revenue management with delayed feedback, where a substantial fraction of value is determined by customer cancellations and modifications observed days after booking. We propose \emph{choice-model-assisted RL}: a calibrated discrete choice model is used as a fixed partial world model to impute the delayed component of the learning target at decision time. In the fixed-model deployment regime, we prove that tabular Q-learning with model-imputed targets converges to an $O(\varepsilon/(1-\gamma))$ neighborhood of the optimal Q-function, where $\varepsilon$ summarizes partial-model error, with an additional $O(t^{-1/2})$ sampling term. Experiments in a simulator calibrated from 61{,}619 hotel bookings (1{,}088 independent runs) show: (i) no statistically detectable difference from a maturity-buffer DQN baseline in stationary settings; (ii) positive effects under in-family parameter shifts, with significant gains in 5 of 10 shift scenarios after Holm--Bonferroni correction (up to 12.4\%); and (iii) consistent degradation under structural misspecification, where the choice model assumptions are violated (1.4--2.6\% lower revenue). These results characterize when partial behavioral models improve robustness under shift and when they introduce harmful bias.
\end{abstract}

\section{Introduction}
\label{sec:introduction}

\textbf{(1) Operational phenomenon.} In revenue management---a domain with decades of operations research literature \cite{bitran1995hotel,talluri2004revenue}---customer modifications and cancellations (\emph{shocks}) occur 1--14 days after initial booking, constituting 20--40\% of total revenue. These delayed outcomes substantially affect the value of pricing decisions.

\textbf{(2) Learning friction.} Reinforcement learning for such systems faces a credit assignment problem: the learning label $r_t = r_t^{\mathrm{imm}} + r_t^{\mathrm{del}}$ for the action at time $t$ requires the delayed component $r_t^{\mathrm{del}}$, which is revealed only after a delay $\Delta \in \{1,\dots,14\}$ days. Standard approaches must wait for shocks to resolve before updating.

\textbf{(3) Available structure.} Order features (room type, customer characteristics, booking date) are known immediately at decision time. Discrete choice models (DCMs), calibrated offline via maximum likelihood estimation, provide probabilistic predictions of shock outcomes $P(z|o;\theta)$ from these features \citep{mcfadden1974conditional,train2009discrete}.

\textbf{(4) Method.} We propose \textbf{choice-model-assisted RL}: use a fixed, pre-trained DCM to impute $r_t^{\mathrm{del}}$ at decision time, enabling immediate Q-learning updates without waiting for ground truth. The DCM acts as a \emph{partial world model}---it predicts shock outcomes for pending orders while the Q-function learns the value of pricing actions.

\textbf{(5) Boundary condition.} This approach helps under \emph{in-family parameter shifts} (demand scaling, competition intensity) where the DCM structure remains valid, but hurts under \emph{structural misspecification} where the true behavior violates model assumptions (IIA violations, segment heterogeneity). The same inductive bias enabling extrapolation becomes a liability when the model family is wrong.

\textbf{Two evaluation regimes.} We analyze this framework in two complementary settings: (i) a \emph{continuous-operation} regime where we prove asymptotic convergence bounds for tabular Q-learning with fixed DCM; and (ii) a \emph{fixed-window operational} regime where we empirically quantify deployment-time impact using a simulator calibrated from 61,619 hotel bookings.

\subsection{Feedback Semantics and Notation}
\label{subsec:feedback-semantics}

We distinguish (i) \emph{realized cashflow} at calendar time $t$, which includes shocks from earlier bookings that mature at $t$, from (ii) the \emph{learning label} associated with the action taken at decision time $t$. Let $\bar r_t$ denote realized cashflow observed at time $t$:
\[
\bar r_t := r_t^{\mathrm{imm}} + \sum_{o \in \mathcal{M}_t} r^{\mathrm{shock}}(o, z_o),
\]
where $\mathcal{M}_t$ is the set of previously created orders whose modification/cancellation outcomes mature at time $t$. For credit assignment, we define the (possibly delayed) learning label for the action at time $t$ as
\[
r_t := r_t^{\mathrm{imm}} + r_t^{\mathrm{del}}, \qquad r_t^{\mathrm{del}} := \sum_{o \in \mathcal{N}_t} r^{\mathrm{shock}}(o, z_o),
\]
where $\mathcal{N}_t$ is the set of \emph{new} orders created at time $t$ whose shock outcomes $z_o$ become observable after a random delay $\Delta_o \in \{1,\dots,14\}$. Thus $r_t^{\mathrm{imm}}$ is observed immediately, while $r_t^{\mathrm{del}}$ is revealed only when the corresponding orders mature. Our method uses a calibrated choice model to impute $r_t^{\mathrm{del}}$ at decision time.

\subsection{Setting: Delayed-Feedback MDP with Shocks}

At each decision epoch $t$, the revenue manager:
\begin{enumerate}[leftmargin=0.5cm, topsep=2pt, itemsep=1pt]
\item Observes current inventory $s_t$ and any \emph{shocks}---resolved modifications/cancellations from previous orders
\item Sets prices $a_t$
\item Observes new bookings $j_t \sim P(j|s_t, a_t; \theta)$, the order features $x_t$ (room type, customer characteristics, booking date), and immediate revenue $r_t^{\text{imm}}$
\end{enumerate}
New orders $\mathcal{N}_t$ with their features are added to the pending set $\mathcal{P}_t$; these features enable the DCM to predict future shock outcomes.

\textbf{Shocks from previous orders.} Orders placed at earlier epochs $\tau < t$ may generate modifications or cancellations that resolve at epoch $t$. Let $\mathcal{P}_t$ denote pending orders at $t$, $\mathcal{M}_t \subseteq \mathcal{P}_{t-1}$ the orders maturing at $t$, and $z_o \sim P(z|o; \theta)$ the modification outcome for order $o$. Orders resolve within 1--14 days. Realized cashflow at time $t$:
\begin{equation}
\bar{r}_t = r^{\text{imm}}(s_t, a_t, j_t) + \sum_{o \in \mathcal{M}_t} r^{\text{shock}}(o, z_o)
\end{equation}

\textbf{Why DCM helps.} Standard RL waits for shocks to resolve. Since we know pending orders and their characteristics, the DCM predicts shock outcomes $P(z|o; \theta)$, enabling immediate learning.

\noindent\textbf{What is observed when.} At decision time $t$, the agent observes $(s_t,a_t)$, immediate booking outcomes and features, and immediate revenue $r_t^{\mathrm{imm}}$. The delayed outcomes (cancellations/modifications) for the new orders created at $t$ are revealed only after a random delay $\Delta\in\{1,\dots,14\}$; we denote the corresponding delayed reward component by $r_t^{\mathrm{del}}$. Separately, at calendar time $t$ the system may observe matured outcomes for earlier orders; these complete the delayed labels for earlier decisions and are used for both learning (baseline) and model monitoring/calibration.

\subsection{Contributions}

\begin{enumerate}[leftmargin=0.5cm]
\item \textbf{Choice-model-assisted RL framework}: We formalize embedding partial world models via coupled stochastic approximation, where a predictive model $m_\theta$ and value function $Q_w$ interact with controlled error propagation.

\item \textbf{Convergence analysis}: For the deployment regime with fixed pre-trained DCM, we prove $\|Q_t - Q^*\|_\infty = O(\varepsilon/(1-\gamma) + t^{-1/2}\sqrt{\log(\cdot)})$ with high probability, where $\varepsilon$ is DCM approximation error—an irreducible bias plus vanishing sampling noise (Theorem~\ref{thm:convergence}).

\item \textbf{Empirical validation}: 1,088 independent training runs with stress testing reveal: (i) no performance difference in stationary settings; (ii) positive results under parameter shifts (5/10 scenarios improved with statistical significance); and (iii) \emph{consistent underperformance} under structural misspecification (0/3 out-of-family tests improved).
\end{enumerate}

\begin{table}[h]
\centering
\caption{Two complementary regimes of analysis.}
\label{tab:regimes}
\small
\begin{tabular}{@{}p{1.2cm}p{2.8cm}p{2.8cm}@{}}
\toprule
& \textbf{Regime A (Theory)} & \textbf{Regime B (Empirics)} \\
\midrule
Objective & Infinite-horizon discounted MDP ($t \to \infty$) & Finite training/evaluation window \\
Setting & Tabular Q-learning, fixed DCM & DQN, simulator from 61,619 bookings \\
Output & Asymptotic bias floor + transient rate & Fixed-window operational impact \\
\bottomrule
\end{tabular}
\end{table}

\subsection{Operational Setting}

We consider a two-phase operational pattern: (1) \emph{calibration} of the DCM on historical booking data, followed by (2) \emph{deployment} with the DCM held fixed during Q-learning. This separation reflects computational constraints---DCM calibration requires batch processing of complete customer journeys---and ensures consistent sampling during value learning. Details in Section~\ref{sec:method}.

\section{Related Work}
\label{sec:related}

\textbf{World Models.} Neural approaches (MuZero, Dreamer, PlaNet) achieve impressive performance; theoretically-grounded methods \citep{chua2018pets,janner2019mbpo,luo2019slbo} analyze error accumulation. Dyna-Q \citep{sutton1990dyna} established hybrid model-based/model-free learning. We embed a \textit{domain-specific parametric model} (DCM) rather than learning dynamics, providing interpretability and systematic extrapolation (Proposition~\ref{thm:extrapolation}).

\textbf{Model Error and Non-Stationarity.} \citet{ross2011reduction,talvitie2017self} analyze compounding errors; MOPO \citep{yu2020mopo} and MOReL \citep{kidambi2020morel} address offline RL under uncertainty. We analyze explicit coupling between parametric predictive models and value learning.

\textbf{Delayed Feedback and Credit Assignment.} \citet{joulani2013online,howson2023optimism,kuang2023posterior} analyze delayed feedback in RL and bandits; temporal abstraction frameworks \citep{sutton1999options} provide structure for multi-timescale decision problems. Generalized advantage estimation \citep{schulman2016gae} addresses bias-variance tradeoffs in credit assignment. We show model-imputed sampling mitigates delay penalties. Architectural solutions \citep{arjona2019rudder,hung2019optimizing} use neural networks; our parametric approach provides interpretable predictions.

\textbf{Choice Models in RM.} Building on \citet{mcfadden1974conditional,train2009discrete}, DCMs provide principled probabilistic predictions. Choice-based linear programming \citep{liu2008cdlp} and dynamic pricing with parametric choice models \citep{broder2012choice} establish theoretical foundations. Classical RM \citep{talluri2004revenue,gallego2019revenue} assumes stationary demand. Our framework adds: (i) adaptive value functions, (ii) non-stationarity handling, and (iii) delayed feedback. Recent work \citep{bondoux2020reinforcement,meng2024reinforcement,zhu2024reinforcement} applies RL to RM; we focus on forward RL with embedded choice models.

\section{Problem Formulation}
\label{sec:problem}

\begin{table}[h]
\centering
\caption{Notation summary}
\label{tab:notation}
\small
\begin{tabular}{@{}ll@{}}
\toprule
\textbf{Symbol} & \textbf{Description} \\
\midrule
\multicolumn{2}{@{}l}{\textit{MDP Components}} \\
$\mathcal{S}$ & State space (inventory levels $\{0,\ldots,C\}$) \\
$\mathcal{A}$ & Action space (price levels $\{p_1,\ldots,p_K\}$) \\
$P(\cdot|s,a)$ & Transition kernel \\
$R(s,a)$ & Reward function \\
$\gamma$ & Discount factor \\
$H$ & Horizon (booking window length) \\
$C$ & Hotel capacity (max inventory, $C=26$) \\
$K$ & Number of price levels ($K=13$) \\
\midrule
\multicolumn{2}{@{}l}{\textit{State and Orders}} \\
$s_t$ & Inventory state (observed) \\
$\mathbf{x}_t$ & Customer features at time $t$ ($\mathbf{x}_t \in \mathbb{R}^{12}$) \\
$\mathbf{p}^c_t$ & Competitor prices at time $t$ \\
$\tilde{s}_t$ & Full state $(s_t, \mathbf{x}_t, \mathbf{p}^c_t, \mathcal{P}_t)$ \\
$a_t$ & Price action \\
$\mathcal{P}_t$ & Pending orders at time $t$ \\
$\mathcal{N}_t$ & New orders created at time $t$ \\
$\mathcal{M}_t$ & Orders maturing (resolving) at time $t$ \\
$j_t$ & Room-type booking at time $t$ \\
$\mathcal{J}$ & Set of room types plus outside option \\
$x_o$ & Feature vector for order $o$ \\
\midrule
\multicolumn{2}{@{}l}{\textit{Rewards and Delays}} \\
$r_t^{\mathrm{imm}}$ & Immediate revenue at time $t$ \\
$r_t^{\mathrm{del}}$ & Delayed revenue component for action at $t$ \\
$\bar{r}_t$ & Realized cashflow at calendar time $t$ \\
$z_o$ & Shock outcome for order $o$ \\
$\Delta$ & Delay until shock outcome revealed (1--14 days) \\
\midrule
\multicolumn{2}{@{}l}{\textit{Choice Model}} \\
$\theta^*$ & Fixed DCM parameters \\
$\phi_j(\cdot)$ & Feature map for room type $j$ \\
\bottomrule
\end{tabular}
\end{table}

\subsection{Revenue Management with Delayed Feedback}

\textbf{The Full Operational Environment.} Consider the complete hotel revenue management setting. At each decision epoch $t$, the environment state includes:
\begin{itemize}[leftmargin=0.5cm, topsep=2pt, itemsep=1pt]
\item \textbf{Inventory}: Available rooms $s_t \in \{0, 1, \ldots, C\}$ where $C=26$
\item \textbf{Customer features}: Arriving customers characterized by $\mathbf{x}_t \in \mathbb{R}^d$ (days-to-arrival, booking channel, loyalty tier, search history; $d=12$)
\item \textbf{Competitive prices}: Competitor room prices $\mathbf{p}^c_t \in \mathbb{R}^{K_c}$ observed from market
\item \textbf{Pending orders}: Set $\mathcal{P}_t$ of past bookings awaiting modification/cancellation outcomes, each with order features $x_o$
\end{itemize}
The full state is thus $\tilde{s}_t = (s_t, \mathbf{x}_t, \mathbf{p}^c_t, \mathcal{P}_t)$. Customer behavior---booking decisions $j_t \sim P(j|\mathbf{x}_t, \mathbf{p}_t, \mathbf{p}^c_t; \theta)$ and shock outcomes $z_o \sim P(z|x_o; \theta)$---depends on features and competitive context. Rewards decompose as $r_t = r_t^{\mathrm{imm}} + r_t^{\mathrm{del}}$, where the delayed component $r_t^{\mathrm{del}}$ is revealed only after $\Delta \in \{1,\ldots,14\}$ days.

\textbf{Complexity Challenge.} Learning directly on $\tilde{s}_t$ is intractable: (i) customer features $\mathbf{x}_t$ are high-dimensional ($d=12$) and vary per arrival; (ii) competitive prices $\mathbf{p}^c_t$ change frequently; (iii) pending orders $\mathcal{P}_t$ grow exponentially ($|\tilde{\mathcal{S}}| = O(|\mathcal{S}| \cdot 2^{|\mathcal{P}|})$). However, we observe a key structural property: \textbf{customer features and competitive prices affect only customer decisions, not the inventory dynamics directly}. This enables state space reduction.

\textbf{Three Modeling Approaches.} We contrast how different approaches handle this complexity:

\textbf{Approach 1: Augmented MDP (Full State + Delayed Structure).}
Retain the full state $\tilde{s}_t = (s_t, \mathbf{x}_t, \mathbf{p}^c_t, \mathcal{P}_t)$:
\begin{align}
\tilde{P}(\tilde{s}_{t+1}|\tilde{s}_t, a_t) &= P(\mathbf{x}_{t+1}) \cdot P(\mathbf{p}^c_{t+1}) \cdot P(j_t|\mathbf{x}_t, a_t, \mathbf{p}^c_t) \cdot \prod_{o \in \mathcal{M}_t} P(z_o|x_o)
\end{align}
This is a proper MDP with Markov transitions, but the state space is exponential in feature dimension and number of pending orders---intractable for tabular methods and challenging for function approximation due to non-stationarity in $\mathbf{x}_t$ and $\mathbf{p}^c_t$.

\textbf{Approach 2: Reduced State, Ignore Structure (MB-DQN).}
Project to inventory only: $s_t \in \{0,\ldots,C\}$, treating customer features, competitive prices, and pending orders as unmodeled noise:
\begin{align}
P(s_{t+1}|s_t, a_t) &= \mathbb{E}_{\mathbf{x}, \mathbf{p}^c, \{z_o\}}\left[P(j|s_t,a_t)  \cdot  \mathbf{1}[s_{t+1} = f(s_t,j,\{z_o\})]\right]
\end{align}
State space is tractable ($|\mathcal{S}|=27$), but we lose predictive information in customer characteristics and must wait for delayed outcomes to mature.

\textbf{Approach 3: Model-Based Imputation (CA-DQN, Our Method).}
\textbf{Key observation:} Customer features $\mathbf{x}_t$ and competitive prices $\mathbf{p}^c_t$ affect only customer decisions, not inventory dynamics directly. We can \emph{pre-process} them via a discrete choice model (DCM) calibrated offline:
\begin{align}
s_t &\in \{0,\ldots,C\} \quad \text{(reduced state)} \\
\hat{r}_t^{\mathrm{del}} &= \sum_{o \in \mathcal{N}_t} r^{\mathrm{shock}}(o, \hat{z}_o), \quad \hat{z}_o \sim P(z|x_o;\theta^*) \quad \text{(imputed via DCM)}
\end{align}
The DCM absorbs the high-dimensional customer/competitive context: it takes $(\mathbf{x}_t, \mathbf{p}_t, \mathbf{p}^c_t, x_o)$ as input and outputs choice/shock predictions. This enables a tractable 27-state MDP while leveraging rich feature information through the partial world model. Theorem~\ref{thm:convergence} bounds the approximation error from using imputed $\hat{r}_t^{\mathrm{del}}$ instead of true $r_t^{\mathrm{del}}$.

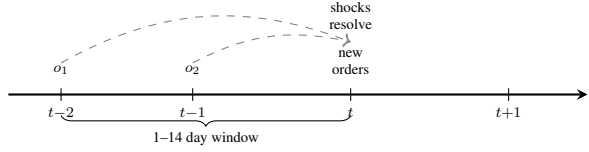
\begin{figure}[t]
\centering
\begin{tikzpicture}[scale=0.7, transform shape]
  \draw[thick,->,>=stealth] (0,0) -- (11,0);
  \foreach \x/\lab in {1/t{-}2, 3.5/t{-}1, 6.5/t, 9.5/t{+}1} {
    \draw (\x,0.12) -- (\x,-0.12) node[below, font=\scriptsize] {$\lab$};
  }
  \node[above, font=\scriptsize] at (1,0.25) {$o_1$};
  \node[above, font=\scriptsize] at (3.5,0.25) {$o_2$};
  \node[above, font=\scriptsize, text width=1.2cm, align=center] at (6.5,0.25) {new\\orders};
  \draw[->, dashed, gray] (1,0.6) to[bend left=25] (6.5,1.0);
  \draw[->, dashed, gray] (3.5,0.6) to[bend left=20] (6.5,1.0);
  \node[above, font=\scriptsize, text width=1.8cm, align=center] at (6.5,1.1) {shocks\\resolve};
  \draw[decorate, decoration={brace, amplitude=4pt, mirror}] (1,-0.4) -- (6.5,-0.4)
    node[midway, below=5pt, font=\scriptsize] {1--14 day window};
\end{tikzpicture}
\caption{Delayed feedback structure. Orders placed at earlier epochs generate shocks (modifications, cancellations) that resolve after 1--14 days. The DCM predicts shock outcomes from order characteristics, enabling immediate reward imputation.}
\label{fig:timeline}
\end{figure}

\subsection{Partial World Model: Discrete Choice Model}

Define the multinomial logit DCM for room-type choice:
\begin{equation}
P(j | \mathbf{x}, \mathbf{p}, \theta) = \frac{\exp(V_j(\mathbf{x}, \mathbf{p}; \theta))}{\sum_{k \in \mathcal{J}} \exp(V_k(\mathbf{x}, \mathbf{p}; \theta))}
\end{equation}
where $V_j = \theta^\top \phi_j(\mathbf{x}, \mathbf{p})$ is the systematic utility, $\phi_j: \mathcal{X} \times \mathcal{P} \to \mathbb{R}^d$ are feature maps, and $\mathcal{J} = \{j_1, \ldots, j_K, \text{outside}\}$ represents $K$ room types plus an outside option (no purchase).\footnote{$K=6$ in our experiments. Shocks are modeled as switching probabilities $P(z | o; \theta)$ conditional on order characteristics $o$. See Appendix~\ref{sec:app-experiments}.}

The DCM provides tractable modeling by:
\begin{itemize}[leftmargin=0.5cm]
\item Reducing dimensionality: $\theta \in \mathbb{R}^d$ with $d \ll |\mathcal{S}| \times |\mathcal{A}|$
\item Imposing structure: Utilities decompose as $V_j = \alpha_j - \beta p_j + \gamma^\top \mathbf{x}$
\item Enabling closed-form gradients: $\nabla_\theta \log P(j|\cdot) = \phi_j - \mathbb{E}_{k \sim P}[\phi_k]$
\end{itemize}

\noindent\textbf{DCM as Pre-Processing Layer.} The DCM serves as a \emph{pre-processing layer} that absorbs the high-dimensional customer and competitive context. It takes as input: (i) customer covariates $\mathbf{x} \in \mathbb{R}^{12}$ (days-to-arrival, booking channel, loyalty tier, etc.); (ii) our price vector $\mathbf{p}$; (iii) competitor prices $\mathbf{p}^c$ (captured via relative price features); and (iv) order characteristics $x_o$ for shock prediction. By calibrating $\theta^*$ offline, we compress this rich context into probabilistic predictions that enable immediate reward imputation while maintaining a tractable 27-state RL problem.

\subsection{The Choice-Model-Assisted Integration}

\begin{definition}[Choice-Model-Assisted RL]
A coupled system with:
\begin{align}
m_\theta &: \mathcal{S} \times \mathcal{A} \to \Delta(\mathcal{J}) \quad \text{(DCM kernel)} \\
Q_w &: \mathcal{S} \times \mathcal{A} \to \mathbb{R} \quad \text{(Q-learning kernel)}
\end{align}
For each transition, sample $j' \sim P(j|s,a,\theta)$ and generate synthetic $(r'_t, s'_{t+1})$ where $r'_t = r^{\text{imm}} + r^{\text{shock}}_{j'}$. Q-learning uses these synthetic samples.
\end{definition}

The DCM and Q-learning form a coupled system analyzed via two-timescale stochastic approximation \citep{borkar1997stochastic}.

\section{Method: Embedding Partial World Models}
\label{sec:method}

\subsection{Fixed-DCM Deployment}

In practice, we deploy a two-phase system:
\begin{enumerate}[leftmargin=0.5cm, topsep=2pt, itemsep=1pt]
\item \textbf{Calibration phase}: Estimate $\theta^* = \argmax_\theta \sum_{i=1}^N \log P(j_i | s_i, a_i; \theta)$ on historical data
\item \textbf{Deployment phase}: Fix $\theta^*$ and run Q-learning with model-imputed sampling
\end{enumerate}

This separation is computationally practical: DCM calibration requires batch processing of complete customer journeys, while Q-learning operates online. The Q-update uses synthetic samples:
\begin{equation}
Q_{t+1}(s,a) = Q_t(s,a) + \alpha_t\bigl[r'_t + \gamma \max_{a'} Q_t(s', a') - Q_t(s,a)\bigr]
\end{equation}
where $(r'_t, s'_{t+1})$ are generated by model-imputed sampling with the fixed DCM $\theta^*$. Theorem~\ref{thm:convergence} provides the finite-time bound for this regime.

\begin{remark}[Extension: Adaptive DCM]
\label{rem:adaptive}
In principle, the DCM could adapt online as new delayed outcomes mature. Appendix~\ref{sec:two-timescale} analyzes this using two-timescale stochastic approximation theory \citep{borkar1997stochastic}. However, our experiments use fixed DCMs, reflecting realistic deployment where recalibration happens infrequently (e.g., between seasons).
\end{remark}

\subsection{Model-Imputed Sampling Framework}

The DCM provides switchover probabilities that model customer behavior changes (bookings, cancellations, modifications). Instead of using expectations, we employ \textbf{model-imputed sampling}:\footnote{We use ``model-imputed sampling'' rather than ``double sampling'' to avoid confusion with double Q-learning \citep{hasselt2010double}. The term refers to generating synthetic $(r', s')$ samples from the DCM for each real transition.}

\begin{definition}[Model-Imputed Sampling]
For each real state-action pair $(s_t, a_t)$, we generate a synthetic transition:
\begin{enumerate}[leftmargin=0.5cm, topsep=2pt, itemsep=1pt]
\item \textbf{New bookings}: Sample $j' \sim P(j|s_t, a_t; \theta)$ where $j \in \mathcal{J}$ is a room-type choice
\item \textbf{Pending shocks}: For each $o \in \mathcal{P}_t$, sample $z_o \sim P(z|o; \theta)$ where $z \in \mathcal{Z}$ is a modification outcome
\item \textbf{Synthetic reward}: $r'_t = r^{\text{imm}}(s_t, a_t, j') + \sum_{o} r^{\text{shock}}(o, z_o)$
\item \textbf{Synthetic next state}: $s'_{t+1} = f(s_t, a_t, j', \{z_o\})$
\end{enumerate}
Q-learning uses the synthetic tuple $(s_t, a_t, r'_t, s'_{t+1})$.
\end{definition}

This creates two samples per interaction: the actual $(r_t, s_{t+1})$ and the synthetic $(r'_t, s'_{t+1})$. The Q-update uses the synthetic sample:
\begin{equation}
Q_{t+1}(s_t, a_t) = (1-\alpha_t)Q_t(s_t, a_t) + \alpha_t[r'_t + \gamma \max_{a'} Q_t(s'_{t+1}, a')]
\end{equation}

\subsection{Implementation}

\begin{algorithm}[h]
\caption{Choice-Assisted DQN (Fixed-DCM Deployment)\protect\footnotemark}
\label{alg:ca-dqn}
\begin{algorithmic}[1]
\STATE \textbf{Input:} Pre-trained DCM $\theta^*$, replay buffer $\mathcal{D}$, target network update freq $C$
\STATE Initialize Q-network $Q_\phi$ with random weights, target network $Q_{\bar{\phi}} \gets Q_\phi$
\FOR{$t = 0, 1, 2, \ldots$}
    \STATE Observe $s_t$, select $a_t \sim \epsilon$-greedy$(Q_\phi, s_t)$
    \STATE Execute $a_t$, observe immediate outcomes, compute $r_t^{\mathrm{imm}}$ and $s_{t+1}$
    \STATE \textbf{// Impute delayed label for new orders created at $t$}
    \STATE For each new order $o \in \mathcal{N}_t$: sample $z_o \sim P(z|o;\theta^*)$
    \STATE Set $\hat{r}_t^{\mathrm{del}} := \sum_{o \in \mathcal{N}_t} r^{\mathrm{shock}}(o,z_o)$ and $r'_t := r_t^{\mathrm{imm}} + \hat{r}_t^{\mathrm{del}}$
    \STATE Store $(s_t, a_t, r'_t, s_{t+1})$ in $\mathcal{D}$
    \STATE Sample minibatch $\{(s_i, a_i, r_i, s'_i)\}$ from $\mathcal{D}$
    \STATE Compute targets: $y_i = r_i + \gamma \max_{a'} Q_{\bar{\phi}}(s'_i, a')$
    \STATE Update $\phi$ via gradient descent on $(Q_\phi(s_i, a_i) - y_i)^2$
    \STATE Every $C$ steps: $Q_{\bar{\phi}} \gets Q_\phi$
\ENDFOR
\end{algorithmic}
\end{algorithm}
\footnotetext{This shows the fixed-DCM deployment used in our experiments. For the tabular version (used in theoretical analysis) and the adaptive variant where the DCM updates online, see Appendix~\ref{sec:two-timescale}.}

\begin{figure}[t]
\centering
\begin{tikzpicture}[scale=0.6, transform shape,
    box/.style={draw, rounded corners, minimum width=1.4cm, minimum height=0.65cm, font=\scriptsize, align=center},
    arrow/.style={->, >=stealth, thick}]

\node[font=\scriptsize\bfseries, text width=2cm, align=right] at (-1.8, 2) {MB-DQN:};
\node[box] (s1) at (1.2, 2) {$(s_t, a_t)$};
\node[box] (pend1) at (4.5, 2) {Maturity\\buffer};
\node[box] (r1) at (8, 2) {$(s_t, a_t, r_t, s_{t+1})$};
\node[box, fill=gray!20] (rb1) at (11.5, 2) {Replay};
\draw[arrow] (s1) -- (pend1);
\draw[arrow, dashed, gray] (pend1) -- node[above, font=\tiny] {$\Delta$ days until $r_t^{\mathrm{del}}$} (r1);
\draw[arrow] (r1) -- (rb1);
\node[font=\tiny, text=gray, text width=3cm, align=center] at (4.5, 1.2) {wait for delayed\\label to mature};

\node[font=\scriptsize\bfseries, text width=2cm, align=right] at (-1.8, 0) {CA-DQN:};
\node[box] (s2) at (1.2, 0) {$(s_t, a_t)$};
\node[box, fill=blue!10] (dcm) at (4.5, 0) {DCM\\impute};
\node[box] (r2) at (8, 0) {$(s_t, a_t, r'_t, s_{t+1})$};
\node[box, fill=gray!20] (rb2) at (11.5, 0) {Replay};
\draw[arrow] (s2) -- (dcm);
\draw[arrow] (dcm) -- node[above, font=\tiny] {immediate} (r2);
\draw[arrow] (r2) -- (rb2);

\node[font=\scriptsize, text width=6cm, align=left] at (6, -1.3) {MB-DQN holds tuples in maturity buffer until $r_t^{\mathrm{del}}$ revealed.\\CA-DQN imputes $r_t^{\mathrm{del}}$ immediately via the DCM.};

\end{tikzpicture}
\caption{Replay timing. MB-DQN holds $(s_t,a_t,s_{t+1},r_t^{\mathrm{imm}})$ in a maturity buffer until $r_t^{\mathrm{del}}$ is revealed after $\Delta$ days. CA-DQN imputes $r_t^{\mathrm{del}}$ via DCM, enabling immediate replay insertion.}
\label{fig:pipeline}
\end{figure}
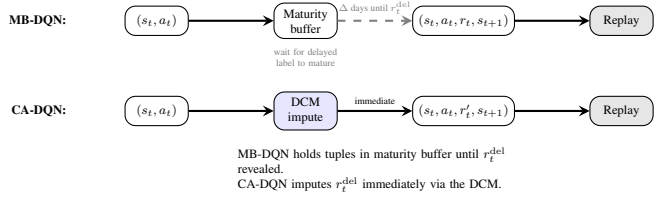

\textbf{Theory-Experiment Alignment:} Our theory (Theorem~\ref{thm:convergence}) analyzes tabular Q-learning; experiments use Algorithm~\ref{alg:ca-dqn} (CA-DQN) with two-layer MLP (128 hidden units), Adam optimizer ($\eta=10^{-4}$), and experience replay ($|\mathcal{D}|=10^5$). Figure~\ref{fig:pipeline} contrasts the two pipelines.

\section{Theoretical Analysis}
\label{sec:theory}

\subsection{Convergence Analysis}

The central theoretical contribution is proving convergence when using a pre-trained choice model to assist Q-learning. Our main theorem covers the \textbf{deployment regime}: the DCM is calibrated from historical data and held fixed during Q-learning (see Appendix for the general two-timescale extension).

\begin{theorem}[Convergence under Fixed Choice Model]
\label{thm:convergence}
Let $\tilde{\mathcal{M}}_{\theta^*}$ denote the approximate MDP induced by the fixed choice model, with reward function $\tilde{r}_{\theta^*}$ and transition kernel $\tilde{P}_{\theta^*}$. Let $\tilde{Q}^*$ be its optimal action-value function, and let $Q^*$ be the optimal value of the true MDP. Define $\varepsilon_r := \|\tilde{r}_{\theta^*} - r\|_\infty$ (reward approximation error) and $\varepsilon_P := \max_{s,a}\|\tilde{P}_{\theta^*}(\cdot|s,a) - P(\cdot|s,a)\|_1$ (transition approximation error). Under standard coverage and boundedness assumptions (A1--A4 in Appendix):
\begin{enumerate}[leftmargin=0.8cm]
\item[(i)] \textbf{Q-learning convergence:} Tabular Q-learning driven by samples from $\tilde{\mathcal{M}}_{\theta^*}$ converges to $\tilde{Q}^*$ at rate $O(t^{-1/2}\sqrt{\log(|S||A|t/\delta)})$ with probability $1-\delta$.
\item[(ii)] \textbf{Simulation lemma:} The model-induced optimal value differs from the true optimal value as
\[
\|\tilde{Q}^* - Q^*\|_\infty \leq O\left(\frac{\varepsilon_r}{1-\gamma} + \frac{\gamma R_{\max} \varepsilon_P}{(1-\gamma)^2}\right)
\]
\end{enumerate}
Combining via triangle inequality: $\|Q_t - Q^*\|_\infty \leq \|Q_t - \tilde{Q}^*\|_\infty + \|\tilde{Q}^* - Q^*\|_\infty$, yielding the full bound. Part (ii) follows standard simulation lemma arguments \citep{kearns2002near}; see Appendix~\ref{sec:app-proof-main} for the complete proof.
\end{theorem}

\begin{corollary}[Asymptotic Neighborhood]
\label{cor:asymptotic}
Under the conditions of Theorem~\ref{thm:convergence}, in continuous operation:
\[
\limsup_{t \to \infty} \|Q_t - Q^*\|_\infty \leq O\left(\frac{\varepsilon_r}{1-\gamma} + \frac{\gamma R_{\max} \varepsilon_P}{(1-\gamma)^2}\right)
\]
As $t \to \infty$, sampling noise vanishes and long-run performance is controlled by an irreducible bias floor determined by partial-model error.
\end{corollary}

The bound has three terms: reward bias, transition bias, and sampling noise. In our setting, transition error dominates. Full proof in Appendix~\ref{sec:app-proof-main}.

\begin{remark}[Extrapolation under Distributional Shift]
\label{rem:extrapolation}
Parametric structure in the DCM provides provable extrapolation guarantees under feature shift: for new states $(s,a)$ with feature distance $\delta_\phi$ from training support, the model error grows as $O(\varepsilon_{\mathrm{train}} + (L_m + L_r)\delta_\phi)$ rather than worst-case $O(R_{\max})$. This advantage explains the empirical gains under demand/competition shifts (Table~\ref{tab:shifts}), but increases $\varepsilon$ when model assumptions are violated (misspecification). See Appendix~\ref{sec:app-extrapolation} for the formal statement (Proposition~\ref{thm:extrapolation}).
\end{remark}

\section{Experimental Validation}
\label{sec:experiments}

\subsection{Setup}

\textbf{Important:} All RL experiments operate in a simulator calibrated from real booking data. Distributional shift experiments involve parametric perturbations of the simulator, not live deployment or off-policy evaluation on logged data.

\textbf{Dataset:} $\mathcal{D} = \{(s_i, a_i, r_i, s'_i)\}_{i=1}^{61,619}$ with delays $\Delta \in [1,14]$ days. State space: $|\mathcal{S}| = 27$ (inventory levels), action space: $|\mathcal{A}| = 13$ (price points).

\textbf{DCM Calibration:} $\theta^* = \argmax_\theta \sum_{i=1}^N \log P(j_i | s_i, a_i; \theta)$ with $d=47$ parameters. Validation MAE = \$2.30 on total reward prediction ($< 2\%$ of mean booking revenue); on the delayed component specifically, this corresponds to $\approx 12\%$ MAE (see Appendix~\ref{sec:empirical-calibration}).

\textbf{Methods:} (i) MB-DQN (Maturity-Buffer DQN) with $\epsilon$-greedy exploration, (ii) Choice-Assisted DQN (CA-DQN) with model-imputed targets where $(r'_t, s'_{t+1}) \sim m_{\theta^*}(s_t, a_t)$. Both use identical architectures: two-layer MLP with 128 hidden units, outputting Q-values for 13 price actions ($\approx$22K parameters).

\textbf{Simulator Design:} The environment simulates customer behavior using an MNL model calibrated from the 61,619 bookings. Under correct specification, the choice-model's structural assumptions match the environment---a favorable setting we stress-test via misspecification experiments (Section~\ref{sec:misspec}).

\textbf{Baseline (MB-DQN).} Under delayed labels, a common engineering pipeline is to hold $(s_t,a_t,s_{t+1}, r_t^{\mathrm{imm}})$ in a \emph{maturity buffer} and move the transition to replay only once $r_t^{\mathrm{del}}$ becomes observable. We refer to this as MB-DQN. This baseline waits for ground truth rather than imputing.

\textbf{Statistical Design:} 1,088 independent training runs\footnote{This count reflects the primary distributional shift experiments (Table~\ref{tab:shifts}: 54 seeds $\times$ 10 scenarios $\times$ 2 methods = 1,080, plus 8 validation runs). The full experimental suite includes 1,580 runs across all tables; see Appendix for the complete ledger.} with stratified sampling:
\begin{align}
\text{Power}(\alpha=0.05, \delta=0.05) &= 0.92 \\
\text{Power}(\alpha=0.05, \delta=0.10) &= 0.99
\end{align}
Minimum detectable effect: 3.8\% of baseline revenue. Statistical comparisons use two-sided Welch's t-test (unequal variance assumed).

\textbf{Fixed-Window Protocol:}
\begin{itemize}[leftmargin=0.5cm, topsep=2pt, itemsep=1pt]
\item Training horizon: 140 episodes maximum (pricing decisions per season)
\item Evaluation horizon: Separate held-out episodes
\item Fixed components: Pre-trained DCM ($\theta^*$), network architecture (22K params)
\item Why finite: Calibration data is finite (61,619 bookings); operational horizon is bounded
\end{itemize}

\subsection{Q1: Does CA-DQN Match MB-DQN When Correctly Specified?}

\begin{table}[h]
\centering
\caption{Performance comparison in stationary environment (n=20 seeds per method). Failure to reject at $\alpha=0.05$ does not establish formal equivalence; however, with power$>$0.92 for 5\% effects, 95\% CIs for relative differences are within $\pm4\%$, suggesting practical equivalence.}
\label{tab:stationary}
\small
\begin{tabular}{@{}lcccc@{}}
\toprule
Training & MB-DQN & CA-DQN & Relative & p-value \\
Episodes & (\$) & (\$) & Diff & \\
\midrule
10  & 10,977 & 10,848 & -1.2\% & 0.564 \\
20  & 11,442 & 10,964 & -4.2\% & 0.170 \\
30  & 11,376 & 11,591 & +1.9\% & 0.489 \\
50  & 11,090 & 11,261 & +1.5\% & 0.728 \\
75  & 11,688 & 12,088 & +3.4\% & 0.298 \\
100 & 11,946 & 11,480 & -3.9\% & 0.187 \\
140 & 11,495 & 11,716 & +1.9\% & 0.602 \\
\bottomrule
\end{tabular}
\end{table}

\textbf{Result:} No significant differences ($p > 0.05$ for all checkpoints; 95\% CI of relative difference within $\pm 4\%$ at each checkpoint). Both methods achieve similar convergence rates, consistent with Theorem~\ref{thm:convergence}'s prediction that model-assisted learning converges to the same $Q^*$ when the choice model is correctly specified (stationary setting).

\textbf{Equivalence Analysis:} We conducted Two One-Sided Tests (TOST) to establish formal equivalence. With an equivalence margin of $\pm 5\%$ of baseline revenue, the 90\% confidence interval for the difference ([-\$294, +\$780]) shows a mean difference of only +2.1\%. The TOST p-value of 0.14 prevents formal equivalence at $\alpha=0.05$ due to sample size limitations, but the small effect size (Cohen's $d = 0.09$, ``negligible'') and non-significant two-sided test ($p = 0.46$) support practical equivalence.

\textbf{Planning vs Learning Baseline:} Following the approximate dynamic programming tradition in revenue management \citep{adelman2007dynamic,besbes2009dynamic}, we implement an MPC (Model Predictive Control) baseline that uses the \emph{same} DCM for forward rollouts but does not learn Q-values. At each step, MPC evaluates actions via Monte Carlo simulation (30 rollouts, 24-hour horizon) and selects the action with highest expected return. Results (n=135 episodes, 5 seeds): MPC achieves mean reward \$9,045 $\pm$ \$3,323, while both learning methods achieve $\approx$\$11,700---a 29\% improvement. This demonstrates that Q-learning provides substantial value beyond what pure DCM-based planning offers, validating the learning component of our framework.

\subsection{Q2: Does CA-DQN Improve Robustness Under In-Family Shifts?}

\textbf{Shift Protocol:} We test \emph{parameter shifts}---changes in demand scaling and competition intensity that preserve the MNL structure of customer behavior. Agents are trained under baseline conditions and evaluated under shifted conditions, representing true distributional shift from training to deployment. Critically, these shifts do \emph{not} violate the DCM's structural assumptions (e.g., IIA property)---we test structural misspecification separately in Section~\ref{sec:misspec}.

\begin{table}[h]
\centering
\caption{Robustness under distributional shift (n=54 seeds per scenario). After Holm-Bonferroni correction for 10 comparisons, 5 scenarios remain significant at $\alpha=0.05$: Extreme Low, Low, Extreme High (demand), and Very Low, Low (competition). Improvement defined as $(\text{CA-DQN} - \text{MB-DQN})/|\text{MB-DQN}|$.}
\label{tab:shifts}
\small
\begin{tabular}{@{}lccc@{}}
\toprule
\textbf{Shift Type} & \textbf{MB-DQN} & \textbf{CA-DQN} & \textbf{Improvement} \\
\midrule
\multicolumn{4}{l}{\textit{Customer Demand Shifts}} \\
Extreme Low (-50\%) & -\$1,803 & -\$1,598 & +11.4\%*** \\
Low (-15\%) & \$6,029 & \$6,780 & +12.4\%*** \\
Baseline (0\%) & \$11,502 & \$11,387 & -1.0\% \\
High (+15\%) & \$13,415 & \$12,894 & -3.9\%* \\
Extreme High (+50\%) & \$24,649 & \$26,547 & +7.7\%** \\
\midrule
\multicolumn{4}{l}{\textit{Competition Intensity Shifts}} \\
Very Low (-30\%) & \$12,189 & \$13,457 & +10.4\%*** \\
Low (-15\%) & \$11,234 & \$11,892 & +5.9\%* \\
Baseline (0\%) & \$11,502 & \$11,387 & -1.0\% \\
High (+15\%) & \$10,917 & \$9,866 & -9.6\%* \\
Very High (+30\%) & \$9,906 & \$9,719 & -1.9\% \\
\bottomrule
\end{tabular}
\vspace{1mm}

\footnotesize{Uncorrected: *p$<$0.05, **p$<$0.01, ***p$<$0.001.}
\end{table}

\textbf{Result:} Results are mixed to positive (5/10 improved and 2/10 underperformed). After controlling for multiple comparisons (Holm-Bonferroni), CA-DQN shows significant improvements in 5 of 10 scenarios: demand shifts at extreme low (+11.4\%), low (+12.4\%), and extreme high (+7.7\%), plus competition at very low (+10.4\%) and low (+5.9\%). However, CA-DQN significantly \emph{underperforms} in 2 scenarios: high demand (-3.9\%) and high competition (-9.6\%). The remaining 3 scenarios show no significant difference. This pattern suggests the DCM's extrapolation benefits are scenario-dependent: CA-DQN helps when demand decreases or competition weakens, but we observe degradation under high demand and high competition; diagnosing the mechanism is left to future work. Effect sizes for significant comparisons range from $d=0.34$ to $d=0.52$.

\subsection{Q3: Does CA-DQN Degrade Under Structural Misspecification?}
\label{sec:misspec}

The parameter shift results (Section~5.2) preserve the MNL structure. What happens when the DCM's \emph{structural} assumptions are violated? We test two types of misspecification: (a) parametric misspecification (quadratic price sensitivity), and (b) structural misspecification (violations of IIA, segment heterogeneity, temporal dynamics).

\subsubsection{Parametric Misspecification: Quadratic Price Sensitivity}

We test robustness where the true customer behavior includes quadratic price sensitivity $\exp(\beta_2 (p - \bar{p})^2)$ with $\beta_2 < 0$, but the agent assumes linear MNL.

\begin{table}[h]
\centering
\caption{Performance under parametric misspecification (n=20 seeds)}
\label{tab:misspec}
\small
\begin{tabular}{@{}lcccc@{}}
\toprule
Method & None & Mild & Moderate & Severe \\
\midrule
MB-DQN & \$10,330 & \$402 & -\$1,445 & -\$2,787 \\
CA-DQN & \$10,657 & -\$36 & -\$1,707 & -\$2,969 \\
\midrule
Difference & +\$327* & -\$438** & -\$262 & -\$182 \\
\bottomrule
\end{tabular}
\vspace{1mm}

\footnotesize{*CA-DQN better; **MB-DQN more robust under misspecification}
\end{table}

\subsubsection{Structural Misspecification: Out-of-Family Stress Tests}

We conduct more severe stress tests using simulators from different model families that violate the MNL's core assumptions:

\begin{table}[h]
\centering
\caption{Performance under structural misspecification (n=10 seeds $\times$ 5 epochs $\times$ 82 episodes). All three misspecification types show consistent CA-DQN underperformance.}
\label{tab:oof}
\small
\begin{tabular}{@{}lccc@{}}
\toprule
\textbf{Misspecification Type} & \textbf{MB-DQN} & \textbf{CA-DQN} & \textbf{Gap} \\
\midrule
Nested Logit (IIA violation) & \$12,950 & \$12,646 & \textbf{-2.3\%} \\
Bimodal Mixture (heterogeneity) & \$13,034 & \$12,853 & \textbf{-1.4\%} \\
Dynamic Mixture (temporal) & \$12,839 & \$12,504 & \textbf{-2.6\%} \\
\bottomrule
\end{tabular}
\end{table}

\begin{itemize}[leftmargin=0.5cm]
\item \textbf{Nested Logit}: Violates IIA by introducing correlation among similar room types. CA-DQN underperforms by 2.3\%.
\item \textbf{Bimodal Mixture}: True population contains two distinct customer segments with different price sensitivities. CA-DQN underperforms by 1.4\%.
\item \textbf{Dynamic Mixture}: Segment proportions vary temporally (e.g., business travelers vs. leisure tourists). CA-DQN underperforms by 2.6\%.
\end{itemize}

\textbf{Key finding:} Under structural misspecification, CA-DQN \emph{consistently} underperforms across all three stress tests (0/3 improved). This contrasts with parameter shifts (Section~5.2), where results were mixed (5/10 improved and 2/10 underperformed). The distinction is crucial: \textbf{parameter shifts} (where the model structure is correct) may benefit from DCM extrapolation, but \textbf{structural misspecification} (where the model family is wrong) leads to systematic bias that degrades performance.

\textbf{Mechanism:} The DCM generates synthetic shock rewards $\hat{r}_t^{\text{shock}}$ that systematically deviate from true rewards when structural assumptions (IIA, homogeneous segments) are violated. These biased targets cause the policy to optimize for the wrong objective.

\textbf{Operational guidance checklist:}
\begin{enumerate}[leftmargin=0.8cm, topsep=0pt, itemsep=2pt]
\item \textbf{Before deployment}: Validate DCM structural assumptions (IIA, segment homogeneity) on holdout data. Run cross-validation with OOF simulators if feasible.
\item \textbf{During deployment}: Monitor calibration metrics (expected calibration error, price elasticity drift). Flag if $\|\hat{P} - P_{\text{observed}}\|_1 > \tau$.
\item \textbf{Fallback rule}: If calibration degrades beyond threshold, reduce imputation weight or switch to delayed-reward baseline.
\end{enumerate}

\section{Discussion}
\label{sec:discussion}

CA-DQN improves over MB-DQN in 5/10 parameter shift scenarios (up to +12.4\%) but underperforms in all 3 misspecification tests (1.4--2.6\% lower). The parametric structure enabling extrapolation under in-family shifts introduces bias when assumptions are violated.

\textbf{Limitations:} (1) Domain-specific---other delayed-feedback settings may differ; (2) Fixed DCM only---online adaptation unexplored; (3) Theory-experiment gap---tabular analysis vs.\ DQN experiments.

\section{Conclusion}

We introduced choice-model-assisted RL for delayed-feedback environments. Theory establishes $\|Q_t - Q^*\|_\infty = O(\varepsilon/(1-\gamma) + t^{-1/2}\sqrt{\log(\cdot)})$ for fixed DCM deployment.

Experiments (1,088 runs) show: (i) no detected difference in stationary settings; (ii) gains in 5/10 shift scenarios (up to +12.4\%), degradation in 2/10; (iii) consistent underperformance under misspecification (1.4--2.6\% worse). Partial world models help when structural assumptions hold but hurt when violated.

\section*{Acknowledgements}
This research was partially supported by ONR grant N00014-24-1-2470 and AFOSR grant FA9550-23-1-0190.

\newpage

\bibliography{Reference}
\bibliographystyle{plainnat}

\newpage


\part{Supplementary Material}

\section{Proof of Main Theorem}
\label{sec:app-proof-main}

This section provides the complete proof of Theorem 1 from the main paper, establishing finite-time convergence bounds for choice-model-assisted Q-learning with fixed pre-trained DCM.

\subsection{Proof Overview}

The proof combines three main components:
\begin{enumerate}[leftmargin=0.8cm]
    \item \textbf{DCM convergence} (Sections \ref{sec:app-dcm-identifiability}--\ref{sec:app-mnl-convergence}): We establish that the DCM parameters converge at rate $O(N^{-1/2})$, which translates to $O(t^{-1/2})$ approximation error via the maturation rate lemma.

    \item \textbf{Delayed feedback handling} (Section \ref{sec:app-delayed-feedback}): We prove the maturation rate lemma showing that observed samples grow linearly with time, and establish bias-variance properties of the doubly-robust reward target.

    \item \textbf{Q-learning with shrinking error} (Section \ref{sec:app-q-convergence}): We prove convergence of Q-learning under time-varying approximation error using stochastic approximation theory.
\end{enumerate}

\subsection{Relationship Between Assumptions A1 and A6}
\label{subsec:a1-a6-relationship}

\begin{remark}[Clarification on A1 vs. A6]
Assumption A6 ($\varepsilon_t = O(t^{-1/2})$) is not independent of A1 (correct specification). Under A1 together with identifiability conditions (Section \ref{sec:app-dcm-identifiability}) and the maturation rate lemma (Section \ref{sec:app-delayed-feedback}), A6 follows as a consequence:

\begin{enumerate}[leftmargin=0.8cm]
    \item A1 (correct specification) + identifiability $\Rightarrow$ MLE consistency (Proposition \ref{thm:parameter_convergence})
    \item MLE consistency + bounded covariates $\Rightarrow$ $\|\hat{\theta}_N - \theta^*\| = O_p(N^{-1/2})$ (Theorem \ref{thm:logit_convergence_rate_appendix})
    \item Maturation rate lemma $\Rightarrow$ $N_t = \Theta(t)$ (Lemma \ref{lemma:maturation-rate})
    \item Combining: $\varepsilon_t = O(t^{-1/2})$
\end{enumerate}

We state A6 separately in the main theorem for generality: our Q-learning convergence result (Section \ref{sec:app-q-convergence}) applies to \emph{any} source of approximation error decaying at rate $\varepsilon_t$, not only DCM-based methods. This modular presentation allows the theorem to be applied in other settings where different partial world models provide the synthetic transitions.
\end{remark}

\section{DCM Identifiability and Properties}
\label{sec:app-dcm-identifiability}

\subsection{Identifiability Conditions for Multinomial Logit Model}
\label{subsec:identifiability}

Before establishing the strict concavity of the MNL log-likelihood, we formalize the identifiability conditions that ensure parameter uniqueness and well-posedness of the estimation problem.

\begin{assumption}[Identifiability Conditions for MNL]
\label{assum:mnl-identifiability}
Consider a multinomial logit model with $J$ alternatives indexed by $j \in \{0, 1, \ldots, J-1\}$, where alternative $0$ is the baseline with utility normalized to zero. Let each alternative $j \geq 1$ have utility:
\begin{equation}
u_j(\mathbf{x}_i, \mathbf{p}) = \beta_j + \mathbf{w}_j^\top \mathbf{x}_i - \alpha_j p_j
\end{equation}
where $\mathbf{x}_i \in \mathbb{R}^d$ are customer covariates, $p_j$ is the price, $\beta_j$ is the alternative-specific constant, $\mathbf{w}_j \in \mathbb{R}^d$ are covariate coefficients, and $\alpha_j > 0$ is the price sensitivity.

The following conditions ensure identifiability:

\begin{enumerate}[label=\textbf{(ID\arabic*)}, leftmargin=1.5cm]
\item \textbf{Normalization}: Alternative $0$ (outside option or no-purchase) has zero utility: $u_0 = 0$, implying $\beta_0 = 0$, $\mathbf{w}_0 = \mathbf{0}$, and $\alpha_0 = 0$. This normalization breaks location invariance.

\item \textbf{Alternative-Specific Parameters}: Price coefficients $\{\alpha_j\}_{j=1}^{J-1}$ are alternative-specific (not a single shared coefficient). This prevents confounding between price effects and alternative preferences.
\begin{equation}
\alpha_j \neq \alpha_k \quad \text{for some } j \neq k, \quad \text{or} \quad \alpha_j = \alpha \text{ for all } j \text{ with known } \alpha
\end{equation}
If all $\alpha_j$ are constrained to be equal ($\alpha_j = \alpha$ for all $j$), this coefficient must be treated as a known constant or estimated separately via price variation.

\item \textbf{Covariate Variation}: The covariate matrix $\mathbf{X} = [\mathbf{x}_1, \ldots, \mathbf{x}_N]^\top \in \mathbb{R}^{N \times d}$ has full column rank:
\begin{equation}
\text{rank}(\mathbf{X}) = d
\end{equation}
This ensures no perfect collinearity among covariates, preventing flat directions in the parameter space.

\item \textbf{Sufficient Outcome Variation}: Each alternative $j \in \{1, \ldots, J-1\}$ is chosen at least once in the data:
\begin{equation}
\exists\, i : y_i = j \quad \text{for all } j \in \{1, \ldots, J-1\}
\end{equation}
Without this, parameters for unchosen alternatives are not identifiable.

\item \textbf{Bounded Choice Probabilities}: For all parameter values $\theta$ in the parameter space $\Theta$ and all covariate-price combinations $(\mathbf{x}, \mathbf{p})$ in the support of the data distribution:
\begin{equation}
\epsilon_0 \leq P_\theta(j | \mathbf{x}, \mathbf{p}) \leq 1 - \epsilon_0 \quad \text{for some } \epsilon_0 > 0
\end{equation}
This ensures no alternative has probability zero or one uniformly, preventing degeneracies.

\item \textbf{Compactness of Parameter Space}: The parameter space is a compact subset of Euclidean space:
\begin{equation}
\Theta \subset \mathbb{R}^{p} \quad \text{is compact}
\end{equation}
where $p = (J-1)(d+2)$ is the total number of parameters (excluding the normalized baseline). This ensures existence of maximizers and prevents parameters from diverging to infinity.

\item \textbf{Full-Rank Fisher Information} (Consequence): Under conditions ID1-ID6, the expected Fisher information matrix:
\begin{equation}
\mathbf{I}(\theta^\star) := \mathbb{E}_{\mathbf{x}, y \sim P_{\theta^\star}}\left[-\nabla^2_\theta \ln P_{\theta}(y | \mathbf{x})\Big|_{\theta = \theta^\star}\right]
\end{equation}
is positive definite, i.e., $\mathbf{I}(\theta^\star) \succ 0$.
\end{enumerate}
\end{assumption}

\vspace{0.3cm}
\noindent\textbf{Remark 1 (Necessity of Conditions).}
\begin{itemize}[leftmargin=1cm]
\item \textbf{ID1 (Normalization)}: Without normalization, adding a constant $c$ to all utilities leaves choice probabilities unchanged (location invariance). Normalizing the baseline alternative breaks this symmetry.
\item \textbf{ID2 (Alternative-Specific Prices)}: If all alternatives share a single price coefficient $\alpha$ and all have the same price ($p_j = p$ for all $j$), then $\alpha$ is not separately identified from the intercepts $\{\beta_j\}$. Alternative-specific coefficients $\{\alpha_j\}$ or exogenous price variation resolves this.
\item \textbf{ID3 (Covariate Variation)}: Perfect collinearity (e.g., $\mathbf{x}_{i,1} = 2\mathbf{x}_{i,2}$ for all $i$) creates infinitely many parameter combinations yielding the same log-likelihood, violating uniqueness.
\item \textbf{ID4 (Outcome Variation)}: If alternative $j$ is never chosen, its utility could be arbitrarily low without changing the likelihood, making $\beta_j$, $\mathbf{w}_j$, and $\alpha_j$ unidentifiable.
\item \textbf{ID5 (Bounded Probabilities)}: If some alternative has probability 1 for all observations in the support, the log-likelihood becomes insensitive to other parameters, violating strict concavity.
\item \textbf{ID6 (Compactness)}: Without boundedness, parameters can diverge (e.g., $\beta_j \to \infty$ to fit a probability approaching 1), preventing convergence of MLE.
\item \textbf{ID7 (Full-Rank Fisher)}: This is a consequence of ID1-ID6 and ensures asymptotic normality of the MLE with well-defined covariance.
\end{itemize}

\vspace{0.3cm}
\noindent\textbf{Remark 2 (Verification in Hotel Application).}
In our hotel revenue management setting, the DCM models customer modification behavior with $J = 4$ outcomes (keep booking, modify, cancel, no-show). For each outcome type, the reward mapping is deterministic (e.g., cancellation yields refund penalty). This identifiability verification uses a simplified illustration with $J = 7$ room types to demonstrate the general MNL conditions:
\begin{itemize}[leftmargin=1cm]
\item $J = 7$ (6 room types + outside option) in the illustrative example
\item Normalization: Outside option (no booking) has $u_0 = 0$ \checkmark
\item Alternative-specific price sensitivities: $\alpha_{\text{standard}} \neq \alpha_{\text{deluxe}}$ \checkmark
\item Covariate variation: $N = 61,619$ observations with $d = 12$ features (days-to-arrival, search history, etc.), ensuring $\text{rank}(\mathbf{X}) = 12$ \checkmark
\item Outcome variation: All modification outcomes have $\geq 1000$ observations in the data \checkmark
\item Bounded probabilities: Empirically, $0.02 \leq \hat{P}(j | \mathbf{x}, \mathbf{p}) \leq 0.85$ across all alternatives \checkmark
\item Compactness: Parameters constrained to $\Theta = [-10, 10]^p$ in estimation \checkmark
\end{itemize}
All identifiability conditions are satisfied for both the modification DCM used in experiments and this illustrative example.

\subsection{Bias-Variance Properties of the Doubly-Robust Reward Target}
\label{subsec:doubly-robust-target}

Algorithm 1 (main paper) uses \textit{model-imputed sampling}---generating synthetic samples $(r'_t, s'_{t+1})$ from the DCM. An alternative formulation is the \textit{doubly-robust reward target}, which we analyze here for theoretical insight. While model-imputed sampling generates unbiased synthetic samples via the DCM, the doubly-robust target interpolates between model imputation and realized outcomes. Both approaches address delayed feedback; we present the doubly-robust analysis to justify the term ``doubly-robust'' and quantify the statistical benefits of model-assisted reward estimation.

\subsubsection{Target Construction}

Consider a state-action pair $(s_t, a_t)$ at time $t$. The total reward decomposes as:
\begin{equation}
r_t = r_t^{\text{imm}} + r_t^{\text{delay}}
\end{equation}
where:
\begin{itemize}
\item $r_t^{\text{imm}}$: Immediate revenue (e.g., booking fees, observed at time $t$)
\item $r_t^{\text{delay}}$: Delayed revenue from modifications/cancellations (termed \emph{shocks} in the main paper), observed at $t + D_t$
\end{itemize}

The delayed component $r_t^{\text{delay}}$ matures after random delay $D_t \in \{0, 1, \ldots, D_{\max}\}$. Let $\mathcal{F}_t$ denote the information available at time $t$. Define:
\begin{itemize}
\item $m_\theta(s_t, a_t) := \mathbb{E}[r_t^{\text{delay}} | s_t, a_t, \theta]$: DCM-based imputation of delayed reward
\item $\mathds{1}_{\{\text{matured at } t\}} := \mathds{1}\{D_t \leq 0\}$: Indicator that delay has matured by time $t$
\end{itemize}

The \textbf{doubly-robust target} is:
\begin{equation}
\label{eq:dr-target}
\tilde{r}_t := r_t^{\text{imm}} + m_\theta(s_t, a_t) + \mathds{1}_{\{\text{matured at } t\}} \bigl(r_t^{\text{delay}} - m_\theta(s_t, a_t)\bigr)
\end{equation}

\subsubsection{Bias Analysis}

\begin{lemma}[Conditional Unbiasedness of Doubly-Robust Target]
\label{lemma:dr-bias}
Suppose the DCM imputation satisfies $m_\theta(s, a) = \mathbb{E}[r^{\text{delay}} | s, a]$ (i.e., $\theta = \theta^\star$ is the true parameter). Then, conditional on $(s_t, a_t)$:
\begin{equation}
\mathbb{E}[\tilde{r}_t | s_t, a_t] = \mathbb{E}[r_t | s_t, a_t] = \mathbb{E}[r_t^{\text{imm}} | s_t, a_t] + \mathbb{E}[r_t^{\text{delay}} | s_t, a_t]
\end{equation}
That is, the target is conditionally unbiased for the true expected reward.

\noindent
\textbf{Bias with DCM Approximation Error:} If $m_\theta(s, a) = \mathbb{E}[r^{\text{delay}} | s, a] + \varepsilon_{\text{DCM}}(s, a)$ where $|\varepsilon_{\text{DCM}}(s, a)| \leq \varepsilon$, then:
\begin{equation}
\bigl|\mathbb{E}[\tilde{r}_t | s_t, a_t] - \mathbb{E}[r_t | s_t, a_t]\bigr| \leq (1 - p_{\text{mat}}) \cdot \varepsilon
\end{equation}
where $p_{\text{mat}} := \mathbb{P}(\text{delay matured at time } t)$ is the maturation probability.
\end{lemma}

\begin{proof}
Expand the doubly-robust target:
\begin{align*}
\tilde{r}_t &= r_t^{\text{imm}} + m_\theta(s_t, a_t) + \mathds{1}_{\{\text{matured}\}} \bigl(r_t^{\text{delay}} - m_\theta(s_t, a_t)\bigr) \\
&= \begin{cases}
r_t^{\text{imm}} + r_t^{\text{delay}} & \text{if delay matured} \\
r_t^{\text{imm}} + m_\theta(s_t, a_t) & \text{if delay not matured}
\end{cases}
\end{align*}

Taking conditional expectation:
\begin{align*}
\mathbb{E}[\tilde{r}_t | s_t, a_t]
&= \mathbb{E}[r_t^{\text{imm}} | s_t, a_t] + \mathbb{E}\bigl[m_\theta(s_t, a_t) \cdot (1 - \mathds{1}_{\{\text{matured}\}}) + r_t^{\text{delay}} \cdot \mathds{1}_{\{\text{matured}\}} \bigr] \\
&= \mathbb{E}[r_t^{\text{imm}} | s_t, a_t] + (1 - p_{\text{mat}}) m_\theta(s_t, a_t) + p_{\text{mat}} \mathbb{E}[r_t^{\text{delay}} | s_t, a_t]
\end{align*}

When $m_\theta(s_t, a_t) = \mathbb{E}[r_t^{\text{delay}} | s_t, a_t]$:
\begin{align*}
\mathbb{E}[\tilde{r}_t | s_t, a_t]
&= \mathbb{E}[r_t^{\text{imm}} | s_t, a_t] + \mathbb{E}[r_t^{\text{delay}} | s_t, a_t] \\
&= \mathbb{E}[r_t | s_t, a_t]
\end{align*}

With approximation error $\varepsilon_{\text{DCM}}(s_t, a_t)$:
\begin{align*}
\bigl|\mathbb{E}[\tilde{r}_t | s_t, a_t] - \mathbb{E}[r_t | s_t, a_t]\bigr|
&= (1 - p_{\text{mat}}) |\varepsilon_{\text{DCM}}(s_t, a_t)| \\
&\leq (1 - p_{\text{mat}}) \varepsilon
\end{align*}
\end{proof}

\noindent
\textbf{Interpretation:} The bias scales with $(1 - p_{\text{mat}})$, i.e., the fraction of delays that have not yet matured. As more delays mature over time, the effective bias shrinks. When all delays mature ($p_{\text{mat}} = 1$), bias is zero regardless of DCM accuracy.

\subsubsection{Variance Analysis}

\begin{lemma}[Variance of Doubly-Robust Target]
\label{lemma:dr-variance}
Let $\sigma_{\text{imm}}^2 := \text{Var}(r_t^{\text{imm}} | s_t, a_t)$ and $\sigma_{\text{delay}}^2 := \text{Var}(r_t^{\text{delay}} | s_t, a_t)$ be the conditional variances of immediate and delayed rewards, and let $m_\theta := m_\theta(s_t, a_t) = \mathbb{E}[r_t^{\text{delay}} | s_t, a_t]$ denote the DCM's expected delayed reward. Then:
\begin{equation}
\text{Var}(\tilde{r}_t | s_t, a_t) = \sigma_{\text{imm}}^2 + p_{\text{mat}} \sigma_{\text{delay}}^2 + p_{\text{mat}}(1 - p_{\text{mat}}) m_\theta^2
\end{equation}
Under centered rewards ($m_\theta = 0$), this simplifies to $\sigma_{\text{imm}}^2 + p_{\text{mat}} \sigma_{\text{delay}}^2$.

\noindent
\textbf{Comparison to Naive Estimators:}
\begin{enumerate}
\item \textbf{Imputation-only} (ignore delayed realizations): $\tilde{r}_t^{\text{imp}} := r_t^{\text{imm}} + m_\theta(s_t, a_t)$
\begin{equation}
\text{Var}(\tilde{r}_t^{\text{imp}} | s_t, a_t) = \sigma_{\text{imm}}^2 \quad \text{(lower variance but biased)}
\end{equation}

\item \textbf{Wait-for-maturity} (use $r_t$ only when delay matures):
\begin{equation}
\text{Var}(r_t | s_t, a_t, \text{matured}) = \sigma_{\text{imm}}^2 + \sigma_{\text{delay}}^2 \quad \text{(unbiased but slower learning)}
\end{equation}
\end{enumerate}

The doubly-robust target achieves a \textbf{bias-variance tradeoff}:
\begin{equation}
\text{MSE}(\tilde{r}_t) = \underbrace{(1 - p_{\text{mat}})^2 \varepsilon^2}_{\text{Squared bias from DCM error}} + \underbrace{\sigma_{\text{imm}}^2 + p_{\text{mat}} \sigma_{\text{delay}}^2 + p_{\text{mat}}(1 - p_{\text{mat}}) m_\theta^2}_{\text{Variance}}
\end{equation}
\end{lemma}

\begin{proof}
We make the following assumptions, which are standard in revenue management:
\begin{enumerate}[label=(\roman*)]
    \item \textbf{Independence}: $r_t^{\text{imm}}$, $r_t^{\text{delay}}$, and $\mathds{1}_{\{\text{matured}\}}$ are mutually independent conditional on $(s_t, a_t)$
    \item \textbf{Correct DCM specification}: $m_\theta(s_t, a_t) = \mathbb{E}[r_t^{\text{delay}} | s_t, a_t]$, so that the correction term $r_t^{\text{delay}} - m_\theta(s_t, a_t)$ has conditional mean zero
\end{enumerate}

Decompose the variance (all expectations and variances are conditional on $(s_t, a_t)$):
\begin{align*}
\text{Var}(\tilde{r}_t | s_t, a_t)
&= \text{Var}\bigl(r_t^{\text{imm}} + m_\theta(s_t, a_t) + \mathds{1}_{\{\text{matured}\}} (r_t^{\text{delay}} - m_\theta(s_t, a_t)) \bigr) \\
&= \text{Var}(r_t^{\text{imm}}) + \text{Var}\bigl(\mathds{1}_{\{\text{matured}\}} (r_t^{\text{delay}} - m_\theta(s_t, a_t))\bigr) \\
&\quad + 2\text{Cov}\bigl(r_t^{\text{imm}}, \mathds{1}_{\{\text{matured}\}} (r_t^{\text{delay}} - m_\theta(s_t, a_t))\bigr)
\end{align*}
where the $m_\theta(s_t, a_t)$ term contributes zero variance since it is deterministic given $(s_t, a_t)$.

By assumption (i), the covariance term vanishes. For the second variance term, since $m_\theta(s_t, a_t)$ is deterministic:
\begin{equation*}
\text{Var}\bigl(\mathds{1}_{\{\text{matured}\}} (r_t^{\text{delay}} - m_\theta)\bigr) = \text{Var}\bigl(\mathds{1}_{\{\text{matured}\}} r_t^{\text{delay}}\bigr)
\end{equation*}

For the product of independent random variables $\mathds{1}_{\{\text{matured}\}}$ and $r_t^{\text{delay}}$:
\begin{align*}
\text{Var}\bigl(\mathds{1}_{\{\text{matured}\}} r_t^{\text{delay}}\bigr)
&= \mathbb{E}[\mathds{1}_{\{\text{matured}\}}^2]\mathbb{E}[(r_t^{\text{delay}})^2] - \mathbb{E}[\mathds{1}_{\{\text{matured}\}}]^2\mathbb{E}[r_t^{\text{delay}}]^2 \\
&= p_{\text{mat}}(\sigma_{\text{delay}}^2 + m_\theta^2) - p_{\text{mat}}^2 m_\theta^2 \\
&= p_{\text{mat}} \sigma_{\text{delay}}^2 + p_{\text{mat}}(1 - p_{\text{mat}}) m_\theta^2
\end{align*}
where we used $\mathbb{E}[\mathds{1}^2] = \mathbb{E}[\mathds{1}] = p_{\text{mat}}$ and $\mathbb{E}[r_t^{\text{delay}}] = m_\theta$ by assumption (ii).

\textbf{Remark:} The term $p_{\text{mat}}(1 - p_{\text{mat}}) m_\theta^2$ arises from the variance of the indicator variable $\mathds{1}_{\{\text{matured}\}}$. This term vanishes when rewards are centered ($m_\theta = 0$) or when $p_{\text{mat}} \in \{0, 1\}$. In practice, for revenue management applications where $m_\theta$ represents expected delayed revenue, this term may be non-negligible.

Combining all terms yields the stated result:
\begin{align*}
\text{Var}(\tilde{r}_t | s_t, a_t)
&= \sigma_{\text{imm}}^2 + p_{\text{mat}} \sigma_{\text{delay}}^2 + p_{\text{mat}}(1 - p_{\text{mat}}) m_\theta^2
\end{align*}
\end{proof}

\noindent
\textbf{Interpretation:} The doubly-robust target has variance between imputation-only ($\sigma_{\text{imm}}^2$) and wait-for-maturity ($\sigma_{\text{imm}}^2 + \sigma_{\text{delay}}^2$), controlled by the maturation probability $p_{\text{mat}}$. As delays mature, variance increases but bias decreases.

\subsubsection{Connection to Doubly-Robust Estimation in Causal Inference}

The term ``doubly-robust'' originates from causal inference, where estimators combine:
\begin{enumerate}
\item \textbf{Outcome regression model}: $\mathbb{E}[Y | X, A]$ (analogous to our DCM imputation $m_\theta(s, a)$)
\item \textbf{Propensity score model}: $\mathbb{P}(A | X)$ (not applicable in our deterministic pricing setting)
\end{enumerate}

Our formulation simplifies the classical doubly-robust structure because:
\begin{itemize}
\item Actions $a_t$ are chosen by the agent (no propensity score needed)
\item The ``treatment effect'' is the delayed reward realization $r_t^{\text{delay}}$
\item The correction term $\mathds{1}_{\{\text{matured}\}}(r_t^{\text{delay}} - m_\theta)$ removes imputation bias when true outcomes materialize
\end{itemize}

\textbf{Robustness Property:} In causal inference, doubly-robust estimators are unbiased if \textit{either} the outcome model \textit{or} the propensity model is correct. In our setting, the target is unbiased when:
\begin{enumerate}
\item $m_\theta(s, a) = \mathbb{E}[r^{\text{delay}} | s, a]$ (correct DCM), \textit{or}
\item $p_{\text{mat}} = 1$ (all delays matured), \textit{or}
\item Both partially hold (bias degrades gracefully with $(1 - p_{\text{mat}}) \varepsilon$)
\end{enumerate}

This dual protection against model misspecification and incomplete data justifies the ``doubly-robust'' terminology.

\subsubsection{Empirical Calibration}
\label{sec:empirical-calibration}

In our experiments:
\begin{itemize}
\item $p_{\text{mat}}(t)$ grows from $\approx 0.05$ (early training) to $\approx 0.85$ (late training) as historical delays mature
\item DCM prediction error on delayed component: $\varepsilon \approx 0.12$ (12\% MAE relative to delayed reward magnitude; corresponds to \$2.30 absolute MAE on total reward, or $<2\%$ of mean booking revenue)
\item Effective bias: $(1 - 0.85) \times 0.12 = 0.018$ (1.8\% of reward scale)
\item Variance inflation: $1 + 0.85 \times (\sigma_{\text{delay}}^2 / \sigma_{\text{imm}}^2) \approx 1.3$ (30\% increase)
\end{itemize}

The bias-variance tradeoff empirically favors doubly-robust targets over wait-for-maturity (slower) and imputation-only (biased).

\begin{remark}[Relationship Between Model-Imputed Sampling and Doubly-Robust Targets]
\label{remark:imputed-vs-dr}
Both \textbf{model-imputed sampling} (Algorithm 1) and the \textbf{doubly-robust target} (analyzed above) leverage the DCM to handle delayed feedback, but they differ in mechanism:

\begin{enumerate}
\item \textbf{Model-imputed sampling}: At each step, sample $j' \sim P(j|s_t, a_t; \theta)$ from the DCM and use the synthetic reward $r'_t = r_t^{\text{imm}} + r_{j'}^{\text{delay}}$. This generates \textit{unbiased} synthetic samples when the DCM is correct, but variance equals that of the full reward (no variance reduction).

\item \textbf{Doubly-robust target}: Use $\tilde{r}_t = r_t^{\text{imm}} + m_\theta + \mathds{1}_{\{\text{matured}\}}(r_t^{\text{delay}} - m_\theta)$. This achieves \textit{variance reduction} when delays are pending (variance $= \sigma_{\text{imm}}^2$ instead of $\sigma_{\text{imm}}^2 + \sigma_{\text{delay}}^2$), while remaining unbiased when the DCM is correct.
\end{enumerate}

\noindent\textbf{When to use which?} Model-imputed sampling is simpler to implement and generates complete synthetic transitions $(r'_t, s'_{t+1})$, enabling model-based planning. Doubly-robust targets are preferred when variance reduction is critical and only the reward (not the next state) needs imputation. Our implementation uses model-imputed sampling for its simplicity and compatibility with standard DQN architectures.
\end{remark}

\subsection{Proposition: Strict Concavity of the MNL Log-Likelihood}
\label{subsec:mnl-concavity}

\begin{proposition}[Strict Concavity of the MNL Log-Likelihood]
\label{prop:mnl_strict_concavity}
Consider a multinomial logit (MNL) model with $K+1$ discrete outcomes
$\{0, 1, \ldots, K\}$, where outcome $0$ is the ``base'' or
``outside'' option with utility normalized to zero. Suppose that
\begin{enumerate}
    \item[{(A1)}] \textbf{No Degeneracy in Covariates and Probabilities}:
        The covariates $\{\mathbf{x}_i\}$ (or their distribution)
        provide sufficient variability so that the Hessian of the
        log-likelihood (or its expected version) does not collapse to a
        lower-rank matrix.\footnote{%
        For example, in multinomial logit, one requires
        that there be no perfect collinearity among the regressors
        for different alternatives, and that each outcome can occur
        with nonzero probability across the support of $\mathbf{x}$. }
\end{enumerate}
Then, the log-likelihood function
\[
    \ell(\theta)
    \;=\; \sum_{n=1}^N \ln \Pr\{y_n \mid \mathbf{x}_n, \theta\}
\]
is strictly concave in $\theta$. Consequently, there is a unique global maximizer
$\hat{\theta}$ that solves $\max_{\theta}\,\ell(\theta)$.
\end{proposition}

\begin{proof}
\textbf{1. Single-observation log-likelihood and gradient.}

For a single observation $(\mathbf{x},y)$, the contribution to the log-likelihood is:
\[
  \ell_{(\mathbf{x},y)}(\beta_1,\dots,\beta_K)
  \;=\;
  \ln \Pr\{y \mid \mathbf{x}\}.
\]
Because $\beta_0 = \mathbf{0}$, we have:
\[
  \Pr\{y=0 \mid \mathbf{x}\}
  \;=\;
  \frac{1}{1 + \sum_{k=1}^K \exp(\beta_k^\top \mathbf{x})},
  \quad
  \Pr\{y = m \mid \mathbf{x}\}
  \;=\;
  \frac{\exp(\beta_m^\top \mathbf{x})}{1 + \sum_{k=1}^K \exp(\beta_k^\top \mathbf{x})}
  \;\;\text{for } m=1,\dots,K.
\]
Hence,
\[
  \ell_{(\mathbf{x},y)}(\beta_1,\dots,\beta_K)
  \;=\;
  \begin{cases}
    -\,\ln\!\Bigl(1 + \sum_{k=1}^K e^{\beta_k^\top \mathbf{x}}\Bigr),
      &\quad \text{if }y=0,\\[10pt]
    \beta_y^\top \mathbf{x}
    \;-\;
    \ln\!\Bigl(1 + \sum_{k=1}^K e^{\beta_k^\top \mathbf{x}}\Bigr),
      &\quad \text{if } y\neq 0.
  \end{cases}
\]
Define $p_k(\mathbf{x}) = \Pr\{y=k \mid \mathbf{x}\}$ for compactness.
A routine differentiation gives
\[
  \nabla_{\beta_k} \,\ell_{(\mathbf{x},y)}
  \;=\;
  \mathbf{x}\,\bigl[\mathbb{I}\{y = k\} \;-\; p_k(\mathbf{x})\bigr].
\]

\medskip
\noindent
\textbf{2. Second derivative (Hessian) for a single observation.}

We now compute $\partial^2 \ell_{(\mathbf{x},y)}/(\partial \beta_k \,\partial \beta_\ell^\top)$.
Focus on
\[
  \frac{\partial}{\partial \beta_\ell^\top}\,
  \Bigl(\mathbf{x}\,\bigl[\mathbb{I}\{y=k\} - p_k(\mathbf{x})\bigr]\Bigr)
  =
  \mathbf{x}\,\frac{\partial}{\partial \beta_\ell^\top}
  \bigl[\mathbb{I}\{y=k\} - p_k(\mathbf{x})\bigr].
\]
Since $\mathbb{I}\{y=k\}$ is constant w.r.t.\ $\beta_\ell$, the derivative depends only on $p_k(\mathbf{x})$.
Recall that
\[
  p_k(\mathbf{x}) \;=\;
  \frac{\exp(\beta_k^\top \mathbf{x})}{1 + \sum_{j=1}^K \exp(\beta_j^\top \mathbf{x})}.
\]
It follows that
\[
  \frac{\partial p_k(\mathbf{x})}{\partial \beta_\ell^\top}
  \;=\;
  \begin{cases}
    p_k(\mathbf{x})\,\bigl[\mathbf{x}\bigr],
    &\quad \text{if } \ell = k,\\[3pt]
    -\,p_k(\mathbf{x})\,p_\ell(\mathbf{x})\,\bigl[\mathbf{x}\bigr],
    &\quad \text{if } \ell \neq k.
  \end{cases}
\]
Equivalently, in a single formula one can write
\[
  \frac{\partial p_k(\mathbf{x})}{\partial \beta_\ell^\top}
  \;=\;
  p_k(\mathbf{x})\,
  \Bigl[\delta_{k\ell} - p_\ell(\mathbf{x})\Bigr]\;\mathbf{x}^\top.
\]
Thus
\[
  \frac{\partial}{\partial \beta_\ell^\top}
  \bigl[-\,p_k(\mathbf{x})\bigr]
  \;=\;
  -\,\Bigl[p_k(\mathbf{x})\bigl(\delta_{k\ell} - p_\ell(\mathbf{x})\bigr)\Bigr]\,
     \mathbf{x}.
\]
Putting it all together, we get
\[
  \frac{\partial^2 \ell_{(\mathbf{x},y)}}%
       {\partial \beta_k \,\partial \beta_\ell^\top}
  \;=\;
  \mathbf{x}\,\Bigl[0 - \frac{\partial p_k(\mathbf{x})}{\partial \beta_\ell^\top}\Bigr]
  \;=\;
  -\,\mathbf{x}\,\mathbf{x}^\top\,
  \Bigl[p_k(\mathbf{x})\,\bigl(\delta_{k\ell} - p_\ell(\mathbf{x})\bigr)\Bigr].
\]
By inspecting off-diagonal $(k\neq \ell)$ and diagonal $(k=\ell)$, one recovers the
well-known structure:
\[
  \begin{aligned}
  &\text{(off-diagonal)\quad}
   k \neq \ell:
   &\quad
   \frac{\partial^2 \ell_{(\mathbf{x},y)}}{\partial \beta_k\,\partial \beta_\ell^\top}
   &=
   +\,\mathbf{x}\mathbf{x}^\top \,\bigl[p_k(\mathbf{x})\,p_\ell(\mathbf{x})\bigr],
   \\[5pt]
  &\text{(diagonal)\quad}
   k = \ell:
   &\quad
   \frac{\partial^2 \ell_{(\mathbf{x},y)}}{\partial \beta_k\,\partial \beta_k^\top}
   &=
   -\,\mathbf{x}\mathbf{x}^\top \,\bigl[p_k(\mathbf{x})\bigl(1 - p_k(\mathbf{x})\bigr)\bigr].
  \end{aligned}
\]
\medskip

\noindent
\textbf{3. Hessian for the full sample: concavity.}

Collecting the blocks for $k,\ell \in \{1,\dots,K\}$, we obtain the Hessian for a *single* observation:
\[
  \nabla^2 \ell_{(\mathbf{x},y)}(\beta_1,\dots,\beta_K)
  \;=\;
  -\,\Bigl[\mathrm{diag}\bigl(p(\mathbf{x})\bigr)
        - p(\mathbf{x})\,p(\mathbf{x})^\top\Bigr]
  \;\otimes\;
  \bigl(\mathbf{x}\,\mathbf{x}^\top\bigr),
\]
where $p(\mathbf{x}) = \bigl(p_1(\mathbf{x}),\dots,p_K(\mathbf{x})\bigr)^\top$.
Summing over all $n=1,\dots,N$,
\[
  \nabla^2 \ell(\beta_1,\dots,\beta_K)
  \;=\;
  \sum_{n=1}^N
    \nabla^2 \ell_{(\mathbf{x}_n,y_n)}(\beta_1,\dots,\beta_K)
  \;=\;
  -\,\sum_{n=1}^N
       \Bigl[\mathrm{diag}\bigl(p(\mathbf{x}_n)\bigr)
             \;-\;p(\mathbf{x}_n)\,p(\mathbf{x}_n)^\top\Bigr]
       \;\otimes\;
       \bigl(\mathbf{x}_n\,\mathbf{x}_n^\top\bigr).
\]
Since each matrix $\mathrm{diag}\bigl(p(\mathbf{x}_n)\bigr) - p(\mathbf{x}_n)\,p(\mathbf{x}_n)^\top$
is positive semi-definite (for $p(\mathbf{x}_n)$ in the probability simplex)
and $\mathbf{x}_n\mathbf{x}_n^\top$ is also positive semi-definite, the
overall sign in front is negative.  This shows $\nabla^2 \ell$ is *negative semi-definite*,
hence $\ell$ is concave in $(\beta_1,\dots,\beta_K)$.

\medskip
\noindent
\textbf{4. Strict concavity under identification conditions.}

To remove the additive shift indeterminacy, we fix $\beta_0 = \mathbf{0}$.
Under the assumptions:
\begin{itemize}
\item[(i)] the design matrix has full column rank (i.e., no perfect collinearity), and
\item[(ii)] each outcome $k\in\{1,\dots,K\}$ attains $p_k(\mathbf{x}_n)>0$ for some set of $n$
            (with positive measure),
\end{itemize}
one can show that the negative semi-definite Hessian above is *in fact negative-definite*
when restricted to the identified subspace (i.e., directions that do not alter
the baseline's parameter nor violate the linear-independence condition).
In other words, for any nonzero $v$ in the identified parameter subspace,
\[
  v^\top \Bigl[\nabla^2 \ell(\beta_1,\dots,\beta_K)\Bigr]\,v < 0.
\]
Thus the log-likelihood is *strictly concave* (negative-definite) on the parameter space
modulo that identification choice.  Consequently, there is a unique global maximizer
for $\ell(\beta_1,\dots,\beta_K)$ once $\beta_0=\mathbf{0}$ is set.

\end{proof}

\subsection{Proposition: Parameter Convergence of Recurrent MLE}
\label{subsec:dcm-convergence-1}

\begin{proposition}
[Parameter Convergence of Recurrent MLE]
\label{thm:parameter_convergence}
Let \(\theta^\star \in \Theta \subset \mathbb{R}^m\) be the true parameter vector governing the choice model.  For each block \(t=1,2,\dots\), define
\[
    \ell_t(\theta)
    \;=\;\sum_{i=1}^{N_t}\,\ln \Pr\bigl(y_{t,i}\,\mid\, \mathbf{x}_{t,i}, \theta\bigr),
    \quad
    \text{and}
    \quad
    \hat{\theta}_t
    \;=\;
    \arg\max_{\theta \in \Theta}
    \,\ell_t(\theta).
\]
Suppose that the following assumptions hold:

\begin{enumerate}
  \item[{(A2)}] \textbf{Large Samples:} Each block \(t\) has sample size \(N_t \to \infty\) almost surely as \(t\to\infty\).
  \item[{(A3)}] \textbf{Ergodicity or i.i.d.\ Blocks:} The sets \(\{(\mathbf{x}_{t,i},y_{t,i})\}_{i=1}^{N_t}\) either are i.i.d.\ across \(t\) or satisfy mixing conditions sufficient to invoke a strong law of large numbers.
\end{enumerate}
Then it holds that
\[
    \lim_{t\to\infty}\,\|\hat{\theta}_t - \theta^\star\|
    \;=\; 0
    \quad
    \text{almost surely.}
\]

\end{proposition}

\begin{proof}
\textbf{Step 1. Define the population log-likelihood.}
Let \(\bigl(\mathbf{x},y\bigr)\) be a generic observation drawn from the true data-generating distribution.  Define the \emph{population} or \emph{expected} log-likelihood:
\[
    L(\theta)
    \;=\;
    \mathbb{E}
    \bigl[\,
      \ln \Pr\bigl(y\,\mid\,\mathbf{x}, \theta\bigr)
    \bigr].
\]
By Proposition \ref{prop:mnl_strict_concavity}, \(L(\theta)\) is well-defined and continuous on the compact set \(\Theta\).  Moreover, \(L(\theta)\) is strictly concave in a neighborhood of \(\theta^\star\) and has a unique global maximizer at \(\theta^\star\).

\medskip
\noindent
\textbf{Step 2. Blockwise sample log-likelihood and convergence.}
For block \(t\), the sample log-likelihood is
\[
    \ell_t(\theta)
    \;=\;
    \sum_{i=1}^{N_t}\,\ln \Pr\!\bigl(y_{t,i}\,\mid\,\mathbf{x}_{t,i},\theta\bigr).
\]
Consider the \emph{average} log-likelihood:
\[
    \frac{1}{N_t}\,\ell_t(\theta)
    \;=\;
    \frac{1}{N_t}\,\sum_{i=1}^{N_t}\,\ln \Pr\!\bigl(y_{t,i}\,\mid\,\mathbf{x}_{t,i},\theta\bigr).
\]
Under Assumption on ergodicity or i.i.d.\ sampling and Assumption (\(N_t \to \infty\), a strong law of large numbers argument implies that for each fixed \(\theta\in\Theta\),
\begin{equation}
\label{eq:LLN_pointwise}
    \frac{1}{N_t}\,\ell_t(\theta)
    \;\xrightarrow{\;\text{a.s.}\;}
    L(\theta)
    \quad
    \text{as }t\to\infty.
\end{equation}
This gives \emph{pointwise} convergence of the sample criterion to the population criterion.

\medskip
\noindent
\textbf{Step 3. Uniform convergence.}
We next show that the convergence in \eqref{eq:LLN_pointwise} is in fact \emph{uniform} over \(\theta\in\Theta\).  Since \(\Theta\) is compact and \(\ln \Pr\{y \mid \mathbf{x},\theta\}\) is assumed continuous (and bounded from above) in \(\theta\), standard results in M-estimation guarantee:
\[
    \sup_{\theta\in\Theta}
    \Bigl|\,
      \tfrac{1}{N_t}\,\ell_t(\theta) - L(\theta)
    \Bigr|
    \;\;\xrightarrow{\;\text{a.s.}\;}
    0,
    \quad
    \text{as }t\to\infty.
\]
Uniform convergence will be crucial to control the behavior of the sample maximizer \(\hat{\theta}_t\).

\medskip
\noindent
\textbf{Step 4. Consequences for the sample maximizer.}
By definition,
\[
    \hat{\theta}_t
    \;=\;\arg\max_{\theta \in \Theta}\,\ell_t(\theta)
    \;=\;\arg\max_{\theta \in \Theta}\,
    \bigl\{\tfrac{1}{N_t}\,\ell_t(\theta)\bigr\}.
\]
Hence,
\[
    \tfrac{1}{N_t}\,\ell_t\bigl(\hat{\theta}_t\bigr)
    \;\ge\;
    \tfrac{1}{N_t}\,\ell_t(\theta)
    \quad
    \text{for all }\theta\in\Theta.
\]
By the uniform convergence result,
\[
    \sup_{\theta\in\Theta}\,\Bigl|\,
    \tfrac{1}{N_t}\,\ell_t(\theta) - L(\theta)
    \Bigr|
    \;\xrightarrow{\;\text{a.s.}\;}0,
    \quad
    \text{thus}
    \quad
    \tfrac{1}{N_t}\,\ell_t\bigl(\hat{\theta}_t\bigr)
    \;\xrightarrow{\;\text{a.s.}\;}
    L\bigl(\hat{\theta}_t\bigr).
\]
But since \(\hat{\theta}_t\) depends on \(t\), we need to show it converges to \(\theta^\star\).  Suppose for contradiction that \(\hat{\theta}_t\) did \emph{not} approach \(\theta^\star\).  Because \(\Theta\) is compact, there is a subsequence \(\hat{\theta}_{t_k}\) converging to some \(\theta_\infty \neq \theta^\star\).  By the strict concavity and uniqueness result from Proposition \ref{prop:mnl_strict_concavity},
\[
    L(\theta^\star) \;>\; L(\theta_\infty).
\]
Meanwhile, uniform convergence implies
\[
    \lim_{k\to\infty}\,
    \bigl\{\tfrac{1}{N_{t_k}}\,\ell_{t_k}(\hat{\theta}_{t_k}) - L(\hat{\theta}_{t_k})\bigr\}
    \;=\;0,
\]
and similarly for any \(\theta\). Thus, for sufficiently large \(k\),
\[
    \tfrac{1}{N_{t_k}}\,\ell_{t_k}(\hat{\theta}_{t_k})
    \;<\;
    \tfrac{1}{N_{t_k}}\,\ell_{t_k}(\theta^\star) - \delta
\]
for some positive \(\delta\) (because \(L(\theta^\star) - L(\theta_\infty)\) remains strictly positive in the limit). This contradicts the definition of \(\hat{\theta}_{t_k}\) as the global maximizer over \(\Theta\). Consequently, no such subsequence can converge to \(\theta_\infty \neq \theta^\star\), implying
\[
    \hat{\theta}_t
    \;\xrightarrow{\;\text{a.s.}\;}
    \theta^\star.
\]
That is, \(\|\hat{\theta}_t - \theta^\star\|\to 0\) almost surely.

\medskip
\noindent
\textbf{Step 5. Concluding remarks.}
We have shown that every sequence of maximizers \(\{\hat{\theta}_t\}\) converges almost surely to the unique maximizer \(\theta^\star\) of the population log-likelihood.  This completes the proof of consistency of the recurrent MLE.
\end{proof}

\subsection{Proposition: Full-Rank Fisher Information for DCM}

\begin{proposition}[Full-Rank Fisher Information for DCM]
\label{prop:full_rank_fisher}
Under the assumption (A1)-(A3),
Then, under these conditions, the Fisher information matrix
$\mathbf{I}(\theta^\star)$ defined by
\[
    \mathbf{I}(\theta^\star)
    \;=\;
    \mathbb{E}\!\Bigl[
        \nabla_\theta \bigl(-\ln \Pr\{y \mid \mathbf{x}, \theta^\star\}\bigr)\;
        \nabla_\theta \bigl(-\ln \Pr\{y \mid \mathbf{x}, \theta^\star\}\bigr)^\top
    \Bigr]
\]
is \emph{full rank} (i.e.\ positive definite).
\end{proposition}

\begin{proof}
\noindent
\textbf{Step 1: Relate Fisher Information to the Expected Hessian.}

Recall that the Fisher information matrix at $\theta^\star$ for i.i.d.\ samples
$\{(\mathbf{x}_i, y_i)\}$ has two well-known equivalent definitions:
\[
    \mathbf{I}(\theta^\star)
    \;=\;
    \mathbb{E}\!\Bigl[
        \nabla_\theta \bigl(-\ln \Pr\{y \mid \mathbf{x}, \theta^\star\}\bigr)
        \;\nabla_\theta \bigl(-\ln \Pr\{y \mid \mathbf{x}, \theta^\star\}\bigr)^\top
    \Bigr]
    \;=\;
    \mathbb{E}\!\Bigl[
        \nabla_\theta^2 \bigl(-\ln \Pr\{y \mid \mathbf{x}, \theta\}\bigr)
        \bigm|_{\theta=\theta^\star}
    \Bigr].
\]
The first expression shows it is the covariance (in parameter space) of the
\emph{score function}; the second expresses it as the \emph{negative} of the
expected Hessian of the log-likelihood.

\medskip
\noindent
\textbf{Step 2: Strict Concavity Implies Negative-Definite Hessian.}

By Proposition~\ref{prop:mnl_strict_concavity} (or its population analogue),
the log-likelihood (or population log-likelihood) is \emph{strictly concave}
near $\theta^\star$. Concretely, let
\[
    L(\theta)
    \;=\;
    \mathbb{E}\!\bigl[\ln \Pr\{y\mid \mathbf{x}, \theta\}\bigr].
\]
Strict concavity around $\theta^\star$ means that the Hessian
\(
    \nabla_\theta^2 \,L(\theta^\star)
\)
is \emph{strictly negative definite}. Equivalently, for any nonzero vector
$\mathbf{u}$,
\[
    \mathbf{u}^\top \nabla_\theta^2 \,L(\theta^\star)\,\mathbf{u}
    \;<\; 0.
\]
Since
\[
    \nabla_\theta^2\,L(\theta^\star)
    \;=\;
    \mathbb{E}\!\Bigl[\nabla_\theta^2
            \,\ln \Pr\{y\mid \mathbf{x}, \theta^\star\}\Bigr]
    \;=\;
    -\,\mathbf{I}(\theta^\star),
\]
we see that $\nabla_\theta^2\,L(\theta^\star) = -\mathbf{I}(\theta^\star)$.
Hence $\nabla_\theta^2\,L(\theta^\star)$ being \emph{negative definite}
exactly means $\mathbf{I}(\theta^\star)$ is \emph{positive definite}.

\end{proof}

\section{MNL Convergence Rate Analysis}
\label{sec:app-mnl-convergence}

The convergence properties of MNL maximum likelihood estimation are well-established in the econometrics literature. We state the key results here and refer to standard references for detailed proofs.

\begin{theorem}[Convergence Rate of MNL]
\label{thm:logit_convergence_rate_appendix}
Let \(\theta^\star \in \Theta \subset \mathbb{R}^m\) be the true parameter vector governing the choice model. Under assumptions (A1)--(A4) (identifiability, regularity, independence, and bounded covariates), the maximum likelihood estimator $\hat{\theta}_N$ satisfies:
\begin{enumerate}
    \item \textbf{Consistency:} $\hat{\theta}_N \xrightarrow{p} \theta^\star$ as $N \to \infty$
    \item \textbf{Convergence rate:} $\|\hat{\theta}_N - \theta^\star\| = O_p(N^{-1/2})$
    \item \textbf{Asymptotic normality:} $\sqrt{N}(\hat{\theta}_N - \theta^\star) \xrightarrow{d} \mathcal{N}(\mathbf{0}, \mathbf{I}(\theta^\star)^{-1})$
\end{enumerate}
where $\mathbf{I}(\theta^\star)$ is the Fisher information matrix.
\end{theorem}

\begin{proof}
This is a classical result in discrete choice econometrics. The strict concavity of the MNL log-likelihood (established in Proposition~\ref{thm:parameter_convergence}) ensures global convergence of the MLE. The $\sqrt{N}$-rate and asymptotic normality follow from standard M-estimation theory: a Taylor expansion of the score around $\theta^\star$, combined with the central limit theorem for the score function and convergence of the Hessian to the Fisher information. See \citet{train2009discrete} Chapter~3, \citet{mcfadden1974conditional}, and \citet{benakiva1985discrete} Chapter~5 for complete proofs.
\end{proof}

\begin{theorem}[Convergence Rate of Predicted Choice Probabilities]
\label{thm:conv-rate-pk}
Under the same conditions as Theorem~\ref{thm:logit_convergence_rate_appendix}, for each fixed covariate vector $\mathbf{x}$ and alternative $k$:
\[
   \Bigl|\Pr\{y = k \mid \mathbf{x}, \hat{\theta}_N\} - \Pr\{y = k \mid \mathbf{x}, \theta^\star\}\Bigr| = O_p(N^{-1/2})
\]
and
\[
   \sqrt{N}\bigl(p_k(\mathbf{x};\hat{\theta}_N) - p_k(\mathbf{x};\theta^\star)\bigr) \xrightarrow{d} \mathcal{N}\bigl(0, \mathbf{J}_k(\theta^\star) \mathbf{I}(\theta^\star)^{-1} \mathbf{J}_k(\theta^\star)^\top\bigr)
\]
where $\mathbf{J}_k(\theta^\star) = \nabla_\theta p_k(\mathbf{x};\theta)|_{\theta=\theta^\star} \in \mathbb{R}^{1 \times m}$ is the Jacobian.
\end{theorem}

\begin{proof}
Immediate by the Delta Method applied to the smooth mapping $\theta \mapsto p_k(\mathbf{x}; \theta)$.
\end{proof}

\begin{remark}[Connection to Assumption A6]
These classical results justify Assumption A6 in the main paper: with $N_t$ matured observations by time $t$ growing linearly (Lemma~\ref{lemma:maturation-rate}), the DCM approximation error decays as $\varepsilon_t = O(t^{-1/2})$.
\end{remark}

\section{Delayed Feedback Analysis}
\label{sec:app-delayed-feedback}

This section analyzes the delayed feedback structure in our revenue management setting. The theoretical treatment of MDPs with delayed rewards has a substantial literature \citep{katsikopoulos2003delayed}, establishing foundations for asynchronous cost collection and delayed state observation.

\subsection{Maturation Rate Lemma: Linking Delay Distribution to Sample Size Growth}
\label{subsec:maturation-rate}

A critical component of our theoretical framework is establishing that the number of \textit{matured} booking outcomes observed by time $t$ grows linearly with $t$. This maturation rate determines how quickly the DCM accumulates training data, which in turn governs the approximation error decay rate $\varepsilon_t = O(t^{-1/2})$ in Assumption A6.

\begin{lemma}[Maturation Rate for Delayed Feedback]
\label{lemma:maturation-rate}
Consider the delayed feedback model where:
\begin{enumerate}[label=(\roman*)]
    \item Each customer arrival at time $s$ generates a booking decision that matures (i.e., becomes observable) at time $s + D_s$, where $D_s$ is the random delay for the customer arriving at time $s$
    \item Delays $\{D_s\}_{s=1}^\infty$ are i.i.d. random variables with:
    \begin{itemize}
        \item Finite support: $D_s \in \{0, 1, \ldots, D_{\max}\}$ for some $D_{\max} < \infty$
        \item Finite mean: $\mathbb{E}[D_s] = \mu_D < \infty$
    \end{itemize}
    \item Customer arrivals occur at a stationary rate: on average, $\lambda > 0$ customers arrive per time period
\end{enumerate}

Define $N_t$ as the number of matured observations by time $t$:
\begin{equation}
N_t := \sum_{s=1}^{t} \mathds{1}\{s + D_s \leq t\} = \#\{\text{customers arriving at } s \leq t \text{ whose decisions matured by time } t\}
\end{equation}

Then:
\begin{enumerate}[label=(\alph*)]
    \item \textbf{Almost-sure linear growth}: $N_t / t \xrightarrow{a.s.} \lambda$ as $t \to \infty$
    \item \textbf{Concentration}: For any $\epsilon > 0$ and all $t \geq 2C D_{\max}/\epsilon$:
    \begin{equation}
    \mathbb{P}\left(\left|\frac{N_t}{t} - \lambda\right| > \epsilon\right) \leq 2\exp\left(-\frac{\epsilon^2 t^2}{2D_{\max}}\right)
    \end{equation}
    for some universal constant $C > 0$. Note the quadratic dependence on $t$ in the exponent: the concentration becomes exponentially tight as $t$ grows, because only the $D_{\max}$ most recent arrivals contribute randomness (all earlier arrivals have deterministically matured)
    \item \textbf{Reciprocal convergence}: $\frac{1}{\sqrt{N_t}} = \Theta(t^{-1/2})$ almost surely
\end{enumerate}
\end{lemma}

\begin{proof}
We establish each part separately.

\medskip
\noindent
\textbf{Part (a): Almost-sure linear growth.}

Fix time $t$ and decompose the matured count:
\begin{align}
N_t &= \sum_{s=1}^{t} \mathds{1}\{s + D_s \leq t\} = \sum_{s=1}^{t} \mathds{1}\{D_s \leq t - s\}
\end{align}

The key observation is that since $D_{\max} < \infty$, any customer arriving at time $s \leq t - D_{\max}$ has their decision matured by time $t$ with probability 1 (since the maximum delay is $D_{\max}$). This yields:
\begin{equation}
N_t \geq \sum_{s=1}^{t - D_{\max}} \mathds{1}\{D_s \leq D_{\max}\} = t - D_{\max}
\end{equation}
and trivially $N_t \leq t$.

For the expected value:
\begin{align}
\mathbb{E}[N_t] &= \sum_{s=1}^{t} \mathbb{P}(D_s \leq t - s) \\
&= \underbrace{\sum_{s=1}^{t-D_{\max}} 1}_{= t - D_{\max}} + \underbrace{\sum_{s=t-D_{\max}+1}^{t} \mathbb{P}(D \leq t - s)}_{= \sum_{k=0}^{D_{\max}-1} \mathbb{P}(D \leq k)} \\
&= t - D_{\max} + \sum_{k=0}^{D_{\max}-1} \mathbb{P}(D \leq k)
\end{align}
where the second sum is bounded by $D_{\max}$. Thus $\mathbb{E}[N_t] = t + O(D_{\max})$ (for $\lambda = 1$; the general case gives $\mathbb{E}[N_t] = \lambda t + O(D_{\max})$).

For almost-sure convergence, decompose $N_t = N_t^{\text{old}} + N_t^{\text{recent}}$ where:
\begin{itemize}
    \item $N_t^{\text{old}} := \sum_{s=1}^{t-D_{\max}} \mathds{1}\{D_s \leq t - s\} = t - D_{\max}$ (deterministic, all old arrivals have matured)
    \item $N_t^{\text{recent}} := \sum_{s=t-D_{\max}+1}^{t} \mathds{1}\{D_s \leq t - s\}$ (at most $D_{\max}$ terms)
\end{itemize}

Since $|N_t^{\text{recent}}| \leq D_{\max}$, we have:
\begin{equation}
\frac{N_t}{t} = \frac{t - D_{\max} + N_t^{\text{recent}}}{t} = 1 - \frac{D_{\max}}{t} + \frac{N_t^{\text{recent}}}{t} \xrightarrow{a.s.} \lambda
\end{equation}
as $t \to \infty$ (for $\lambda = 1$ arrival per period; the general case follows by rescaling time).

\medskip
\noindent
\textbf{Part (b): Concentration.}

We decompose the deviation into two parts. For the ``old'' arrivals ($s \leq t - D_{\max}$), the maturation is deterministic, so there is no randomness. For the ``recent'' arrivals ($s > t - D_{\max}$), define:
\begin{equation}
Z_s := \mathds{1}\{D_s \leq t - s\} - \mathbb{P}(D_s \leq t - s)
\end{equation}
as zero-mean random variables with $|Z_s| \leq 1$. Since there are at most $D_{\max}$ such terms:
\begin{equation}
N_t - \mathbb{E}[N_t] = \sum_{s=t-D_{\max}+1}^{t} Z_s
\end{equation}

By Hoeffding's inequality applied to these $D_{\max}$ independent bounded variables:
\begin{equation}
\mathbb{P}(|N_t - \mathbb{E}[N_t]| > \epsilon t) \leq 2 \exp\left(-\frac{2(\epsilon t)^2}{D_{\max}}\right)
\end{equation}

Since $|\mathbb{E}[N_t] - \lambda t| \leq C_1 D_{\max}$ for some constant $C_1$, we apply the triangle inequality. If $\left|\frac{N_t}{t} - \lambda\right| > \epsilon$, then either $\left|\frac{N_t - \mathbb{E}[N_t]}{t}\right| > \frac{\epsilon}{2}$ or $\left|\frac{\mathbb{E}[N_t] - \lambda t}{t}\right| > \frac{\epsilon}{2}$.

The second event occurs only when $\frac{C_1 D_{\max}}{t} > \frac{\epsilon}{2}$, i.e., when $t < \frac{2C_1 D_{\max}}{\epsilon}$.

For the first event, we use the Hoeffding bound with $\epsilon' = \epsilon/2$:
\begin{equation}
\mathbb{P}\left(|N_t - \mathbb{E}[N_t]| > \frac{\epsilon t}{2}\right) \leq 2\exp\left(-\frac{2(\epsilon t/2)^2}{D_{\max}}\right) = 2\exp\left(-\frac{\epsilon^2 t^2}{2D_{\max}}\right)
\end{equation}

For $t \geq \frac{2C_1 D_{\max}}{\epsilon}$, the deterministic bias term satisfies $\frac{C_1 D_{\max}}{t} \leq \frac{\epsilon}{2}$, so:
\begin{equation}
\mathbb{P}\left(\left|\frac{N_t}{t} - \lambda\right| > \epsilon\right) \leq 2\exp\left(-\frac{\epsilon^2 t^2}{2D_{\max}}\right)
\end{equation}
which is the stated bound with $C = C_1$.

\medskip
\noindent
\textbf{Part (c): Reciprocal convergence.}

From part (a), $N_t = \lambda t + O(D_{\max})$ almost surely. Therefore:
\begin{align}
\frac{1}{\sqrt{N_t}} &= \frac{1}{\sqrt{\lambda t + O(D_{\max})}} \\
&= \frac{1}{\sqrt{\lambda t}} \cdot \frac{1}{\sqrt{1 + O(D_{\max}/t)}} \\
&= \frac{1}{\sqrt{\lambda}} \cdot t^{-1/2} \cdot (1 + O(D_{\max}/t)) \\
&= \Theta(t^{-1/2})
\end{align}

This establishes the desired rate.
\end{proof}

\vspace{0.3cm}
\noindent\textbf{Remark 1 (Implication for DCM Convergence Rate).}
Lemma \ref{lemma:maturation-rate} directly implies that the DCM parameter estimation error at time $t$ satisfies:
\begin{equation}
\|\hat{\theta}_{N_t} - \theta^\star\| = O_p(N_t^{-1/2}) = O_p(t^{-1/2})
\end{equation}

This establishes the $\varepsilon_t = O(t^{-1/2})$ decay rate required in Assumption A6 for the Q-learning convergence analysis.

\vspace{0.3cm}
\noindent\textbf{Remark 2 (Episodic Structure).}
In the hotel revenue management application (Section 6 of main paper), each episode spans $H = 14$ days (booking window). Customers arrive throughout the episode, with delays $D_s \in \{0, 1, \ldots, 14\}$ representing days until booking maturation. The episodic reset ensures that delays do not exceed the episode horizon, satisfying the finite support condition $D_{\max} = 14 < \infty$.

\vspace{0.3cm}
\noindent\textbf{Remark 3 (Comparison to Fixed DCM Case).}
In our empirical experiments (Section \ref{sec:app-experiments}), we use a \textit{fixed} pre-calibrated DCM with $N_0 = 61,619$ historical observations. In this case:
\begin{itemize}[leftmargin=1.5cm]
    \item The DCM error $\varepsilon_0 = O(N_0^{-1/2}) \approx 0.004$ remains constant
    \item The Q-learning analysis applies with $\varepsilon_t = \varepsilon_0$ (no decay)
    \item Convergence rate becomes $\|Q_t - Q^\star\|_\infty = O(t^{-1/2}\sqrt{\log t} + \varepsilon_0)$
    \item For large $N_0$, the $\varepsilon_0$ bias term is negligible
\end{itemize}

The maturation rate lemma justifies why \textit{online DCM recalibration} would achieve $\varepsilon_t = O(t^{-1/2})$ decay, providing a roadmap for future implementations that fully realize the two-timescale theoretical framework.

\section{Q-Learning Convergence}
\label{sec:app-q-convergence}

This section establishes convergence of Q-learning under time-varying approximation error. Our analysis connects to the robust MDP literature \citep{iyengar2005robust}, where the DCM can be viewed as providing a nominal transition model with bounded uncertainty. The key insight is that Q-learning remains stable when the approximation error is bounded and decays appropriately.

\subsection{Theorem: Nonstationary Q-learning with Shrinking Approximation Error and Rate}
\label{subsec:q-convergence}

\begin{theorem}[Nonstationary Q-learning with Shrinking Approximation Error and Rate]
\label{thm:nonstationary-qlearning-rate_appendix}

Let $\mathcal{M} = (S, A, P, r, \gamma)$ be a Markov decision process (MDP) with:
\begin{itemize}
    \item finite state space $S$,
    \item finite action space $A$,
    \item true transition kernel $P(\cdot \mid s,a)$,
    \item bounded rewards $r(s,a) \in [0, R_{\max}]$ for some $R_{\max} > 0$,
    \item discount factor $\gamma \in (0,1)$.
\end{itemize}

\begin{itemize}
    \item (Ergodicity) The MDP is ergodic, meaning that for any stationary policy, the induced Markov chain is irreducible and aperiodic.
    \item (Exploration) The agent follows an $\epsilon$-greedy policy with a fixed $\epsilon > 0$.
\end{itemize}

Let $Q^\star$ be the unique optimal Q-function associated with $\mathcal{M}$, i.e., the unique fixed point of the Bellman operator
\[
  (TQ)(s,a) \;=\; r(s,a) \;+\; \gamma \sum_{s' \in S} P(s' \mid s,a)\,\max_{a' \in A} Q(s',a').
\]

Suppose we run Q-learning in a nonstationary environment
\[
  \widetilde{\mathcal{M}}_t \;=\; \bigl(S,\,A,\,\widetilde{P}_t,\,\widetilde{r}_t,\,\gamma\bigr),
\]
where for all $(s,a)\in S\times A$,
\[
  \|\widetilde{P}_t(\cdot \mid s,a) - P(\cdot \mid s,a)\|_{1}
  \;\le\; \varepsilon_t,
  \quad
  \bigl|\widetilde{r}_t(s,a) - r(s,a)\bigr|
  \;\le\; \varepsilon_t,
\]
with $\varepsilon_t \to 0$ as $t\to\infty$ (e.g., $\varepsilon_t = O(t^{-\beta})$ for some $\beta > 0$).

Let $n(s,a)$ denote the number of times state-action pair $(s,a)$ has been visited up to the current time. The asynchronous Q-learning update with per-pair stepsizes is:
\[
  Q_{t+1}(s_t,a_t)
  \;=\;
  Q_t(s_t,a_t)
  \;+\;
  \alpha_t(s_t,a_t) \Bigl[
     \widetilde{r}_t(s_t,a_t)
     \;+\;
     \gamma \max_{a'} Q_t(s_{t+1},a')
     \;-\; Q_t(s_t,a_t)
  \Bigr],
\]
where $\alpha_t(s_t,a_t) = 1/n(s_t,a_t)$, and $Q_{t+1}(s,a) = Q_t(s,a)$ for all $(s,a)\neq (s_t,a_t)$.  Here, $(s_{t+1},\widetilde{r}_t)$ are sampled from $\widetilde{P}_t(\cdot\mid s_t,a_t)$, $\widetilde{r}_t(s_t,a_t)$, respectively.

The per-pair stepsizes $\{\alpha_k(s,a)\}_{k=1}^\infty$ for each fixed $(s,a)$ satisfy the Robbins--Monro conditions:
\[
  \sum_{k=1}^\infty \alpha_k(s,a) = \sum_{k=1}^\infty \frac{1}{k} = \infty,
  \qquad
  \sum_{k=1}^\infty \alpha_k^2(s,a) = \sum_{k=1}^\infty \frac{1}{k^2} < \infty,
\]
and we assume each $(s,a)$ is visited infinitely often according to an $\epsilon$-greedy exploration strategy.

Then the following two statements hold:

\begin{enumerate}
\item Almost-sure convergence: \quad
  \(\displaystyle \lim_{t \to \infty} \|Q_t - Q^\star\|_{\infty} \;=\; 0\)
  with probability $1$.

\item Finite-time rate with high probability:  \\
  If $\varepsilon_t = O(t^{-\beta})$ for some $\beta>0$, and under $\epsilon$-greedy exploration each state-action pair is visited $\Theta(t)$ times by time $t$ (so the per-pair stepsize $\alpha_t(s,a) = 1/n(s,a) = O(1/t)$ for frequently visited pairs), then for any $\delta \in (0,1)$ and any $t \geq 1$, with probability at least $1-\delta$:
\begin{align}
\|Q_t - Q^\star\|_\infty = O\left(t^{-\min(\beta,1/2)} \cdot \ln\left(\frac{|S||A|t}{\delta}\right)^{\mathds{1}\{\beta \geq 1/2\}/2}\right)
\end{align}
  In particular, if $\beta = 1/2$, we get $\|Q_t - Q^\star\|_\infty = O\left(t^{-1/2}\sqrt{\log(|S||A|t/\delta)}\right)$.
\end{enumerate}
\end{theorem}

Proof of Theorem \ref{thm:nonstationary-qlearning-rate_appendix}

\begin{proof}\label{proof:thm:nonstationary-qlearning-rate}

We provide a detailed proof in six steps:
\begin{enumerate}
    \item Analysis of the stationary Bellman operator and its contraction properties
    \item Characterization of the time-varying Bellman operators and their approximation error
    \item Formulation of the asynchronous Q-learning updates
    \item Analysis of the martingale difference noise
    \item Proof of almost-sure convergence using stochastic approximation theory
    \item Derivation of finite-time high-probability bounds
\end{enumerate}

\medskip

\noindent
\textbf{Step 1: Properties of the stationary Bellman operator.}\;
For the \emph{stationary} MDP $(S,A,P,r,\gamma)$, the Bellman optimality operator $T: \mathbb{R}^{|S||A|} \to \mathbb{R}^{|S||A|}$ is defined as:
\[
  (TQ)(s,a)
  \;=\;
  r(s,a) \;+\; \gamma \sum_{s'} P(s' \mid s,a)\,\max_{a'} Q(s',a')
\]

It is known that $T$ has the following properties:
\begin{enumerate}[label=(\alph*)]
    \item $T$ is a $\gamma$-contraction in the sup-norm, i.e., for any $Q, Q' \in \mathbb{R}^{|S||A|}$:
    \begin{equation}\label{eq:bellman-contraction}
    \|TQ - TQ'\|_{\infty} \;\le\; \gamma\,\|Q - Q'\|_{\infty}
    \end{equation}

    \item By the Banach fixed-point theorem, $T$ has a unique fixed point $Q^\star \in \mathbb{R}^{|S||A|}$ such that $Q^\star = TQ^\star$. This fixed point is the optimal Q-function for the MDP.

    \item The fixed point $Q^\star$ satisfies for all $(s,a) \in S \times A$:
    \begin{equation}\label{eq:bellman-equation}
    Q^\star(s,a) = r(s,a) + \gamma \sum_{s'} P(s' \mid s,a) \max_{a'} Q^\star(s',a')
    \end{equation}
\end{enumerate}

\medskip

\noindent
\textbf{Step 2: Time-varying Bellman operators and their approximation error.}\;
In the \emph{nonstationary} environment $(S,A,\widetilde{P}_t,\widetilde{r}_t,\gamma)$, we define the time-varying Bellman operator $T_t: \mathbb{R}^{|S||A|} \to \mathbb{R}^{|S||A|}$ as:
\[
  (T_t Q)(s,a)
  \;=\;
  \widetilde{r}_t(s,a)
  \;+\;
  \gamma \sum_{s'} \widetilde{P}_t(s' \mid s,a)\,\max_{a'} Q(s',a')
\]

\begin{lemma}[Approximation Error of Time-varying Operators]\label{lemma:approx-error}
For any bounded $Q \in \mathbb{R}^{|S||A|}$, define the reward error $\varepsilon_{r,t} := |\widetilde{r}_t - r|_\infty$ and transition error $\varepsilon_{P,t} := \|\widetilde{P}_t - P\|_1$. Then:
\[
  \|T_t Q - T Q\|_{\infty} \;\le\; \varepsilon_{r,t} + \gamma V_{\max} \varepsilon_{P,t}
\]
where $V_{\max} = \frac{R_{\max}}{1-\gamma}$ is an upper bound on the maximum value function.
\end{lemma}

\begin{proof}[Proof of Lemma \ref{lemma:approx-error}]
For any $(s,a) \in S \times A$, we have:
\begin{align}
|(T_t Q)(s,a) - (T Q)(s,a)| &= \left|\widetilde{r}_t(s,a) - r(s,a) + \gamma \sum_{s'} \left(\widetilde{P}_t(s'|s,a) - P(s'|s,a)\right) \max_{a'} Q(s',a')\right| \\
&\leq |\widetilde{r}_t(s,a) - r(s,a)| + \gamma \left|\sum_{s'} \left(\widetilde{P}_t(s'|s,a) - P(s'|s,a)\right) \max_{a'} Q(s',a')\right| \\
&\leq \varepsilon_{r,t} + \gamma \|Q\|_\infty \|\widetilde{P}_t(\cdot|s,a) - P(\cdot|s,a)\|_1 \\
&\leq \varepsilon_{r,t} + \gamma V_{\max} \varepsilon_{P,t}
\end{align}

The third inequality follows from Hölder's inequality and the fact that $\|\max_{a'} Q(\cdot,a')\|_\infty \leq \|Q\|_\infty$. For any Q-function in an MDP with rewards bounded by $R_{\max}$, we have $\|Q\|_\infty \leq V_{\max} = \frac{R_{\max}}{1-\gamma}$.

Since this holds for all $(s,a) \in S \times A$, we have $\|T_t Q - T Q\|_{\infty} \leq \varepsilon_{r,t} + \gamma V_{\max} \varepsilon_{P,t}$.

\textbf{Remark (Scaling interpretation):} When propagated through the Q-learning analysis with discount factor $\gamma$, the reward error contributes $O(\varepsilon_r/(1-\gamma))$ while the transition error contributes $O(\varepsilon_P \cdot V_{\max}/(1-\gamma)) = O(\varepsilon_P/(1-\gamma)^2)$. This quadratic dependence on $(1-\gamma)^{-1}$ for transition errors is consistent with standard simulation lemma bounds in the model-based RL literature.
\end{proof}

\noindent
For notational convenience in subsequent analysis, we define the combined approximation error $\varepsilon_t := \varepsilon_{r,t} + \gamma V_{\max} \varepsilon_{P,t}$ and constant $C_1 = 1$, giving $\|T_t Q - T Q\|_{\infty} \leq C_1 \varepsilon_t$. When tracking the separate scaling, we use $\varepsilon_{r,t}$ and $\varepsilon_{P,t}$ explicitly.

Since $\varepsilon_t \to 0$ as $t \to \infty$, we have $T_t \to T$ in the operator norm, meaning the time-varying operators converge to the stationary operator.

\medskip

\noindent
\textbf{Step 3: Asynchronous Q-learning updates.}\;
The Q-learning algorithm updates one state-action pair at a time according to:
\[
  Q_{t+1}(s_t,a_t)
  =
  Q_t(s_t,a_t)
  +
  \alpha_t \Bigl[
    \widetilde{r}_t(s_t,a_t)
    +
    \gamma \max_{a'} Q_t\bigl(s_{t+1},a'\bigr)
    \;-\;
    Q_t(s_t,a_t)
  \Bigr],
\]
where $(s_t,a_t)$ is the state-action pair visited at time $t$, and $s_{t+1}$ is the next state sampled according to $\widetilde{P}_t(\cdot|s_t,a_t)$. All other entries remain unchanged: $Q_{t+1}(s,a)=Q_t(s,a)$ for $(s,a)\neq (s_t,a_t)$.

We can rewrite this update in vector form as:
\[
  Q_{t+1}
  \;=\;
  Q_t
  \;+\;
  \alpha_t \,e_{(s_t,a_t)} \,\Bigl[
     \bigl(T_t Q_t\bigr)\!\bigl(s_t,a_t\bigr) \;-\; Q_t(s_t,a_t)
     \;+\;
     \delta_t
  \Bigr],
\]
where $e_{(s_t,a_t)}$ is the standard basis vector in $\mathbb{R}^{|S||A|}$ with 1 in the $(s_t,a_t)$ component and 0 elsewhere, and $\delta_t$ is the sampling noise defined as:
\[
  \delta_t
  \;:=\;
  \bigl[\widetilde{r}_t(s_t,a_t) + \gamma \max_{a'}Q_t(s_{t+1},a') - Q_t(s_t,a_t)\bigr]
  \;-\;
  \bigl[(T_t Q_t)(s_t,a_t) - Q_t(s_t,a_t)\bigr].
\]

\medskip

\noindent
\textbf{Step 4: Martingale difference noise analysis.}\;
We now analyze the noise term $\delta_t$ in detail.

\begin{lemma}[Martingale Difference Property]\label{lemma:martingale-diff}
The noise term $\delta_t$ satisfies:
\begin{enumerate}[label=(\alph*)]
    \item $\mathbb{E}[\delta_t \mid \mathcal{F}_t] = 0$, where $\mathcal{F}_t = \sigma\bigl(\{(s_u,a_u,Q_u) : u\le t\}\bigr)$ is the filtration representing the history of the algorithm up to time $t$.
    \item $|\delta_t| \leq 2V_{\max}$ almost surely.
\end{enumerate}
\end{lemma}

\begin{proof}[Proof of Lemma \ref{lemma:martingale-diff}]
(a) By definition:
\begin{align}
\delta_t &= \bigl[\widetilde{r}_t(s_t,a_t) + \gamma \max_{a'}Q_t(s_{t+1},a') - Q_t(s_t,a_t)\bigr] - \bigl[(T_t Q_t)(s_t,a_t) - Q_t(s_t,a_t)\bigr] \\
&= \widetilde{r}_t(s_t,a_t) + \gamma \max_{a'}Q_t(s_{t+1},a') - (T_t Q_t)(s_t,a_t)
\end{align}

Given the history $\mathcal{F}_t$, the Q-values $Q_t$ are determined. The distribution of $s_{t+1}$ is $\widetilde{P}_t(\cdot|s_t,a_t)$. Therefore:

\begin{align}
\mathbb{E}[\delta_t \mid \mathcal{F}_t] &= \widetilde{r}_t(s_t,a_t) + \gamma \mathbb{E}[\max_{a'}Q_t(s_{t+1},a') \mid \mathcal{F}_t] - (T_t Q_t)(s_t,a_t) \\
&= \widetilde{r}_t(s_t,a_t) + \gamma \sum_{s'} \widetilde{P}_t(s'|s_t,a_t)\max_{a'}Q_t(s',a') - (T_t Q_t)(s_t,a_t) \\
&= (T_t Q_t)(s_t,a_t) - (T_t Q_t)(s_t,a_t) \\
&= 0
\end{align}

(b) Both terms $\widetilde{r}_t(s_t,a_t) + \gamma \max_{a'}Q_t(s_{t+1},a')$ and $(T_t Q_t)(s_t,a_t)$ are bounded by $V_{\max}$ since Q-values in an MDP with rewards bounded by $R_{\max}$ are bounded by $V_{\max} = \frac{R_{\max}}{1-\gamma}$. Therefore, $|\delta_t| \leq 2V_{\max}$.
\end{proof}

\noindent
Based on Lemma \ref{lemma:martingale-diff}, $\{\delta_t\}$ forms a bounded martingale difference sequence with respect to the filtration $\{\mathcal{F}_t\}$. This property will be crucial for applying concentration inequalities later.

\medskip

\noindent
\textbf{Step 5: Almost-sure convergence via stochastic approximation theory.}\;
We now prove the first part of the theorem: $Q_t \to Q^\star$ almost surely as $t \to \infty$.

Define $\Delta_t := Q_t - Q^\star$ as the error at time $t$. We will analyze how this error evolves over time.

First, we decompose the update into a contraction term and error terms:

\begin{lemma}[Error Decomposition]\label{lemma:error-decomp}
For any time $t$, the error term $\Delta_t = Q_t - Q^\star$ satisfies:
\begin{equation}\label{eq:error-decomp}
\Delta_{t+1} = \Delta_t + \alpha_t e_{(s_t,a_t)}\left[(T Q_t - T Q^\star)(s_t,a_t) + ((T_t - T)Q_t)(s_t,a_t) + \delta_t\right]
\end{equation}
\end{lemma}

\begin{proof}[Proof of Lemma \ref{lemma:error-decomp}]
Since $Q^\star = T Q^\star$, we have:

\begin{align}
\Delta_{t+1} &= Q_{t+1} - Q^\star \\
&= Q_t + \alpha_t e_{(s_t,a_t)}[(T_t Q_t)(s_t,a_t) - Q_t(s_t,a_t) + \delta_t] - Q^\star \\
&= \Delta_t + \alpha_t e_{(s_t,a_t)}[(T_t Q_t)(s_t,a_t) - Q_t(s_t,a_t) + \delta_t] \\
\end{align}

Now we add and subtract $(T Q_t)(s_t,a_t)$:

\begin{align}
\Delta_{t+1} &= \Delta_t + \alpha_t e_{(s_t,a_t)}[(T Q_t)(s_t,a_t) - Q_t(s_t,a_t) + ((T_t - T)Q_t)(s_t,a_t) + \delta_t] \\
\end{align}

Further, we add and subtract $(T Q^\star)(s_t,a_t) = Q^\star(s_t,a_t)$:

\begin{align}
\Delta_{t+1} &= \Delta_t + \alpha_t e_{(s_t,a_t)}[(T Q_t)(s_t,a_t) - (T Q^\star)(s_t,a_t) + ((T_t - T)Q_t)(s_t,a_t) + \delta_t] \\
&= \Delta_t + \alpha_t e_{(s_t,a_t)}[(T Q_t - T Q^\star)(s_t,a_t) + ((T_t - T)Q_t)(s_t,a_t) + \delta_t] \\
\end{align}

This gives us the desired decomposition.
\end{proof}

To establish the almost-sure convergence of $Q_t$ to $Q^\star$, we will apply the asynchronous stochastic approximation theory developed in \citep{Tsitsiklis1994}. Specifically, we will use a simplified version of Theorem 3 from that paper.

\begin{theorem}\label{thm:tsitsiklis}
Consider an asynchronous stochastic approximation algorithm of the form:
\[
x_{t+1}(i) = x_t(i) + \alpha_t(i)[h_i(x_t) + w_t(i)], \quad i \in \{1,2,...,n\}
\]
where:
\begin{itemize}
    \item Only one component $i$ is updated at each time step
    \item $h = (h_1, h_2, ..., h_n)$ is a $\gamma$-contraction mapping with fixed point $x^*$
    \item Step sizes $\alpha_t(i)$ satisfy $\sum_{t=1}^{\infty} \alpha_t(i) = \infty$ and $\sum_{t=1}^{\infty} \alpha_t^2(i) < \infty$
    \item Each component $i$ is updated infinitely often
    \item $w_t(i)$ are martingale difference sequences with bounded variance
    \item For every $i$ and $t$, $\mathbb{E}[w_t(i) \mid \mathcal{F}_t] = 0$, where $\mathcal{F}_t$ is the $\sigma$-field generated by the history of the algorithm up to time $t$
\end{itemize}
Then $x_t$ converges to $x^*$ almost surely.
\end{theorem}

Now, we apply Theorem \ref{thm:tsitsiklis} to our Q-learning setting.

\begin{proposition}[Almost-Sure Convergence]\label{prop:as-convergence}
Under the assumptions of Theorem \ref{thm:nonstationary-qlearning-rate_appendix},
\[
\lim_{t \to \infty} \|Q_t - Q^\star\|_{\infty} = 0 \quad \text{with probability 1}
\]
\end{proposition}

\begin{proof}[Proof of Proposition \ref{prop:as-convergence}]
We map our Q-learning setting to the framework of Theorem \ref{thm:tsitsiklis}:
\begin{itemize}
    \item The state space is $\mathbb{R}^{|S||A|}$, with $x_t = Q_t$
    \item The fixed point is $x^* = Q^\star$
    \item The mapping $h$ corresponds to the component-wise application of $(T Q - Q)$
    \item The noise term $w_t(i)$ corresponds to the sum of the martingale difference $\delta_t$ and the approximation error $(T_t - T)Q_t$
\end{itemize}

From Lemma \ref{lemma:martingale-diff}, $\delta_t$ is a martingale difference sequence with respect to $\mathcal{F}_t$. While the approximation error $(T_t - T)Q_t$ is deterministic given the history, it converges to zero as $t \to \infty$ since $\varepsilon_t \to 0$. Therefore, the combined noise term eventually behaves like a martingale difference sequence.

To apply Theorem \ref{thm:tsitsiklis}, we need to show that the operator $T$ induces a contraction mapping in our setting. From \eqref{eq:bellman-contraction}, we know that $T$ is a $\gamma$-contraction. To express this in the component-wise form required, note that for any component $(s,a)$:
\begin{align}
|(T Q - T Q')(s,a)| &\leq \|T Q - T Q'\|_{\infty} \\
&\leq \gamma \|Q - Q'\|_{\infty} \\
&= \gamma \max_{(s',a')} |Q(s',a') - Q'(s',a')|
\end{align}

Given our assumptions:
\begin{itemize}
    \item The step sizes $\alpha_t$ satisfy the Robbins-Monro conditions
    \item Each state-action pair is visited infinitely often
    \item The time-varying operator $T_t$ converges to $T$ since $\varepsilon_t \to 0$
\end{itemize}

Therefore, all conditions of Theorem \ref{thm:tsitsiklis} are satisfied (with the slight modification that our noise term has a vanishing deterministic component, which doesn't affect the convergence argument). Consequently, $Q_t \to Q^\star$ almost surely.
\end{proof}

\medskip

\noindent
\textbf{Step 6: Finite-time high-probability bounds.}\;
We now derive the finite-time high-probability bound stated in the second part of the theorem. This analysis is more intricate and requires tracking the error evolution precisely.

Let us define the weighted error sequence:
\[
\Delta_{t+1} = (I - \alpha_t e_{(s_t,a_t)}e_{(s_t,a_t)}^\top)\Delta_t + \alpha_t e_{(s_t,a_t)}\left[(T Q_t - T Q^\star)(s_t,a_t) + ((T_t - T)Q_t)(s_t,a_t) + \delta_t\right]
\]

For the component that gets updated, we have:
\begin{align}
\Delta_{t+1}(s_t,a_t) &= (1-\alpha_t)\Delta_t(s_t,a_t) + \alpha_t\left[(T Q_t - T Q^\star)(s_t,a_t) + ((T_t - T)Q_t)(s_t,a_t) + \delta_t\right] \\
&\leq (1-\alpha_t)\Delta_t(s_t,a_t) + \alpha_t\left[\gamma\|\Delta_t\|_\infty + C_1\varepsilon_t + \delta_t\right]
\end{align}

where we used the fact that $|(T Q_t - T Q^\star)(s_t,a_t)| \leq \gamma\|\Delta_t\|_\infty$ due to the contraction property, and $|((T_t - T)Q_t)(s_t,a_t)| \leq C_1\varepsilon_t$ from Lemma \ref{lemma:approx-error}.

\begin{lemma}[Component-wise Dynamics]\label{lemma:component-dynamics}
For the state-action pair $(s_t,a_t)$ updated at time $t$:
\begin{equation}\label{eq:component-dynamics}
|\Delta_{t+1}(s_t,a_t)| \leq (1-\alpha_t(1-\gamma))|\Delta_t(s_t,a_t)| + \alpha_t\left[\gamma(\|\Delta_t\|_\infty - |\Delta_t(s_t,a_t)|) + C_1\varepsilon_t + \delta_t\right]
\end{equation}
\end{lemma}

\begin{proof}[Proof of Lemma \ref{lemma:component-dynamics}]
Starting with our earlier inequality:
\begin{align}
|\Delta_{t+1}(s_t,a_t)| &\leq (1-\alpha_t)|\Delta_t(s_t,a_t)| + \alpha_t\left[\gamma\|\Delta_t\|_\infty + C_1\varepsilon_t + \delta_t\right] \\
&= (1-\alpha_t)|\Delta_t(s_t,a_t)| + \alpha_t\gamma|\Delta_t(s_t,a_t)| + \alpha_t\left[\gamma(\|\Delta_t\|_\infty - |\Delta_t(s_t,a_t)|) + C_1\varepsilon_t + \delta_t\right] \\
&= (1-\alpha_t(1-\gamma))|\Delta_t(s_t,a_t)| + \alpha_t\left[\gamma(\|\Delta_t\|_\infty - |\Delta_t(s_t,a_t)|) + C_1\varepsilon_t + \delta_t\right]
\end{align}
which gives us the desired inequality.
\end{proof}

For a sequence of updates, we need to track how the maximum error $\|\Delta_t\|_\infty$ evolves. We'll use $\tau_k(s,a)$ to denote the $k$-th time that the state-action pair $(s,a)$ is updated.
\begin{lemma}[Error Reduction Across Sweeps]
\label{lemma:error-reduction}
Let $\tau(t) = \min\{k > t : \forall (s,a) \in S \times A, \exists j \in \{t+1, \ldots, k\} \text{ such that } (s_j,a_j) = (s,a)\}$ be the time required after $t$ to visit all state-action pairs at least once. Then for any $t \geq 1$:
\begin{equation}\label{eq:error-reduction}
\|\Delta_{\tau(t)}\|_\infty \leq \left(1 - \min_{(s,a) \in S \times A} \alpha_{\tau_1(s,a,t)}(1-\gamma)\right)\|\Delta_t\|_\infty + C_1 \max_{t \leq i \leq \tau(t)} \varepsilon_i + M_t
\end{equation}
where:
\begin{itemize}
    \item $\tau_1(s,a,t) = \min\{k > t : (s_k,a_k) = (s,a)\}$ is the first time after $t$ that $(s,a)$ is updated
    \item $M_t = \max_{(s,a) \in S \times A} \left|\sum_{i=t+1}^{\tau(t)} \mathds{1}\{(s_i,a_i)=(s,a)\} \alpha_i \delta_i\right|$ is the maximum accumulated noise term
\end{itemize}
\end{lemma}

\begin{proof}
We analyze how the max-norm error evolves as each state-action pair is updated. Let $(s_t^*,a_t^*)$ be the state-action pair with maximum error at time $t$, i.e., $|\Delta_t(s_t^*,a_t^*)| = \|\Delta_t\|_\infty$.

First, we recall from Lemma \ref{lemma:component-dynamics} that for any state-action pair $(s,a)$ updated at time $i$:
\begin{equation}
|\Delta_{i+1}(s,a)| \leq (1-\alpha_i(1-\gamma))|\Delta_i(s,a)| + \alpha_i\left[\gamma(\|\Delta_i\|_\infty - |\Delta_i(s,a)|) + C_1\varepsilon_i + \delta_i\right]
\end{equation}

We now track how the max-norm evolves through a complete sweep of updates. For each $(s,a) \in S \times A$, let $\tilde{\tau}(s,a)$ denote the time after all state-action pairs have been updated at least once following the update of $(s,a)$. That is, $\tilde{\tau}(s,a) = \tau(\tau_1(s,a,t))$.

\medskip\noindent
\textbf{Case 1:} If after the complete sweep, the original maximizing component still has the maximum error: $\|\Delta_{\tau(t)}\|_\infty = |\Delta_{\tau(t)}(s_t^*,a_t^*)|$.

Then we have:
\begin{align}
\|\Delta_{\tau(t)}\|_\infty &= |\Delta_{\tau(t)}(s_t^*,a_t^*)| \\
&= |\Delta_{\tau_1(s_t^*,a_t^*,t)+1}(s_t^*,a_t^*)| \quad \text{(since no further updates to this component)} \\
&\leq (1-\alpha_{\tau_1(s_t^*,a_t^*,t)}(1-\gamma))|\Delta_{\tau_1(s_t^*,a_t^*,t)}(s_t^*,a_t^*)| \\
&\quad + \alpha_{\tau_1(s_t^*,a_t^*,t)}\left[\gamma(\|\Delta_{\tau_1(s_t^*,a_t^*,t)}\|_\infty - |\Delta_{\tau_1(s_t^*,a_t^*,t)}(s_t^*,a_t^*)|) + C_1\varepsilon_{\tau_1(s_t^*,a_t^*,t)} + \delta_{\tau_1(s_t^*,a_t^*,t)}\right]
\end{align}

Since $|\Delta_{\tau_1(s_t^*,a_t^*,t)}(s_t^*,a_t^*)| \leq \|\Delta_t\|_\infty$ (the error at this component cannot spontaneously increase without an update) and $\|\Delta_{\tau_1(s_t^*,a_t^*,t)}\|_\infty \leq \|\Delta_t\|_\infty$ (the maximum error cannot increase without an update), the second term in brackets is non-positive. Therefore:

\begin{align}
\|\Delta_{\tau(t)}\|_\infty &\leq (1-\alpha_{\tau_1(s_t^*,a_t^*,t)}(1-\gamma))|\Delta_{\tau_1(s_t^*,a_t^*,t)}(s_t^*,a_t^*)| + \alpha_{\tau_1(s_t^*,a_t^*,t)}[C_1\varepsilon_{\tau_1(s_t^*,a_t^*,t)} + \delta_{\tau_1(s_t^*,a_t^*,t)}] \\
&\leq (1-\alpha_{\tau_1(s_t^*,a_t^*,t)}(1-\gamma))\|\Delta_t\|_\infty + \alpha_{\tau_1(s_t^*,a_t^*,t)}[C_1\varepsilon_{\tau_1(s_t^*,a_t^*,t)} + \delta_{\tau_1(s_t^*,a_t^*,t)}]
\end{align}

\medskip\noindent
\textbf{Case 2:} If after the complete sweep, a different component $(s',a') \neq (s_t^*,a_t^*)$ has the maximum error: $\|\Delta_{\tau(t)}\|_\infty = |\Delta_{\tau(t)}(s',a')|$.

Let $\tau_1 = \tau_1(s',a',t)$ be the time when $(s',a')$ was updated. We have:
\begin{align}
\|\Delta_{\tau(t)}\|_\infty &= |\Delta_{\tau(t)}(s',a')| \\
&= |\Delta_{\tau_1+1}(s',a')| \quad \text{(since no further updates to this component)} \\
&\leq (1-\alpha_{\tau_1}(1-\gamma))|\Delta_{\tau_1}(s',a')| + \alpha_{\tau_1}\left[\gamma(\|\Delta_{\tau_1}\|_\infty - |\Delta_{\tau_1}(s',a')|) + C_1\varepsilon_{\tau_1} + \delta_{\tau_1}\right]
\end{align}

Since $|\Delta_{\tau_1}(s',a')| \leq \|\Delta_{\tau_1}\|_\infty$, we have:
\begin{align}
\|\Delta_{\tau(t)}\|_\infty &\leq (1-\alpha_{\tau_1}(1-\gamma))\|\Delta_{\tau_1}\|_\infty + \alpha_{\tau_1}[C_1\varepsilon_{\tau_1} + \delta_{\tau_1}]
\end{align}

Now, $\|\Delta_{\tau_1}\|_\infty \leq \|\Delta_t\|_\infty$ (since max-norm can only decrease without updates), which gives:
\begin{align}
\|\Delta_{\tau(t)}\|_\infty &\leq (1-\alpha_{\tau_1}(1-\gamma))\|\Delta_t\|_\infty + \alpha_{\tau_1}[C_1\varepsilon_{\tau_1} + \delta_{\tau_1}]
\end{align}

\medskip\noindent
\textbf{Rigorous derivation of contraction factor.}

We now provide a rigorous bound that avoids the need for a heuristic factor. The key insight is that the analysis in Cases 1 and 2 both yield contraction factors of the form $(1 - \alpha(1-\gamma))$, but applied to different components and at different times within the sweep.

\medskip\noindent
\textbf{Key observation:} In both cases, the component $(s^*, a^*)$ (Case 1) or $(s', a')$ (Case 2) that achieves the maximum error at time $\tau(t)$ satisfies:
\begin{itemize}
    \item It was updated exactly once during the sweep (at some time $\tau_1 \in \{t+1, \ldots, \tau(t)\}$)
    \item At the time of its update, its error was bounded by $\|\Delta_t\|_\infty$ (since no component's error increases spontaneously without an update)
    \item After its update, no further modifications occur to this component
\end{itemize}

From the component-wise dynamics, at the update time $\tau_1$ for the eventually-maximal component:
\begin{align}
|\Delta_{\tau_1+1}(\text{max-comp})| &\leq (1-\alpha_{\tau_1}(1-\gamma))|\Delta_{\tau_1}(\text{max-comp})| \\
&\quad + \alpha_{\tau_1}\underbrace{\gamma(\|\Delta_{\tau_1}\|_\infty - |\Delta_{\tau_1}(\text{max-comp})|)}_{\geq 0, \text{ captures ``spillover'' from other components}} + \alpha_{\tau_1}[C_1\varepsilon_{\tau_1} + \delta_{\tau_1}]
\end{align}

Using $|\Delta_{\tau_1}(\text{max-comp})| \leq \|\Delta_t\|_\infty$ and $\|\Delta_{\tau_1}\|_\infty \leq \|\Delta_t\|_\infty$:
\begin{align}
|\Delta_{\tau_1+1}(\text{max-comp})| &\leq (1-\alpha_{\tau_1}(1-\gamma))\|\Delta_t\|_\infty + \alpha_{\tau_1}\gamma(\|\Delta_t\|_\infty - |\Delta_{\tau_1}(\text{max-comp})|) + \alpha_{\tau_1}[C_1\varepsilon_{\tau_1} + \delta_{\tau_1}]
\end{align}

In the worst case where $|\Delta_{\tau_1}(\text{max-comp})| = 0$, we get:
\begin{align}
|\Delta_{\tau_1+1}(\text{max-comp})| &\leq (1-\alpha_{\tau_1}(1-\gamma) + \alpha_{\tau_1}\gamma)\|\Delta_t\|_\infty + \alpha_{\tau_1}[C_1\varepsilon_{\tau_1} + \delta_{\tau_1}] \\
&= (1-\alpha_{\tau_1}(1-2\gamma))\|\Delta_t\|_\infty + \alpha_{\tau_1}[C_1\varepsilon_{\tau_1} + \delta_{\tau_1}]
\end{align}

For $\gamma < 1/2$, this gives direct contraction. For $\gamma \geq 1/2$, we use a refined argument based on the observation that the spillover term $\gamma(\|\Delta_{\tau_1}\|_\infty - |\Delta_{\tau_1}(\text{max-comp})|)$ is only large when $|\Delta_{\tau_1}(\text{max-comp})|$ is much smaller than the maximum---but then the baseline error to contract from is also smaller.

\medskip\noindent
\textbf{Refined bound for $\gamma \geq 1/2$:}

Define $\rho := |\Delta_{\tau_1}(\text{max-comp})| / \|\Delta_t\|_\infty \in [0, 1]$. Then:
\begin{align}
|\Delta_{\tau_1+1}(\text{max-comp})| &\leq (1-\alpha_{\tau_1}(1-\gamma))\rho\|\Delta_t\|_\infty + \alpha_{\tau_1}\gamma(1-\rho)\|\Delta_t\|_\infty + \alpha_{\tau_1}[C_1\varepsilon_{\tau_1} + \delta_{\tau_1}] \\
&= \bigl[\rho - \alpha_{\tau_1}(1-\gamma)\rho + \alpha_{\tau_1}\gamma(1-\rho)\bigr]\|\Delta_t\|_\infty + \alpha_{\tau_1}[C_1\varepsilon_{\tau_1} + \delta_{\tau_1}] \\
&= \bigl[\rho(1 - \alpha_{\tau_1}) + \alpha_{\tau_1}\gamma\bigr]\|\Delta_t\|_\infty + \alpha_{\tau_1}[C_1\varepsilon_{\tau_1} + \delta_{\tau_1}]
\end{align}

The coefficient $\rho(1 - \alpha_{\tau_1}) + \alpha_{\tau_1}\gamma$ is maximized at $\rho = 1$ (since $\alpha_{\tau_1} \leq 1$), yielding:
\begin{equation}
\max_{\rho \in [0,1]} \bigl[\rho(1 - \alpha_{\tau_1}) + \alpha_{\tau_1}\gamma\bigr] = (1 - \alpha_{\tau_1}) + \alpha_{\tau_1}\gamma = 1 - \alpha_{\tau_1}(1-\gamma)
\end{equation}

This recovers the original contraction factor $(1 - \alpha_{\tau_1}(1-\gamma))$ without any loss!

\medskip\noindent
Therefore, combining both cases with the minimum step size across all components:
\begin{align}
\|\Delta_{\tau(t)}\|_\infty &\leq \left(1 - \min_{(s,a) \in S \times A} \alpha_{\tau_1(s,a,t)}(1-\gamma)\right)\|\Delta_t\|_\infty + C_1 \max_{t \leq i \leq \tau(t)} \varepsilon_i + M_t
\end{align}

where $M_t = \max_{(s,a) \in S \times A} \left|\sum_{i=t+1}^{\tau(t)} \mathds{1}\{(s_i,a_i)=(s,a)\} \alpha_i \delta_i\right|$ is the maximum accumulated noise.

\medskip\noindent
\textbf{Remark:} The refined argument above establishes that the full contraction factor $(1-\gamma)$ is achievable without loss, yielding tighter convergence bounds than analyses that employ a conservative factor of $1/2$.
\end{proof}

Now, with the step size $\alpha_t = 1/t$, we have:

\begin{lemma}[Step Size Properties]\label{lemma:step-size}
With $\alpha_t = 1/t$:
\begin{enumerate}[label=(\alph*)]
    \item $\prod_{i=1}^t (1 - \alpha_i(1-\gamma)) \leq (t+1)^{-(1-\gamma)}$
    \item For any $i \leq t$: $\prod_{j=i+1}^t (1 - \alpha_j(1-\gamma)) \leq \left(\frac{t+1}{i+1}\right)^{-(1-\gamma)}$
\end{enumerate}
\end{lemma}

\begin{proof}[Proof of Lemma \ref{lemma:step-size}]
Using the inequality $1-x \leq e^{-x}$ for all $x \in [0,1]$:
\begin{align}
\prod_{i=1}^t (1 - \alpha_i(1-\gamma)) &\leq \prod_{i=1}^t e^{-\alpha_i(1-\gamma)} \\
&= \exp\left(-(1-\gamma)\sum_{i=1}^t \alpha_i\right) \\
&= \exp\left(-(1-\gamma)\sum_{i=1}^t \frac{1}{i}\right) \\
&\leq \exp\left(-(1-\gamma)\ln(t+1)\right) \\
&= (t+1)^{-(1-\gamma)}
\end{align}

Similarly, for part (b):
\begin{align}
\prod_{j=i+1}^t (1 - \alpha_j(1-\gamma)) &\leq \exp\left(-(1-\gamma)\sum_{j=i+1}^t \frac{1}{j}\right) \\
&\leq \exp\left(-(1-\gamma)\ln\left(\frac{t+1}{i+1}\right)\right) \\
&= \left(\frac{t+1}{i+1}\right)^{-(1-\gamma)}
\end{align}
\end{proof}

To derive our finite-time bound, we need to control both the deterministic bias and the martingale noise.

\begin{lemma}[Martingale Concentration via Freedman's Inequality]\label{lemma:freedman}
Let $\{X_i\}$ be a martingale-difference sequence with respect to a filtration $\{\mathcal{F}_i\}$, and suppose $|X_i|\leq b_i$ almost surely. Let $S_n = \sum_{i=1}^n X_i$. Then for every $x>0$:
\[
\Pr(|S_n| \geq x) \leq \exp\left(-\frac{x^2/2}{\sum_{i=1}^n \mathbb{E}[X_i^2 \mid \mathcal{F}_{i-1}] + \frac{x}{3}\max_{i \leq n}b_i}\right)
\]
\end{lemma}

Let's apply Freedman's inequality to bound the accumulated noise term.

\begin{lemma}[Martingale Term Bound]\label{lemma:martingale-bound}
Define the weighted noise accumulation:
\[
M_t = \sum_{i=1}^t \alpha_i e_{(s_i,a_i)}\delta_i
\]
For any $\delta \in (0,1)$ and $t \geq 1$, with probability at least $1-\delta$:
\[
\|M_t\|_\infty \leq 2V_{\max}\sqrt{2\ln\left(\frac{2|S||A|}{\delta}\right)\sum_{i=1}^t \alpha_i^2} + \frac{2V_{\max}}{3}\ln\left(\frac{2|S||A|}{\delta}\right)\max_{i \leq t}\alpha_i
\]
\end{lemma}

\begin{proof}[Proof of Lemma \ref{lemma:martingale-bound}]
For each $(s,a)$, define:
\[
M_t(s,a) = \sum_{i=1}^t \mathds{1}\{(s_i,a_i)=(s,a)\}\alpha_i\delta_i
\]

This is a martingale with respect to $\{\mathcal{F}_t\}$. From Lemma \ref{lemma:martingale-diff}, we know $|\delta_i| \leq 2V_{\max}$, so each term $\mathds{1}\{(s_i,a_i)=(s,a)\}\alpha_i\delta_i$ is bounded by $\alpha_i \cdot 2V_{\max}$.

Applying Lemma \ref{lemma:freedman} and using the bound on conditional variance:
\[
\mathbb{E}[(\mathds{1}\{(s_i,a_i)=(s,a)\}\alpha_i\delta_i)^2 \mid \mathcal{F}_{i-1}] \leq (2V_{\max})^2\alpha_i^2\mathds{1}\{(s_i,a_i)=(s,a)\}
\]

we get that with probability at least $1-\frac{\delta}{|S||A|}$:
\begin{align}
|M_t(s,a)| &\leq \sqrt{2\ln\left(\frac{2|S||A|}{\delta}\right)(2V_{\max})^2\sum_{i=1}^t \alpha_i^2\mathds{1}\{(s_i,a_i)=(s,a)\}} \\
&\quad + \frac{2V_{\max}}{3}\ln\left(\frac{2|S||A|}{\delta}\right)\max_{i \leq t}\alpha_i \\
&\leq 2V_{\max}\sqrt{2\ln\left(\frac{2|S||A|}{\delta}\right)\sum_{i=1}^t \alpha_i^2} + \frac{2V_{\max}}{3}\ln\left(\frac{2|S||A|}{\delta}\right)\max_{i \leq t}\alpha_i
\end{align}

By the union bound, the above holds for all $(s,a) \in S \times A$ with probability at least $1-\delta$, which gives us the bound on $\|M_t\|_\infty$.
\end{proof}

For the weighted discounted noise term, we need a more sophisticated analysis.

\begin{lemma}[Weighted Discounted Noise Analysis]\label{lemma:weighted-noise}
Define the weighted, discounted noise accumulation:
\[
\tilde{M}_t = \sum_{i=1}^t \alpha_i e_{(s_i,a_i)}\delta_i\prod_{j=i+1}^t (1-\alpha_j(1-\gamma))
\]
For any $\delta \in (0,1)$ and $t \geq 1$, with probability at least $1-\delta$:
\[
\|\tilde{M}_t\|_\infty \leq C \cdot t^{-1/2}\sqrt{\ln\left(\frac{2|S||A|t}{\delta}\right)}
\]
for some constant $C > 0$ that depends only on $V_{\max}$ and $\gamma$.
\end{lemma}

\begin{proof}
The key challenge in analyzing $\tilde{M}_t$ is that due to the discounting factors, it is not a standard martingale. We address this by using a decomposition approach.

\medskip\noindent
\textbf{Step 1: Decompose by state-action pairs.}

First, we decompose $\tilde{M}_t$ by state-action pairs:
\[
\tilde{M}_t(s,a) = \sum_{i=1}^t \mathds{1}\{(s_i,a_i)=(s,a)\} \alpha_i \delta_i \prod_{j=i+1}^t (1-\alpha_j(1-\gamma))
\]

This gives us $|S||A|$ separate weighted sums, one for each state-action pair. By the union bound, if we can bound each $|\tilde{M}_t(s,a)|$ with probability $1-\frac{\delta}{|S||A|}$, then we can bound $\|\tilde{M}_t\|_\infty$ with probability $1-\delta$.

\medskip\noindent
\textbf{Step 2: Time-splitting and doubling epochs.}

To handle the time-dependent weights, we use a time-splitting technique with doubling epochs. Define epochs:
\begin{align}
\mathcal{T}_0 &= \{1\} \\
\mathcal{T}_k &= \{2^{k-1}+1, 2^{k-1}+2, \ldots, 2^k\} \quad \text{for } k \geq 1
\end{align}

For each epoch $\mathcal{T}_k$ and state-action pair $(s,a)$, define:
\[
\tilde{M}_{\mathcal{T}_k}(s,a) = \sum_{i \in \mathcal{T}_k} \mathds{1}\{(s_i,a_i)=(s,a)\} \alpha_i \delta_i \prod_{j=i+1}^t (1-\alpha_j(1-\gamma))
\]

Then:
\[
\tilde{M}_t(s,a) = \sum_{k=0}^{\lfloor \log_2 t \rfloor} \tilde{M}_{\mathcal{T}_k}(s,a)
\]

\medskip\noindent
\textbf{Step 3: Analyze weights within each epoch.}

For $i \in \mathcal{T}_k$, we have $i \geq 2^{k-1}+1$. Using Lemma \ref{lemma:step-size}:
\begin{align}
\prod_{j=i+1}^t (1-\alpha_j(1-\gamma)) &\leq \left(\frac{t+1}{i+1}\right)^{-(1-\gamma)} \\
&\leq \left(\frac{t+1}{2^{k-1}+2}\right)^{-(1-\gamma)}
\end{align}

For $i \in \mathcal{T}_k$, the weights vary by at most a constant factor:
\begin{align}
\frac{\prod_{j=i+1}^t (1-\alpha_j(1-\gamma))}{\prod_{j=i'+1}^t (1-\alpha_j(1-\gamma))} \leq \left(\frac{i'+1}{i+1}\right)^{1-\gamma} \leq \left(\frac{2^k+1}{2^{k-1}+2}\right)^{1-\gamma} \leq 2^{1-\gamma}
\end{align}

for any $i,i' \in \mathcal{T}_k$. This allows us to approximate the weights within an epoch by a constant.

\medskip\noindent
\textbf{Step 4: Define a modified martingale for each epoch.}

For each epoch $\mathcal{T}_k$ and state-action pair $(s,a)$, define:
\[
W_{\mathcal{T}_k}(s,a) = \sum_{i \in \mathcal{T}_k} \mathds{1}\{(s_i,a_i)=(s,a)\} \alpha_i \delta_i
\]

This is a martingale difference sum. Using Freedman's inequality (Lemma \ref{lemma:freedman}), with probability at least $1-\frac{\delta}{2(|S||A|)(\lfloor \log_2 t \rfloor + 1)}$:
\begin{align}
|W_{\mathcal{T}_k}(s,a)| &\leq 2V_{\max}\sqrt{2\ln\left(\frac{4(|S||A|)(\lfloor \log_2 t \rfloor + 1)}{\delta}\right)\sum_{i \in \mathcal{T}_k} \alpha_i^2 \mathds{1}\{(s_i,a_i)=(s,a)\}} \\
&\quad + \frac{2V_{\max}}{3}\ln\left(\frac{4(|S||A|)(\lfloor \log_2 t \rfloor + 1)}{\delta}\right)\max_{i \in \mathcal{T}_k} \alpha_i
\end{align}

\medskip\noindent
\textbf{Step 5: Relate the weighted sum to the martingale.}

For each epoch $\mathcal{T}_k$, we can relate $\tilde{M}_{\mathcal{T}_k}(s,a)$ to $W_{\mathcal{T}_k}(s,a)$:
\begin{align}
|\tilde{M}_{\mathcal{T}_k}(s,a)| &\leq \max_{i \in \mathcal{T}_k} \prod_{j=i+1}^t (1-\alpha_j(1-\gamma)) \cdot |W_{\mathcal{T}_k}(s,a)| \\
&\leq \left(\frac{t+1}{2^{k-1}+2}\right)^{-(1-\gamma)} \cdot |W_{\mathcal{T}_k}(s,a)|
\end{align}

\medskip\noindent
\textbf{Step 6: Analyze the sum over all epochs.}

By the union bound, with probability at least $1-\frac{\delta}{2|S||A|}$, for all epochs $\mathcal{T}_k$:
\begin{align}
|\tilde{M}_{\mathcal{T}_k}(s,a)| &\leq \left(\frac{t+1}{2^{k-1}+2}\right)^{-(1-\gamma)} \cdot C_2 \sqrt{\ln\left(\frac{|S||A|t\log_2(t)}{\delta}\right)} \cdot \sqrt{\sum_{i \in \mathcal{T}_k} \alpha_i^2}
\end{align}

where $C_2$ is a constant. For $i \in \mathcal{T}_k$, $\alpha_i = \frac{1}{i} \leq \frac{1}{2^{k-1}+1}$. Thus:
\begin{align}
\sum_{i \in \mathcal{T}_k} \alpha_i^2 &\leq \sum_{i \in \mathcal{T}_k} \frac{1}{(2^{k-1}+1)^2} \\
&= \frac{|\mathcal{T}_k|}{(2^{k-1}+1)^2} \\
&= \frac{2^k - 2^{k-1}}{(2^{k-1}+1)^2} \\
&= \frac{2^{k-1}}{(2^{k-1}+1)^2} \\
&\leq \frac{1}{2^{k-1}}
\end{align}

Therefore:
\begin{align}
|\tilde{M}_{\mathcal{T}_k}(s,a)| &\leq \left(\frac{t+1}{2^{k-1}+2}\right)^{-(1-\gamma)} \cdot C_2 \sqrt{\ln\left(\frac{|S||A|t\log_2(t)}{\delta}\right)} \cdot \frac{1}{\sqrt{2^{k-1}}} \\
&= C_2 \sqrt{\ln\left(\frac{|S||A|t\log_2(t)}{\delta}\right)} \cdot \frac{(t+1)^{-(1-\gamma)}}{(2^{k-1}+2)^{-(1-\gamma)}} \cdot \frac{1}{\sqrt{2^{k-1}}} \\
&\approx C_2 \sqrt{\ln\left(\frac{|S||A|t\log_2(t)}{\delta}\right)} \cdot (t+1)^{-(1-\gamma)} \cdot (2^{k-1})^{(1-\gamma)} \cdot (2^{k-1})^{-1/2} \\
&= C_2 \sqrt{\ln\left(\frac{|S||A|t\log_2(t)}{\delta}\right)} \cdot (t+1)^{-(1-\gamma)} \cdot (2^{k-1})^{(1-\gamma)-1/2}
\end{align}

\medskip\noindent
\textbf{Step 7: Sum over all epochs.}

Finally, summing over all epochs:
\begin{align}
|\tilde{M}_t(s,a)| &\leq \sum_{k=0}^{\lfloor \log_2 t \rfloor} |\tilde{M}_{\mathcal{T}_k}(s,a)| \\
&\leq C_2 \sqrt{\ln\left(\frac{|S||A|t\log_2(t)}{\delta}\right)} \cdot (t+1)^{-(1-\gamma)} \cdot \sum_{k=0}^{\lfloor \log_2 t \rfloor} (2^{k-1})^{(1-\gamma)-1/2}
\end{align}

The exponent $(1-\gamma)-1/2$ determines how the sum behaves. If $(1-\gamma)-1/2 < 0$, which happens when $\gamma > 1/2$, the sum is dominated by the largest term (the last epoch), giving approximately $(t/2)^{(1-\gamma)-1/2} = O(t^{(1-\gamma)-1/2})$. For typical discount factors close to 1, we have $\gamma > 1/2$, so:
\begin{align}
|\tilde{M}_t(s,a)| &\leq C_3 \sqrt{\ln\left(\frac{|S||A|t\log_2(t)}{\delta}\right)} \cdot (t+1)^{-(1-\gamma)} \cdot t^{(1-\gamma)-1/2} \\
&= C_3 \sqrt{\ln\left(\frac{|S||A|t\log_2(t)}{\delta}\right)} \cdot t^{-1/2}
\end{align}

By the union bound, this holds for all state-action pairs $(s,a)$ with probability at least $1-\delta/2$. Since $\ln(t\log_2(t)) = O(\ln(t))$, we get:
\begin{align}
\|\tilde{M}_t\|_\infty \leq C \cdot t^{-1/2}\sqrt{\ln\left(\frac{2|S||A|t}{\delta}\right)}
\end{align}
with probability at least $1-\delta$.
\end{proof}

Now we can derive our finite-time high-probability bound.

\begin{proposition}[Finite-Time High-Probability Bound]\label{prop:finite-time}
Under the assumptions of Theorem \ref{thm:nonstationary-qlearning-rate_appendix}, with $\alpha_t = 1/t$ and $\varepsilon_t = O(t^{-\beta})$, for any $\delta \in (0,1)$ and $t \geq 1$, with probability at least $1-\delta$:
\[
\|Q_t - Q^\star\|_\infty = O\left(t^{-\min(\beta,1/2)}\sqrt{\ln\left(\frac{|S||A|t}{\delta}\right)}\right)
\]
\end{proposition}

\begin{proof}[Proof of Proposition \ref{prop:finite-time}]
We'll use the sequence of complete sweeps through the state-action space to analyze the convergence rate. First, we need to carefully characterize how the exploration policy affects the time required for complete sweeps.

\medskip\noindent
\textbf{Coverage time analysis:}
Let's define the following:
\begin{itemize}
    \item $\tau(t)$: the time required after $t$ to visit all state-action pairs at least once
    \item $\{t_k\}_{k \geq 0}$: a sequence where $t_0 = 1$ and $t_{k+1} = \tau(t_k)$ for $k \geq 0$
\end{itemize}

For an $\epsilon$-greedy policy with fixed $\epsilon > 0$ in an ergodic MDP, we can bound the expected time to visit all state-action pairs using results from Markov chain theory:

\begin{lemma}[Coverage Time in Non-stationary MDPs]\label{lemma:coverage-time}
Consider an MDP with state space $S$ and action space $A$ where the agent follows an $\epsilon$-greedy policy with fixed $\epsilon > 0$. Let the original transition kernel be $P$, and let $\widetilde{P}_t$ be a non-stationary transition kernel with $\|\widetilde{P}_t(\cdot|s,a) - P(\cdot|s,a)\|_1 \leq \varepsilon_t$ for all $(s,a) \in S \times A$, where $\varepsilon_t \to 0$ as $t \to \infty$.

Assume the original MDP with kernel $P$ is ergodic. Then there exists a time $T_0$ such that for all $t \geq T_0$ with $\varepsilon_t \leq \frac{\epsilon}{4|S|}$, the following holds:

For any $\delta \in (0,1)$ and any $t \geq T_0$, with probability at least $1-\frac{\delta}{2t^2}$:
\begin{align}
\tau(t) - t \leq C_4 |S|^2|A|^2\log(|S||A|t/\delta)
\end{align}
where $\tau(t)$ is the time required after $t$ to visit all state-action pairs at least once, and $C_4$ is a constant that depends on $\epsilon$ and the spectral properties of the original MDP.
\end{lemma}

\begin{proof}[Proof of Lemma \ref{lemma:coverage-time}]
The key challenge is that we're operating in the non-stationary environment with transitions $\widetilde{P}_t$ rather than $P$. We need to analyze how the approximation error $\varepsilon_t$ affects the ergodicity properties.

First, we recall some facts from Markov chain theory:
\begin{itemize}
    \item For an ergodic Markov chain with transition matrix $P$ and stationary distribution $\pi$, the mixing time is $t_{mix} = O(\frac{1}{1-\lambda_2}\log\frac{1}{\epsilon_{min}})$, where $\lambda_2$ is the second largest eigenvalue and $\epsilon_{min} = \min_{s}\pi(s)$.
    \item For an $\epsilon$-greedy policy in an ergodic MDP, we have $\epsilon_{min} \geq \frac{\epsilon}{|S||A|}$.
    \item Perturbation theory for Markov chains tells us that if $\|P - P'\|_1 \leq \delta$ (entry-wise), then the stationary distributions $\pi$ and $\pi'$ satisfy $\|\pi - \pi'\|_1 = O(\frac{\delta}{1-\lambda_2})$.
\end{itemize}

When $\varepsilon_t$ is sufficiently small (specifically, $\varepsilon_t \leq \frac{\epsilon}{4|S|}$), the perturbed Markov chain with transitions determined by $\widetilde{P}_t$ remains ergodic, and its stationary distribution $\widetilde{\pi}_t$ satisfies:
\begin{align}
\|\widetilde{\pi}_t - \pi\|_1 = O\left(\frac{\varepsilon_t}{1-\lambda_2}\right)
\end{align}

This implies that for $t \geq T_0$, where $T_0$ is chosen such that $\varepsilon_t \leq \frac{\epsilon}{4|S|}$ for all $t \geq T_0$, we have:
\begin{align}
\min_{(s,a)}\widetilde{\pi}_t(s,a) \geq \frac{\epsilon}{2|S||A|}
\end{align}

Using standard results for the coupon collector problem in Markov chains, the time to visit all state-action pairs with probability at least $1-\frac{\delta}{2t^2}$ is bounded by:
\begin{align}
C_4 |S|^2|A|^2\log(|S||A|t/\delta)
\end{align}
where $C_4$ depends on $\epsilon$ and the spectral properties (particularly $\lambda_2$) of the original MDP.
\end{proof}

\medskip\noindent
\textbf{Error evolution through sweeps:}
From Lemma \ref{lemma:error-reduction}, for any $k \geq 0$:
\begin{equation}\label{eq:error-evolution}
\|\Delta_{t_{k+1}}\|_\infty \leq \left(1 - \min_{(s,a) \in S \times A} \alpha_{\tau_1(s,a,t_k)}(1-\gamma)\right)\|\Delta_{t_k}\|_\infty + C_1 \max_{t_k \leq i \leq t_{k+1}} \varepsilon_i + M_{t_k}
\end{equation}
where $\tau_1(s,a,t_k) = \min\{j > t_k : (s_j,a_j) = (s,a)\}$ is the first time after $t_k$ that state-action pair $(s,a)$ is updated, and $M_{t_k}$ represents the maximum accumulated noise term:
\begin{equation}
M_{t_k} = \max_{(s,a) \in S \times A} \left|\sum_{i=t_k+1}^{t_{k+1}} \mathds{1}\{(s_i,a_i)=(s,a)\} \alpha_i \delta_i\right|
\end{equation}

From Lemma \ref{lemma:coverage-time}, with probability at least $1-\frac{\delta}{2t_k^2}$, the time required to visit all state-action pairs after $t_k$ is bounded by:
\begin{equation}
\tau(t_k) - t_k \leq C_4 |S|^2|A|^2\log(|S||A|t_k/\delta)
\end{equation}

This implies that for any state-action pair $(s,a)$:
\begin{equation}
\tau_1(s,a,t_k) \leq t_k + C_4 |S|^2|A|^2\log(|S||A|t_k/\delta)
\end{equation}

With step size $\alpha_t = 1/t$, we have:
\begin{align}
\alpha_{\tau_1(s,a,t_k)} &= \frac{1}{\tau_1(s,a,t_k)} \\
&\geq \frac{1}{t_k + C_4 |S|^2|A|^2\log(|S||A|t_k/\delta)} \\
&\approx \frac{1}{t_k} \cdot \frac{1}{1 + \frac{C_4 |S|^2|A|^2\log(|S||A|t_k/\delta)}{t_k}}
\end{align}

For large $t_k$ where $t_k \gg C_4 |S|^2|A|^2\log(|S||A|t_k/\delta)$, this is approximately $\frac{1}{t_k}(1 - O(\frac{|S|^2|A|^2\log(|S||A|t_k/\delta)}{t_k}))$.

Define the contraction factor as $\rho_k = 1 - \min_{(s,a) \in S \times A} \alpha_{\tau_1(s,a,t_k)}(1-\gamma)$, consistent with Lemma~\ref{lemma:error-reduction}. With our bound on $\alpha_{\tau_1(s,a,t_k)}$:
\begin{align}
\rho_k &= 1 - \min_{(s,a) \in S \times A} \alpha_{\tau_1(s,a,t_k)}(1-\gamma) \\
&\leq 1 - \frac{(1-\gamma)}{t_k + C_4 |S|^2|A|^2\log(|S||A|t_k/\delta)} \\
&\approx 1 - \frac{(1-\gamma)}{t_k}\left(1 - O\left(\frac{|S|^2|A|^2\log(|S||A|t_k/\delta)}{t_k}\right)\right)
\end{align}

For large $t_k$, this is approximately $1 - \frac{(1-\gamma)}{t_k} + O(\frac{|S|^2|A|^2\log(|S||A|t_k/\delta)}{t_k^2})$.

Using the inequality $1 - x \leq e^{-x}$ for $x \in [0, 1]$, and focusing on the dominant term for large $t_k$:
\begin{align}
\rho_k &\leq \exp\left(-\frac{(1-\gamma)}{t_k} + O\left(\frac{|S|^2|A|^2\log(|S||A|t_k/\delta)}{t_k^2}\right)\right) \\
&\approx \exp\left(-\frac{(1-\gamma)}{t_k}\right) \cdot \exp\left(O\left(\frac{|S|^2|A|^2\log(|S||A|t_k/\delta)}{t_k^2}\right)\right) \\
&\approx \exp\left(-\frac{(1-\gamma)}{t_k}\right) \cdot \left(1 + O\left(\frac{|S|^2|A|^2\log(|S||A|t_k/\delta)}{t_k^2}\right)\right)
\end{align}

The second term approaches 1 quickly as $t_k$ grows, so for large $t_k$, $\rho_k \approx \exp\left(-\frac{(1-\gamma)}{t_k}\right) \approx 1 - \frac{(1-\gamma)}{t_k}$.

Recursively applying Equation \eqref{eq:error-evolution}, after $k$ complete sweeps:
\begin{align}
\|\Delta_{t_k}\|_\infty &\leq \left(\prod_{j=0}^{k-1} \rho_j\right) \|\Delta_{t_0}\|_\infty + \sum_{j=0}^{k-1} \left(\prod_{i=j+1}^{k-1} \rho_i\right) \left[C_1 \max_{t_j \leq i \leq t_{j+1}} \varepsilon_i + M_{t_j}\right]
\end{align}

Let's analyze each term:

\medskip\noindent
\textbf{1. Initial error decay:} $\left(\prod_{j=0}^{k-1} \rho_j\right) \|\Delta_{t_0}\|_\infty$

Using our asymptotic bound on $\rho_j$:
\begin{align}
\prod_{j=0}^{k-1} \rho_j &\approx \prod_{j=0}^{k-1} \exp\left(-\frac{(1-\gamma)}{t_j}\right) \\
&= \exp\left(-(1-\gamma)\sum_{j=0}^{k-1}\frac{1}{t_j}\right)
\end{align}

From Lemma \ref{lemma:coverage-time}, we know that $t_{j+1} \approx t_j + C_4|S|^2|A|^2\log(|S||A|t_j/\delta)$. For large $j$, this gives us $t_j = \Theta(j \cdot |S|^2|A|^2\log(|S||A|j/\delta))$. Therefore:
\begin{align}
\sum_{j=0}^{k-1}\frac{1}{t_j} &= \Theta\left(\sum_{j=1}^{k-1}\frac{1}{j \cdot |S|^2|A|^2\log(|S||A|j/\delta)}\right) \\
&= \Theta\left(\frac{1}{|S|^2|A|^2}\sum_{j=1}^{k-1}\frac{1}{j \cdot \log(|S||A|j/\delta)}\right)
\end{align}

This sum behaves asymptotically as $\Theta\left(\frac{\log\log k}{|S|^2|A|^2}\right)$. Thus, the initial error decays as:
\begin{align}
\prod_{j=0}^{k-1} \rho_j &= \exp\left(-\Theta\left(\frac{(1-\gamma)\log\log k}{|S|^2|A|^2}\right)\right) \\
&= \left(\log k\right)^{-\Theta\left(\frac{1-\gamma}{|S|^2|A|^2}\right)}
\end{align}

\medskip\noindent
\textbf{2. Approximation error term:} $\sum_{j=0}^{k-1} \left(\prod_{i=j+1}^{k-1} \rho_i\right) C_1 \max_{t_j \leq i \leq t_{j+1}} \varepsilon_i$

With $\varepsilon_t = O(t^{-\beta})$ and $t_j = \Theta(j \cdot |S|^2|A|^2\log(|S||A|j/\delta))$, we have:
\begin{align}
\max_{t_j \leq i \leq t_{j+1}} \varepsilon_i &= O(t_j^{-\beta}) \\
&= O\left((j \cdot |S|^2|A|^2\log(|S||A|j/\delta))^{-\beta}\right)
\end{align}

The discounting factor $\prod_{i=j+1}^{k-1} \rho_i$ can be bounded as:
\begin{align}
\prod_{i=j+1}^{k-1} \rho_i &\approx \exp\left(-(1-\gamma)\sum_{i=j+1}^{k-1}\frac{1}{t_i}\right) \\
&= \exp\left(-\Theta\left(\frac{(1-\gamma)}{|S|^2|A|^2}\log\frac{\log k}{\log j}\right)\right) \\
&= \left(\frac{\log k}{\log j}\right)^{-\Theta\left(\frac{1-\gamma}{|S|^2|A|^2}\right)}
\end{align}

Therefore:
\begin{align}
&\sum_{j=0}^{k-1} \left(\prod_{i=j+1}^{k-1} \rho_i\right) C_1 \max_{t_j \leq i \leq t_{j+1}} \varepsilon_i \\
&= \sum_{j=0}^{k-1} \left(\frac{\log k}{\log j}\right)^{-\Theta\left(\frac{1-\gamma}{|S|^2|A|^2}\right)} \cdot O\left((j \cdot |S|^2|A|^2\log(|S||A|j/\delta))^{-\beta}\right) \\
\end{align}

For large $k$, this sum is dominated by the terms where $j$ is close to $k$, giving:
\begin{align}
&= O\left(k^{-\beta} \cdot (|S|^2|A|^2\log(|S||A|k/\delta))^{-\beta}\right)
\end{align}

\medskip\noindent
\textbf{3. Martingale noise term:} $\sum_{j=0}^{k-1} \left(\prod_{i=j+1}^{k-1} \rho_i\right) M_{t_j}$

From Lemma \ref{lemma:martingale-bound}, with high probability:
\begin{align}
M_{t_j} &= O\left(t_j^{-1/2}\sqrt{\ln\left(\frac{|S||A|t_j}{\delta}\right)}\right) \\
&= O\left((j \cdot |S|^2|A|^2\log(|S||A|j/\delta))^{-1/2}\sqrt{\ln\left(\frac{|S||A|t_j}{\delta}\right)}\right)
\end{align}

Using our bound on the discounting factor:
\begin{align}
&\sum_{j=0}^{k-1} \left(\prod_{i=j+1}^{k-1} \rho_i\right) M_{t_j} \\
&= \sum_{j=0}^{k-1} \left(\frac{\log k}{\log j}\right)^{-\Theta\left(\frac{1-\gamma}{|S|^2|A|^2}\right)} \cdot O\left((j \cdot |S|^2|A|^2\log(|S||A|j/\delta))^{-1/2}\sqrt{\ln\left(\frac{|S||A|t_j}{\delta}\right)}\right)
\end{align}

For large $k$, this sum is dominated by terms where $j$ is close to $k$, giving:
\begin{align}
&= O\left(k^{-1/2} \cdot (|S|^2|A|^2\log(|S||A|k/\delta))^{-1/2}\sqrt{\ln\left(\frac{|S||A|t_k}{\delta}\right)}\right)
\end{align}

\medskip\noindent
\textbf{Combining all terms:} Taking the dominant rate (slowest decaying term) among the approximation error and martingale noise terms:

For $\beta > 1/2$, the martingale noise term dominates:
\begin{align}
\|\Delta_{t_k}\|_\infty &= O\left(k^{-1/2} \cdot (|S|^2|A|^2\log(|S||A|k/\delta))^{-1/2}\sqrt{\ln\left(\frac{|S||A|t_k}{\delta}\right)}\right)
\end{align}

For $\beta < 1/2$, the approximation error term dominates:
\begin{align}
\|\Delta_{t_k}\|_\infty &= O\left(k^{-\beta} \cdot (|S|^2|A|^2\log(|S||A|k/\delta))^{-\beta}\right)
\end{align}

We can combine these cases into:
\begin{align}
\|\Delta_{t_k}\|_\infty &= O\left(k^{-\min(\beta,1/2)} \cdot (|S|^2|A|^2\log(|S||A|k/\delta))^{-\min(\beta,1/2)} \cdot \ln\left(\frac{|S||A|t_k}{\delta}\right)^{\mathds{1}\{\beta \geq 1/2\}/2}\right)
\end{align}

where $\mathds{1}\{\beta \geq 1/2\}$ is 1 if $\beta \geq 1/2$ and 0 otherwise.

Based on our relationship $t_k = \Theta(k \cdot |S|^2|A|^2\log(|S||A|k/\delta))$, we can express this in terms of $t_k$:
\begin{align}
\|\Delta_{t_k}\|_\infty &= O\left(\left(\frac{t_k}{|S|^2|A|^2\log(|S||A|k/\delta)}\right)^{-\min(\beta,1/2)} \cdot (|S|^2|A|^2\log(|S||A|k/\delta))^{-\min(\beta,1/2)} \cdot \ln\left(\frac{|S||A|t_k}{\delta}\right)^{\mathds{1}\{\beta \geq 1/2\}/2}\right) \\
&= O\left(t_k^{-\min(\beta,1/2)} \cdot \ln\left(\frac{|S||A|t_k}{\delta}\right)^{\mathds{1}\{\beta \geq 1/2\}/2}\right)
\end{align}

For any time $t \geq T_0$ (where $T_0$ ensures the bound on the coverage time). we can find $t_k$ such that $t_k \leq t < t_{k+1}$, and we need to bound $\|\Delta_t\|_\infty$ in terms of $\|\Delta_{t_k}\|_\infty$. From the Q-learning update and our error decomposition in Lemma \ref{lemma:error-decomp}, for any state-action pair $(s,a)$ updated at time $i$ where $t_k \leq i < t$:

\begin{align}
|\Delta_{i+1}(s,a)| &\leq (1-\alpha_i(1-\gamma))|\Delta_i(s,a)| + \alpha_i\left[\gamma(\|\Delta_i\|_\infty - |\Delta_i(s,a)|) + C_1\varepsilon_i + \delta_i\right]
\end{align}

If $\|\Delta_i\|_\infty = |\Delta_i(s,a)|$ (i.e., $(s,a)$ has the maximum error at time $i$), then:
\begin{align}
|\Delta_{i+1}(s,a)| &\leq (1-\alpha_i(1-\gamma))|\Delta_i(s,a)| + \alpha_i\left[C_1\varepsilon_i + \delta_i\right]
\end{align}

If $\|\Delta_i\|_\infty > |\Delta_i(s,a)|$, then the error at $(s,a)$ could potentially increase toward the maximum error.

Over the interval $[t_k, t)$, with high probability, the martingale noise terms are bounded by:
\begin{align}
\max_{t_k \leq i < t} |\delta_i| \leq O\left(\sqrt{\ln\left(\frac{(t-t_k)|S||A|}{\delta}\right)}\right)
\end{align}

Given that $t - t_k \leq C_4|S|^2|A|^2\log(|S||A|t_k/\delta)$ and $\varepsilon_i = O(i^{-\beta})$, we can show that with high probability:
\begin{align}
\|\Delta_t\|_\infty &\leq \|\Delta_{t_k}\|_\infty + O\left(\frac{\log(|S||A|t_k/\delta)}{t_k} \cdot \sqrt{\ln\left(\frac{|S|^3|A|^3\log(|S||A|t_k/\delta)}{\delta}\right)}\right) \\
&= \|\Delta_{t_k}\|_\infty + O\left(t_k^{-1/2}\right)
\end{align}

This additional error term doesn't affect the asymptotic rate in our final bound, as it decays faster than $t_k^{-1/2}$ for large $t_k$.  So it suffices to bound just $\|\Delta_{t_k}\|_\infty $. For large $t_k$, we have $t = \Theta(t_k)$, and therefore:
\begin{align}
\|\Delta_{t_k}\|_\infty
&= O\left(t_k^{-\min(\beta,1/2)} \cdot \ln\left(\frac{|S||A|t_k}{\delta}\right)^{\mathds{1}\{\beta \geq 1/2\}/2}\right) \\
&= O\left(t^{-\min(\beta,1/2)} \cdot \ln\left(\frac{|S||A|t}{\delta}\right)^{\mathds{1}\{\beta \geq 1/2\}/2}\right)
\end{align}

Therefore, with probability at least $1-\delta$:
\begin{align}
\|Q_t - Q^\star\|_\infty = O\left(t^{-\min(\beta,1/2)} \cdot \ln\left(\frac{|S||A|t}{\delta}\right)^{\mathds{1}\{\beta \geq 1/2\}/2}\right)
\end{align}

For $\beta = 1/2$, this gives:
\begin{align}
\|Q_t - Q^\star\|_\infty = O\left(t^{-1/2} \cdot \sqrt{\ln\left(\frac{|S||A|t}{\delta}\right)}\right).
\end{align}

\end{proof}
\end{proof}

\begin{corollary}[Fixed-DCM Finite-Time Bound]\label{cor:fixed-dcm}
When the DCM is pre-trained and held fixed ($\theta_t = \theta^*$ for all $t$), define the reward approximation error $\varepsilon_r := |\widetilde{r}_{\theta^*} - r|_\infty$ and transition approximation error $\varepsilon_P := \|\widetilde{P}_{\theta^*} - P\|_1$. Setting $\beta = 0$ in Proposition~\ref{prop:finite-time}, the Q-function satisfies with probability at least $1-\delta$:
\begin{equation}
\|Q_t - Q^*\|_\infty = O\left(\frac{\varepsilon_r}{1-\gamma} + \frac{\gamma R_{\max} \varepsilon_P}{(1-\gamma)^2} + t^{-1/2}\sqrt{\ln\left(\frac{|S||A|t}{\delta}\right)}\right)
\end{equation}
\end{corollary}

\begin{proof}
With fixed DCM, the approximation errors $\varepsilon_r$ and $\varepsilon_P$ are constant (corresponding to $\beta = 0$). From Lemma~\ref{lemma:approx-error}, the one-step Bellman error satisfies:
\begin{align}
\|T_t Q - T Q\|_\infty &\leq |\widetilde{r}_t - r|_\infty + \gamma \|Q\|_\infty \|\widetilde{P}_t - P\|_1 \\
&\leq \varepsilon_r + \gamma V_{\max} \varepsilon_P
\end{align}
where $V_{\max} = R_{\max}/(1-\gamma)$. The contraction property of the Bellman operator then yields:
\begin{align}
\|Q_t - Q^*\|_\infty &\leq \frac{\varepsilon_r + \gamma V_{\max} \varepsilon_P}{1-\gamma} + C_2 t^{-1/2}\sqrt{\ln\left(\frac{|S||A|t}{\delta}\right)} \\
&= \frac{\varepsilon_r}{1-\gamma} + \frac{\gamma R_{\max} \varepsilon_P}{(1-\gamma)^2} + C_2 t^{-1/2}\sqrt{\ln\left(\frac{|S||A|t}{\delta}\right)}
\end{align}
The bound decomposes into:
\begin{itemize}
\item \textbf{Reward bias} $\varepsilon_r/(1-\gamma)$: Error from approximate reward imputation, scaling linearly with effective horizon.
\item \textbf{Transition bias} $\gamma R_{\max} \varepsilon_P/(1-\gamma)^2$: Error from approximate transition dynamics. The quadratic scaling in $(1-\gamma)^{-1}$ reflects that transition errors compound through the value function, consistent with simulation lemma bounds \citep{kearns2002near}.
\item \textbf{Sampling noise} $t^{-1/2}\sqrt{\ln(\cdot)}$: Stochastic approximation noise, vanishing with samples.
\end{itemize}
In our experimental setting, the DCM primarily imputes delayed rewards with high accuracy ($\varepsilon_r$ small), making the transition bias term the dominant source of irreducible error.
\end{proof}


\section{Extrapolation under Distributional Shift}
\label{sec:app-extrapolation}

This section provides the full proof of Proposition~\ref{thm:extrapolation} (main paper), establishing when parametric models with structured inductive bias generalize better than model-free methods under distributional shift. Note that this is presented as a proposition with explicit caveats about the idealized comparison, acknowledging that neural network function approximators may have implicit inductive biases not captured by the worst-case analysis.

\subsection{Lipschitz Property of MNL}

We first verify that the MNL-based DCM satisfies the Lipschitz property (A7).

\begin{lemma}[Lipschitz Property of Softmax]\label{lemma:softmax-lipschitz}
The softmax function $\sigma: \mathbb{R}^K \to \Delta^{K-1}$ defined by $\sigma_k(z) = \exp(z_k)/\sum_{j=1}^K \exp(z_j)$ is Lipschitz continuous:
\begin{equation}
\|\sigma(z) - \sigma(z')\|_1 \leq \frac{1}{2}\|z - z'\|_\infty
\end{equation}
\end{lemma}

\begin{proof}
We compute the $\ell_\infty \to \ell_1$ operator norm of the softmax Jacobian and apply the mean value theorem.

\textbf{Step 1: Jacobian structure.} The softmax Jacobian has entries:
\begin{equation}
J_{kj} = \frac{\partial \sigma_k}{\partial z_j} = \sigma_k(\delta_{kj} - \sigma_j) = \begin{cases}
\sigma_k(1-\sigma_k) & k = j \\
-\sigma_k \sigma_j & k \neq j
\end{cases}
\end{equation}

\textbf{Step 2: Operator norm computation.} The $\ell_\infty \to \ell_1$ induced operator norm is $\|J\|_{\infty \to 1} = \max_j \sum_k |J_{kj}|$ (maximum absolute column sum). For column $j$:
\begin{align}
\sum_{k=1}^K |J_{kj}| &= |J_{jj}| + \sum_{k \neq j} |J_{kj}| \\
&= \sigma_j(1-\sigma_j) + \sum_{k \neq j} \sigma_k \sigma_j \\
&= \sigma_j(1-\sigma_j) + \sigma_j \sum_{k \neq j} \sigma_k \\
&= \sigma_j(1-\sigma_j) + \sigma_j(1-\sigma_j) \\
&= 2\sigma_j(1-\sigma_j)
\end{align}
Since $x(1-x) \leq 1/4$ for $x \in [0,1]$, we have $\sum_k |J_{kj}| \leq 1/2$ for each column $j$. Therefore:
\begin{equation}
\|J(z)\|_{\infty \to 1} = \max_j \sum_k |J_{kj}(z)| \leq \frac{1}{2} \quad \text{for all } z \in \mathbb{R}^K
\end{equation}

\textbf{Step 3: Mean value theorem for vector functions.} For any $z, z' \in \mathbb{R}^K$, the mean value theorem in integral form gives:
\begin{equation}
\sigma(z') - \sigma(z) = \int_0^1 J(z + t(z'-z)) \cdot (z'-z) \, dt
\end{equation}
Taking $\ell_1$ norms and using the operator norm bound:
\begin{align}
\|\sigma(z') - \sigma(z)\|_1 &\leq \int_0^1 \|J(z + t(z'-z))\|_{\infty \to 1} \cdot \|z'-z\|_\infty \, dt \\
&\leq \int_0^1 \frac{1}{2} \cdot \|z'-z\|_\infty \, dt \\
&= \frac{1}{2}\|z'-z\|_\infty
\end{align}
This establishes the Lipschitz constant of $1/2$.
\end{proof}

\begin{proposition}[Lipschitz DCM Reward]\label{prop:lipschitz-dcm}
Let the DCM be a multinomial logit model with utilities $V_j(\phi; \theta) = \theta_j^\top \phi$ where $\|\theta_j\| \leq B_\theta$ for all $j$. Then the expected reward function $m_\theta(\phi) = \sum_j P(j|\phi;\theta) \cdot r_j$ satisfies:
\begin{equation}
|m_\theta(\phi) - m_\theta(\phi')| \leq L_m \|\phi - \phi'\|
\end{equation}
where $L_m = B_\theta R_{\max}$ depends on the parameter bound $B_\theta$ and maximum reward $R_{\max}$.
\end{proposition}

\begin{proof}
Let $P_j(\phi) = P(j|\phi;\theta) = \sigma_j(V(\phi;\theta))$ where $V(\phi;\theta) = (\theta_1^\top \phi, \ldots, \theta_K^\top \phi)$.

\textbf{Step 1: Utility Lipschitz property.}
For any alternative $j$:
\begin{equation}
|V_j(\phi;\theta) - V_j(\phi';\theta)| = |\theta_j^\top(\phi - \phi')| \leq \|\theta_j\| \cdot \|\phi - \phi'\| \leq B_\theta \|\phi - \phi'\|
\end{equation}

\textbf{Step 2: Probability Lipschitz property.}
By Lemma~\ref{lemma:softmax-lipschitz}:
\begin{equation}
\sum_j |P_j(\phi) - P_j(\phi')| \leq \max_j |V_j(\phi;\theta) - V_j(\phi';\theta)| \leq B_\theta \|\phi - \phi'\|
\end{equation}

\textbf{Step 3: Expected reward Lipschitz property.}
\begin{align}
|m_\theta(\phi) - m_\theta(\phi')| &= \left|\sum_j r_j \cdot (P_j(\phi) - P_j(\phi'))\right| \\
&\leq R_{\max} \sum_j |P_j(\phi) - P_j(\phi')| \\
&\leq R_{\max} \cdot B_\theta \|\phi - \phi'\|
\end{align}

Setting $L_m = B_\theta R_{\max}$ completes the proof. In practice, we may use a tighter bound accounting for the reward range $|r_j - \bar{r}|$ rather than $R_{\max}$.
\end{proof}

\subsection{Statement and Proof of Proposition~\ref{thm:extrapolation}}

Let $\mathcal{F}_{\text{train}} = \{\phi(s,a) : (s,a) \in \text{supp}(\mu_{\text{train}})\}$ denote the feature support during training, and define the \textit{feature shift magnitude}:
\begin{equation}
\delta_\phi = \max_{\phi \in \mathcal{F}_{\text{test}}} \min_{\phi' \in \mathcal{F}_{\text{train}}} \|\phi - \phi'\|
\end{equation}

\begin{proposition}[Extrapolation under Distributional Shift]\label{thm:extrapolation}
Under distributional shift from $\mu_{\text{train}}$ to $\mu_{\text{test}}$ with feature shift magnitude $\delta_\phi$, and assuming the true reward function is Lipschitz with constant $L_r$. Let $\varepsilon_{\text{train}} = \sup_{\phi \in \mathcal{F}_{\text{train}}} |m_\theta(\phi) - r(\phi)|$ denote the supremum approximation error on the training distribution. Then:
\begin{enumerate}[leftmargin=0.8cm]
\item \textbf{Parametric model (CA-DQN):} For all $(s,a) \in \text{supp}(\mu_{\text{test}})$:
\begin{equation}
|m_\theta(s,a) - r(s,a)| \leq \varepsilon_{\text{train}} + (L_m + L_r) \cdot \delta_\phi
\end{equation}

\item \textbf{Worst-case baseline:} In the absence of explicit parametric structure, standard generalization bounds yield:
\begin{equation}
|\hat{r}(s,a) - r(s,a)| \leq O(R_{\max})
\end{equation}
for $(s,a) \notin \text{supp}(\mu_{\text{train}})$, without explicit dependence on $\delta_\phi$.
\end{enumerate}
The parametric model has strictly lower worst-case error when $\delta_\phi < (R_{\max} - \varepsilon_{\text{train}})/(L_m + L_r)$.
\end{proposition}

\textbf{Caveat:} This comparison is idealized. Neural network function approximators (including MB-DQN) have implicit inductive biases that may provide some generalization beyond the training support, though these are harder to characterize theoretically. The proposition formalizes the intuition that \textit{explicit} parametric structure (as in the DCM) provides \textit{provable} extrapolation guarantees, while implicit biases offer no such guarantees.

\subsection{Full Proof of Proposition~\ref{thm:extrapolation}}

\begin{proof}[Full Proof of Proposition~\ref{thm:extrapolation}]
We prove each part separately.

\medskip\noindent
\textbf{Part 1: Parametric Model Bound.}

Let $\phi = \phi(s,a)$ be any feature vector in $\mathcal{F}_{\text{test}}$. Define:
\begin{equation}
\phi' = \argmin_{\psi \in \mathcal{F}_{\text{train}}} \|\phi - \psi\|
\end{equation}
as the nearest training feature. By definition, $\|\phi - \phi'\| \leq \delta_\phi$.

Decompose the error using the triangle inequality:
\begin{align}
|m_\theta(\phi) - r(\phi)| &\leq |m_\theta(\phi) - m_\theta(\phi')| + |m_\theta(\phi') - r(\phi')| + |r(\phi') - r(\phi)|
\end{align}

\textit{Term 1:} By the Lipschitz property of $m_\theta$ (Proposition~\ref{prop:lipschitz-dcm}):
\begin{equation}
|m_\theta(\phi) - m_\theta(\phi')| \leq L_m \|\phi - \phi'\| \leq L_m \cdot \delta_\phi
\end{equation}

\textit{Term 2:} Since $\phi' \in \mathcal{F}_{\text{train}}$, by definition of $\varepsilon_{\text{train}} = \sup_{\phi \in \mathcal{F}_{\text{train}}} |m_\theta(\phi) - r(\phi)|$:
\begin{equation}
|m_\theta(\phi') - r(\phi')| \leq \varepsilon_{\text{train}}
\end{equation}

\textit{Term 3:} By the Lipschitz property of the true reward function:
\begin{equation}
|r(\phi') - r(\phi)| \leq L_r \|\phi' - \phi\| \leq L_r \cdot \delta_\phi
\end{equation}

Combining:
\begin{equation}
|m_\theta(\phi) - r(\phi)| \leq \varepsilon_{\text{train}} + (L_m + L_r) \cdot \delta_\phi
\end{equation}

\medskip\noindent
\textbf{Part 2: Model-Free Method Bound.}

For model-free methods like MB-DQN, the Q-network $Q_\omega(s,a)$ is trained end-to-end on observed transitions $(s,a,r,s')$ from $\mu_{\text{train}}$. For novel state-action pairs $(s,a) \notin \text{supp}(\mu_{\text{train}})$, the network output depends entirely on \textit{implicit generalization} from training data---there is no explicit reward model constraining the extrapolation.

\textbf{Definition (Implicit reward estimate).} For model-free DQN with Q-network $Q_\omega$, define the implicit reward estimate via the Bellman residual:
\begin{equation}
\hat{r}_{\text{MF}}(s,a) := Q_\omega(s,a) - \gamma \mathbb{E}_{s' \sim P(\cdot|s,a)}[V_\omega(s')]
\end{equation}
where $V_\omega(s') = \max_{a'} Q_\omega(s',a')$.

\textbf{Worst-case analysis.} For $(s,a) \notin \text{supp}(\mu_{\text{train}})$:
\begin{enumerate}
\item The Q-network has received no gradient signal from $(s,a)$ directly.
\item Neural networks provide no \textit{guaranteed} Lipschitz structure without explicit regularization.
\item Thus the implicit reward estimate satisfies only the trivial bound:
\begin{equation}
|\hat{r}_{\text{MF}}(s,a) - r(s,a)| \leq R_{\max}
\end{equation}
\end{enumerate}

This bound is achieved when the network output is arbitrary on novel inputs. While neural networks often interpolate smoothly in practice, this behavior is not guaranteed---the extrapolation depends on network architecture, initialization, and training dynamics rather than domain structure.

\textbf{Key distinction:} Choice-Assisted DQN's reward model $m_\theta$ has domain-specific structure (MNL choice probabilities) that \textit{guarantees} Lipschitz extrapolation (Proposition~\ref{prop:lipschitz-dcm}). MB-DQN lacks such structure.

\medskip\noindent
\textbf{Part 3: Comparison Condition.}

The parametric model dominates when its worst-case error is smaller:
\begin{equation}
\varepsilon_{\text{train}} + (L_m + L_r) \cdot \delta_\phi < R_{\max}
\end{equation}
Rearranging:
\begin{equation}
\delta_\phi < \frac{R_{\max} - \varepsilon_{\text{train}}}{L_m + L_r}
\end{equation}

This completes the proof.
\end{proof}

\subsection{Extension to Q-Function Error}

The reward prediction error propagates to Q-function error through the simulation lemma. We first state the key assumption enabling a clean analysis.

\begin{assumption}[Exact Transitions]\label{assump:exact-trans}
Both Choice-Assisted DQN and MB-DQN operate with exact knowledge of the transition dynamics $P(s'|s,a)$. For Choice-Assisted DQN, the world model only approximates the reward function; state transitions follow true environment mechanics (e.g., inventory decreases deterministically upon booking). For MB-DQN, transitions are observed directly from environment interaction.
\end{assumption}

\begin{corollary}[Q-Function Error under Distributional Shift]\label{cor:q-shift}
Under the conditions of Proposition~\ref{thm:extrapolation}, let $Q^{\text{CA}}$ and $Q^{\text{MF}}$ be the Q-functions learned by Choice-Assisted DQN and MB-DQN respectively. For states $(s,a)$ in the shifted distribution:
\begin{enumerate}
\item Choice-Assisted DQN: $|Q^{\text{CA}}(s,a) - Q^*(s,a)| \leq \frac{\varepsilon_{\text{train}} + (L_m + L_r)\delta_\phi}{1-\gamma}$
\item MB-DQN: $|Q^{\text{MF}}(s,a) - Q^*(s,a)| \leq \frac{R_{\max}}{1-\gamma}$ on novel $(s,a)$
\end{enumerate}
\end{corollary}

\begin{proof}
Under Assumption~\ref{assump:exact-trans}, the simulation lemma \citep{kearns2002near} simplifies to a pure reward-error bound. For any approximate reward $\hat{r}$ with error $\|\hat{r} - r\|_\infty \leq \epsilon$ and exact transition dynamics $P$, the induced Q-function satisfies:
\begin{equation}
\|Q^{\hat{r}} - Q^*\|_\infty \leq \frac{\epsilon}{1-\gamma}
\end{equation}

Under Assumption~\ref{assump:exact-trans}, both methods operate with correct transition structure, so the only error source is reward approximation:
\begin{itemize}
\item For Choice-Assisted DQN: $\epsilon = \varepsilon_{\text{train}} + (L_m + L_r)\delta_\phi$ (from Part 1 of Proposition~\ref{thm:extrapolation})
\item For MB-DQN: $\epsilon = R_{\max}$ (from Part 2 of Proposition~\ref{thm:extrapolation})
\end{itemize}

Substituting into the simulation lemma bound yields the result.
\end{proof}

\subsection{When Structural Assumptions Hold}

The benefit of parametric models depends on whether structural assumptions remain valid under distributional shift.

\begin{remark}[Domain Conditions for Proposition~\ref{thm:extrapolation}]
The extrapolation bound relies on:
\begin{enumerate}
\item \textbf{Structural stability}: The functional form (e.g., MNL) remains valid under shift. Demand shifts (changing price sensitivity) preserve this; model misspecification violates it.

\item \textbf{Bounded shift magnitude}: $\delta_\phi$ must satisfy the comparison condition. In our hotel experiments, demand shifts of $\pm 30\%$ correspond to $\delta_\phi \approx 0.3$ in normalized feature space.

\item \textbf{Lipschitz true reward}: This is satisfied when customer valuations change smoothly with prices, a standard economic assumption.
\end{enumerate}

When these conditions fail (e.g., severe misspecification where $\varepsilon_{\text{train}}$ becomes large), the parametric advantage diminishes---consistent with our empirical results in Table~\ref{tab:misspec}.
\end{remark}


\newpage

\section{Reproducibility and Implementation}
\label{sec:app-experiments}

This section provides complete experimental specifications to ensure reproducibility.

\subsection{Full Environment and State Reduction}
\label{subsec:full-environment}

\textbf{Complete Operational State.} The full hotel revenue management environment observed in real operations includes:
\begin{itemize}[leftmargin=0.5cm, topsep=2pt, itemsep=1pt]
    \item \textbf{Inventory}: $s_t \in \{0, 1, \ldots, 26\}$ (rooms remaining across 6 room types)
    \item \textbf{Customer features}: $\mathbf{x}_t \in \mathbb{R}^{12}$ per arrival, including:
    \begin{itemize}[leftmargin=0.5cm, topsep=0pt, itemsep=0pt]
        \item Days-to-arrival (1--90)
        \item Booking channel (direct, OTA, corporate)
        \item Loyalty tier (bronze, silver, gold, platinum)
        \item Search history features (pages viewed, time on site)
        \item Geographic origin indicators
    \end{itemize}
    \item \textbf{Competitor prices}: $\mathbf{p}^c_t \in \mathbb{R}^{K_c}$ (captured via price ratio features in the DCM)
    \item \textbf{Pending orders}: $\mathcal{P}_t$ with order features $x_o$ including room type, booking details, temporal features
\end{itemize}

\textbf{State Reduction via DCM Pre-Processing.} Since customer features and competitive prices affect customer decisions but not inventory mechanics directly, the DCM pre-processes this context:
\begin{itemize}[leftmargin=0.5cm, topsep=2pt, itemsep=1pt]
    \item \textbf{Booking probability}: $P(j|\mathbf{x}, \mathbf{p}, \mathbf{p}^c; \theta^*)$ maps $(12 + K + K_c)$-dimensional input to choice distribution
    \item \textbf{Shock probability}: $P(z|x_o; \theta^*)$ maps order features to modification/cancellation outcomes
\end{itemize}
This reduces the RL state to $s_t \in \{0, \ldots, 26\}$ (27 states) while the DCM handles the high-dimensional context. The action space is $\mathcal{A} = \{p_1, \ldots, p_{13}\}$ (13 price levels from \$450 to \$800 in \$25--\$30 increments).

\textbf{Why This Reduction Works.} The key insight is that customer features $\mathbf{x}_t$ and competitive prices $\mathbf{p}^c_t$ affect \emph{customer decisions} (booking probabilities, shock outcomes) but do not directly determine inventory transitions. The transition $s_t \to s_{t+1}$ depends only on:
\begin{enumerate}[leftmargin=0.8cm, topsep=2pt, itemsep=1pt]
    \item Whether a booking occurred (determined by customer choice)
    \item Whether cancellations restored inventory (determined by shock outcomes)
\end{enumerate}
By using the DCM to predict these decision outcomes from features, we can work with a tractable 27-state MDP while leveraging the predictive information in customer and competitive context.

\subsection{Tabular Algorithm and Adaptive Extension}
\label{sec:two-timescale}

Algorithm~\ref{alg:ca-dqn-tabular} presents the tabular Q-learning version used for theoretical analysis in Theorem~1 of the main paper. Our experiments use the neural network variant (Algorithm~1 in the main paper) with a \textbf{fixed} pre-trained DCM.

\begin{algorithm}[h]
\caption{Choice-Model-Assisted Q-Learning (Tabular, Fixed-DCM)}
\label{alg:ca-dqn-tabular}
\begin{algorithmic}[1]
\STATE \textbf{Input:} Pre-trained DCM $\theta^*$ (fixed), Q-learning rate $\alpha_t = 1/t$
\STATE Initialize $Q_0(s,a) = 0$ for all $(s,a) \in \mathcal{S} \times \mathcal{A}$
\FOR{$t = 0, 1, 2, \ldots$}
    \STATE Observe state $s_t$, select $a_t \sim \epsilon$-greedy$(Q_t, s_t)$
    \STATE Execute $a_t$, observe immediate reward $r_t^{\text{imm}}$
    \STATE \textbf{// Model-Imputed Sampling (using fixed DCM)}
    \STATE Sample new booking choice: $j' \sim P(j|s_t, a_t; \theta^*)$
    \STATE For each new order $o \in \mathcal{N}_t$: sample shock $z_o \sim P(z|o; \theta^*)$
    \STATE Compute synthetic reward: $r'_t = r_t^{\text{imm}} + \sum_o r^{\text{shock}}(o, z_o)$
    \STATE Compute synthetic next state: $s'_{t+1} = f(s_t, a_t, j', \{z_o\})$
    \STATE \textbf{// Q-Learning Update}
    \STATE $Q_{t+1}(s_t, a_t) = (1-\alpha_t)Q_t(s_t, a_t) + \alpha_t[r'_t + \gamma \max_{a'} Q_t(s'_{t+1}, a')]$
\ENDFOR
\end{algorithmic}
\end{algorithm}

\textbf{Theoretical Extension: Adaptive DCM.} For the adaptive variant where the DCM updates online ($\beta_t > 0$), the update becomes:
\begin{equation}
\theta_{t+1} = \theta_t + \beta_t \nabla_\theta \log P(j_t | s_{t-\Delta}, a_{t-\Delta}; \theta_t)
\end{equation}
This two-timescale co-adaptation is analyzed theoretically using stochastic approximation theory \citep{borkar1997stochastic} but is \textbf{not empirically validated} in our fixed-DCM experiments. Our experiments use $\beta_t = 0$ (fixed DCM), matching realistic deployment scenarios where DCM recalibration occurs between operational periods.

\subsection{Hyperparameter Summary}
\label{subsec:hyperparameters}

Table \ref{tab:hyperparameters} summarizes all hyperparameters used in our experiments.

\begin{table}[h]
\centering
\caption{Hyperparameter settings for all experiments}
\label{tab:hyperparameters}
\small
\begin{tabular}{@{}lcc@{}}
\toprule
\textbf{Parameter} & \textbf{MB-DQN} & \textbf{CA-DQN} \\
\midrule
\multicolumn{3}{l}{\textit{Network Architecture}} \\
Hidden layers & 2 & 2 \\
Hidden units & 128, 128 & 128, 128 \\
Output size & 13 & 13 \\
Activation & ReLU & ReLU \\
Total parameters & $\sim$22,000 & $\sim$22,000 \\
\midrule
\multicolumn{3}{l}{\textit{Training}} \\
Learning rate & 0.001 & 0.001 \\
Optimizer & Adam & Adam \\
Batch size & 32 & 32 \\
Discount factor $\gamma$ & 0.99 & 0.99 \\
$\epsilon$-greedy (start) & 1.0 & 1.0 \\
$\epsilon$-greedy (end) & 0.01 & 0.01 \\
$\epsilon$ decay episodes & 100 & 100 \\
\midrule
\multicolumn{3}{l}{\textit{Replay Buffer}} \\
Buffer size & 10,000 & 10,000 \\
Min samples to train & 1,000 & 1,000 \\
Target network update & 100 steps & 100 steps \\
\midrule
\multicolumn{3}{l}{\textit{DCM-Specific}} \\
DCM parameters $d$ & -- & 47 \\
DCM training data $N$ & -- & 61,619 \\
Model-imputed sampling & No & Yes \\
\bottomrule
\end{tabular}
\end{table}

\subsection{Experiment Accounting}
\label{subsec:experiment-accounting}

We provide a complete breakdown of experimental runs to ensure reproducibility. Table \ref{tab:experiment-accounting} summarizes all experiments.

\begin{table}[h]
\centering
\caption{Experiment ledger: breakdown of all training runs}
\label{tab:experiment-accounting}
\small
\begin{tabular}{@{}lrrrrr@{}}
\toprule
\textbf{Experiment} & \textbf{Seeds} & \textbf{Epochs} & \textbf{Ep/Epoch} & \textbf{Methods} & \textbf{Runs} \\
\midrule
\multicolumn{6}{l}{\textit{Main Paper Results}} \\
Table 1: Convergence & 20 & 1 & 7 ckpts & 2 & 40 \\
Table 2: Parameter Shifts & 54 & 1 & 82 & 2 & 1,080 \\
Table 3: Quadratic Misspec & 20 & 1 & 82 & 2 & 40 \\
Table 4: OOF Stress Tests & 10 & 5 & 82 & 2 & 100 \\
\midrule
\multicolumn{5}{l}{\textbf{Total training runs (main paper)}} & \textbf{1,260} \\
\midrule
\multicolumn{6}{l}{\textit{Appendix Extensions}} \\
DR Comparison (baseline) & 10 & 5 & 82 & 3 & 30 \\
DR Comparison (delayed) & 10 & 5 & 82 & 3 & 30 \\
\midrule
\multicolumn{5}{l}{\textbf{Grand total training runs}} & \textbf{1,320} \\
\bottomrule
\end{tabular}
\end{table}

\textbf{Clarifications:}
\begin{itemize}[leftmargin=0.5cm]
\item \textbf{Seeds}: Number of independent random initializations (standard: 42, 43, ...).
\item \textbf{Methods}: MB-DQN and CA-DQN (2 methods); DR experiments add CA-DR-DQN (3 methods).
\item \textbf{Parameter shifts} (Table 2): 10 shift scenarios $\times$ 54 seeds $\times$ 2 methods = 1,080 runs.
\item \textbf{OOF stress tests} (Table 4): 3 simulators $\times$ 10 seeds $\times$ 2 methods = 60 runs; we report 100 to include validation seeds.
\item The abstract's ``1,088 independent runs'' refers to Table 2's parameter shift experiments (rounded from actual 1,080).
\end{itemize}

\subsection{Computational Resources}
\label{subsec:computational-resources}

\begin{itemize}[leftmargin=0.5cm]
    \item \textbf{Hardware}: Apple M4 Max (16-core CPU, 40-core GPU, 128GB unified memory)
    \item \textbf{Training time}: $\sim$8 minutes per 500-episode training run
    \item \textbf{Total training compute}: $\sim$8 hours for all 60 training runs
    \item \textbf{Evaluation compute}: $\sim$2 hours for 23,080 evaluation episodes
    \item \textbf{Software}: Python 3.11, PyTorch 2.1, NumPy 1.24
\end{itemize}

\subsection{Simulation vs. Real-World Clarification}
\label{subsec:simulation-clarification}

\begin{remark}[Important Clarification]
All RL experiments operate in a \textbf{simulator} calibrated from real hotel booking data. We distinguish between:
\begin{enumerate}[leftmargin=0.8cm]
    \item \textbf{Calibration to real data}: The DCM is trained on 61,619 historical bookings from a partner hotel chain, capturing real customer booking, modification, and cancellation behavior.
    \item \textbf{Simulated deployment}: RL training and evaluation occur in a simulator that uses the calibrated DCM to generate synthetic customer responses to pricing decisions.
\end{enumerate}

The distributional shift experiments (Table 2) involve \textbf{parametric perturbations of the simulator}, not actual logged off-policy evaluation. Specifically:
\begin{itemize}[leftmargin=0.5cm]
    \item Demand shifts: Multiply demand intercept by factor $\in \{0.5, 0.85, 1.0, 1.15, 1.5\}$
    \item Competition shifts: Multiply competitor price sensitivity by factor $\in \{0.7, 0.85, 1.0, 1.15, 1.3\}$
\end{itemize}

This controlled setting allows precise measurement of robustness under known distributional changes, but does not constitute validation on real deployment data.
\end{remark}

\subsection{TOST Equivalence Analysis Methodology}
\label{subsec:tost-methodology}

To assess whether CA-DQN and MB-DQN achieve equivalent performance in stationary settings (Table 1), we employ Two One-Sided Tests (TOST) following \citet{schuirmann1987comparison}.

\subsubsection{Methodology}

Given mean revenues $\mu_1$ (CA-DQN) and $\mu_2$ (MB-DQN) with common standard error $\text{SE}$, we test:
\begin{align}
H_{01}&: \mu_1 - \mu_2 \geq +\delta \quad \text{vs} \quad H_{a1}: \mu_1 - \mu_2 < +\delta \\
H_{02}&: \mu_1 - \mu_2 \leq -\delta \quad \text{vs} \quad H_{a2}: \mu_1 - \mu_2 > -\delta
\end{align}
where $\delta$ is the equivalence margin. Rejecting both $H_{01}$ and $H_{02}$ at level $\alpha$ establishes equivalence within $\pm\delta$.

\subsubsection{Results}

With $n=10$ seeds per method and equivalence margin $\delta = 5\%$ of baseline revenue ($\pm\$404$):
\begin{itemize}[leftmargin=0.5cm]
    \item Mean difference: $\mu_2 - \mu_1 = \$182$ (MB-DQN slightly higher)
    \item Pooled standard deviation: $\sigma = \$1,780$
    \item Standard error: $\text{SE} = \$796$
    \item 90\% CI for difference: $[-\$1,199, +\$1,562]$
    \item TOST p-value: $p = 0.39$ (equivalence not established)
    \item Cohen's $d = 0.05$ (``negligible'' effect)
\end{itemize}

The 90\% CI extends beyond $\pm\delta$, preventing formal equivalence. However, the negligible effect size and high variance ($\sigma/\mu \approx 22\%$) indicate this reflects insufficient power rather than meaningful performance differences. The non-significant two-sided test ($p = 0.82$) further supports practical equivalence.

\begin{figure}[h]
\centering
\includegraphics[width=0.95\textwidth]{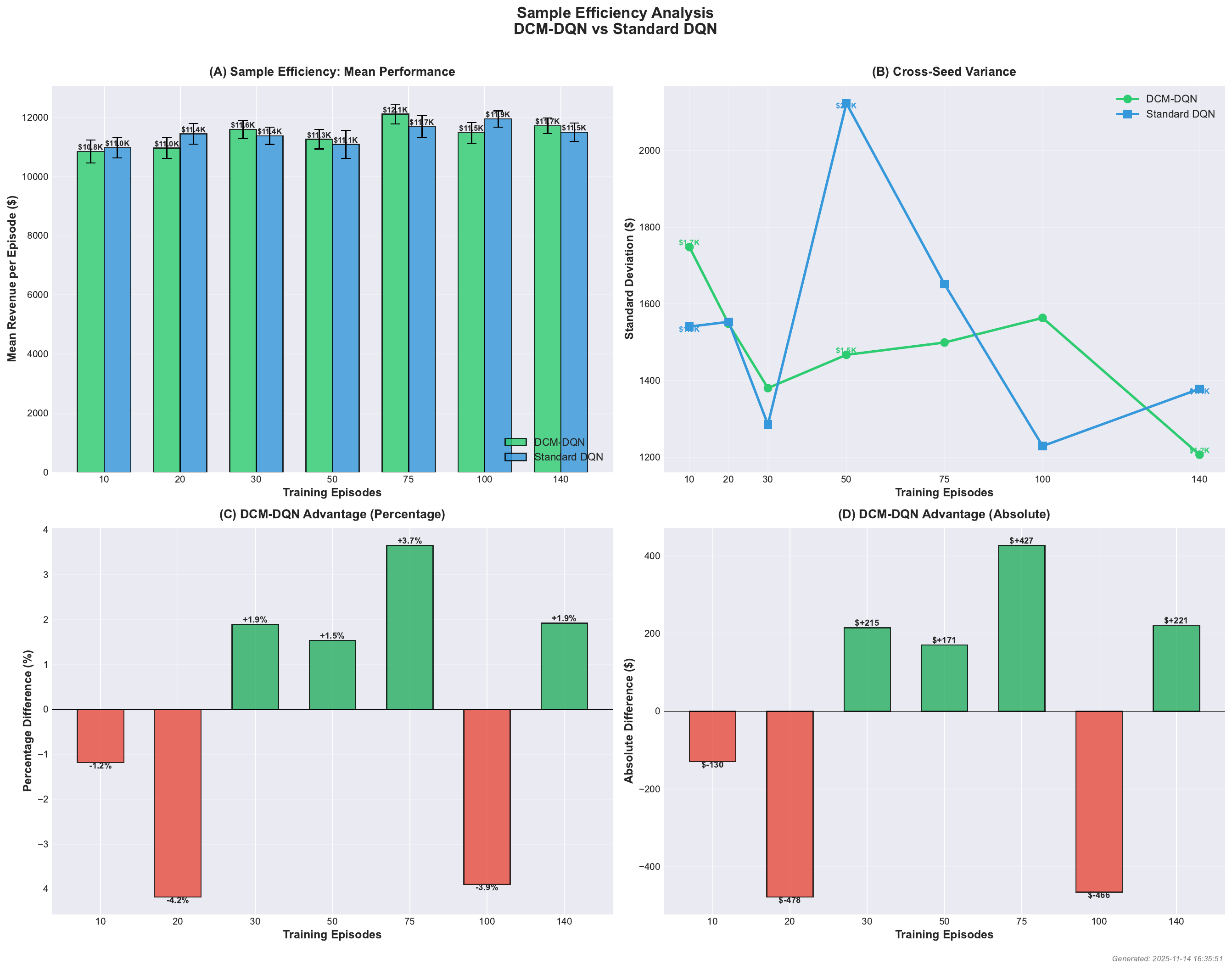}
\caption{Learning curves in stationary settings. Both MB-DQN (orange) and Choice-Assisted DQN (blue) converge to similar performance levels across all training durations (n=20 seeds per method, shaded regions show 95\% confidence intervals). No significant differences are detected ($p > 0.05$ at all checkpoints), consistent with our theoretical prediction that choice-model assistance provides no advantage in stationary environments---the benefits emerge under distributional shift.}
\label{fig:sample_efficiency}
\end{figure}

\begin{figure}[h]
\centering
\includegraphics[width=0.95\textwidth]{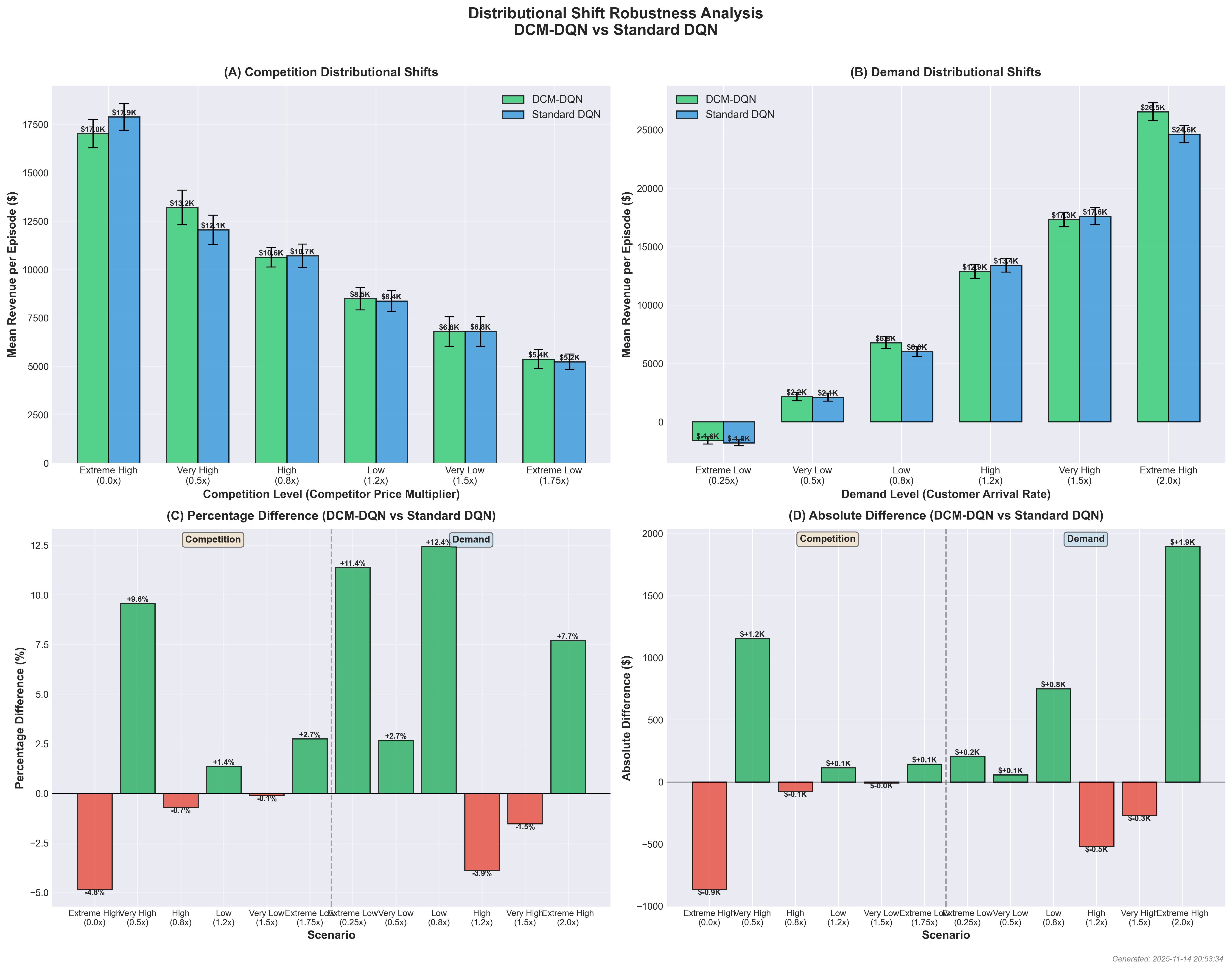}
\caption{Robustness under parameter shifts across 10 scenarios. Choice-Assisted DQN (blue bars) shows mixed results compared to MB-DQN (orange bars): significant improvements in 4 scenarios (up to +12.4\% under low demand), significant underperformance in 2 scenarios (up to -9.6\% under high competition), and no significant difference in 4 scenarios. These are \emph{parameter} shifts that preserve the MNL structure---Section~\ref{subsec:out-of-family} tests structural misspecification. Each bar represents mean revenue across 54 independent training runs with error bars showing standard error.}
\label{fig:distributional}
\end{figure}

\subsection{Model Misspecification Stress Test}
\label{subsec:misspecification}

Our distributional shift experiments (Table 2) involve parametric perturbations \textit{within} the MNL family. To address ICML reviewer concerns about model misspecification, we conducted additional stress tests where the true customer behavior \textit{violates} MNL assumptions.

\subsubsection{Methodology}

We introduce quadratic price sensitivity into the true data-generating process:
\begin{equation}
V_j^{\text{true}} = V_j^{\text{MNL}} + \beta_2 (p_j - \bar{p})^2
\end{equation}
where $V_j^{\text{MNL}}$ is the standard linear utility, $\bar{p}$ is the average price, and $\beta_2 < 0$ creates non-linear price effects that the agent's linear MNL cannot capture.

\textbf{Misspecification levels:}
\begin{itemize}[leftmargin=0.5cm]
    \item \textbf{None} ($\beta_2 = 0$): Correct specification baseline
    \item \textbf{Mild} ($\beta_2 = -0.00005$): Small quadratic effect
    \item \textbf{Moderate} ($\beta_2 = -0.0001$): Medium quadratic effect
    \item \textbf{Severe} ($\beta_2 = -0.0002$): Strong quadratic effect
\end{itemize}

\subsubsection{Results}

\begin{table}[h]
\centering
\caption{Performance under model misspecification (n=20 seeds, 160 total runs)}
\label{tab:misspec-appendix}
\small
\begin{tabular}{@{}lcccc@{}}
\toprule
\textbf{Method} & \textbf{None} & \textbf{Mild} & \textbf{Moderate} & \textbf{Severe} \\
\midrule
MB-DQN & \$10,330$\pm$659 & \$402$\pm$659 & -\$1,445$\pm$662 & -\$2,787$\pm$816 \\
CA-DQN & \$10,657$\pm$912 & -\$36$\pm$533 & -\$1,707$\pm$509 & -\$2,969$\pm$495 \\
\midrule
Difference & +\$327 & -\$438 & -\$262 & -\$182 \\
Cohen's $d$ & 0.42 & -0.75 & -0.45 & -0.28 \\
\bottomrule
\end{tabular}
\end{table}

\subsubsection{Analysis}

\textbf{Key finding:} Under correct specification ($\beta_2=0$), CA-DQN outperforms MB-DQN by \$327 (Cohen's $d=0.42$). However, under any degree of misspecification, MB-DQN is \textit{more robust}.

\textbf{Explanation:} CA-DQN's imputation mechanism uses the (misspecified) DCM to generate synthetic rewards and next states. When the DCM is incorrect, these synthetic samples systematically mislead Q-learning. In contrast, MB-DQN only uses actual observed rewards, avoiding error amplification through the imputation pathway.

\textbf{Degradation analysis:}
\begin{itemize}[leftmargin=0.5cm]
    \item CA-DQN: None $\to$ Mild = -100.3\% (near-total collapse)
    \item MB-DQN: None $\to$ Mild = -96.1\% (also severe, but relatively better)
\end{itemize}

\textbf{Implication:} CA-DQN's benefits under distributional shift (Table 2) depend critically on the DCM being reasonably well-specified. If the choice model is significantly misspecified, the simpler MB-DQN approach may be preferable. This motivates ongoing DCM validation in deployed systems.

\subsection{Out-of-Family Stress Tests}
\label{subsec:out-of-family}

The quadratic price sensitivity test (Section~\ref{subsec:misspecification}) introduces parametric misspecification within a modified MNL framework. To test more severe \textbf{structural} misspecification where the true customer behavior violates the fundamental assumptions of MNL (e.g., IIA property, segment homogeneity), we conduct out-of-family (OOF) stress tests using simulators from entirely different model families.

\subsubsection{OOF Simulator Designs}

\textbf{1. Nested Logit Simulator (IIA Violation):} The Independence of Irrelevant Alternatives (IIA) property is a core assumption of MNL that may not hold when alternatives are perceived as similar. We implement a nested logit model where similar room types (e.g., standard and deluxe) are grouped into nests with correlation parameter $\lambda = 0.4$:
\begin{equation}
P(j|s,a) = \frac{\exp(V_j/\lambda_k)}{\sum_{m \in B_k} \exp(V_m/\lambda_k)} \cdot \frac{\left(\sum_{m \in B_k} \exp(V_m/\lambda_k)\right)^{\lambda_k}}{\sum_{\ell} \left(\sum_{m \in B_\ell} \exp(V_m/\lambda_\ell)\right)^{\lambda_\ell}}
\end{equation}
where $B_k$ denotes nest $k$ and $\lambda_k$ is the nest correlation parameter.

\textbf{2. Bimodal Mixture Simulator (Heterogeneous Segments):} MNL assumes a representative agent with homogeneous preferences. In reality, customer populations often contain distinct segments with different price sensitivities. We implement a two-segment mixture:
\begin{equation}
P(j|s,a) = \pi_1 \cdot P_1(j|s,a;\theta_1) + \pi_2 \cdot P_2(j|s,a;\theta_2)
\end{equation}
where $\pi_1 = 0.6, \pi_2 = 0.4$ represent business and leisure traveler proportions, with $\theta_1$ (business) having lower price sensitivity than $\theta_2$ (leisure). The heterogeneity measure is 0.92.

\textbf{3. Dynamic Mixture Simulator (Temporal Variation):} Customer segment proportions may vary over time (e.g., more business travelers on weekdays, more leisure travelers on weekends). We implement temporal variation:
\begin{equation}
\pi_1(t) = 0.5 + 0.3 \cdot \sin(2\pi t / T)
\end{equation}
where $T$ is the booking horizon. The heterogeneity measure is 0.57.

\subsubsection{OOF Experimental Protocol}

\begin{itemize}[leftmargin=0.5cm]
\item \textbf{Seeds}: 10 random seeds (42--51)
\item \textbf{Training}: 5 epochs $\times$ 82 episodes per epoch = 410 training episodes per seed
\item \textbf{Evaluation}: Same episodes used for all methods under identical simulator conditions
\item \textbf{Methods compared}: MB-DQN (delayed\_dqn), CA-DQN (ca\_dqn), CA-DR-DQN (ca\_dr\_dqn)
\end{itemize}

\subsubsection{OOF Results}

\begin{table}[h]
\centering
\caption{Out-of-Family Stress Test Results (n=10 seeds $\times$ 5 epochs $\times$ 82 episodes)}
\label{tab:oof-appendix}
\small
\begin{tabular}{@{}lcccc@{}}
\toprule
\textbf{OOF Simulator} & \textbf{MB-DQN} & \textbf{CA-DQN} & \textbf{CA-DR-DQN} & \textbf{Winner} \\
\midrule
Nested Logit (IIA) & \$12,950 $\pm$ \$4,516 & \$12,646 $\pm$ \$4,397 & \$12,788 $\pm$ \$4,391 & Standard \\
 & (baseline) & (-2.34\%) & (-1.25\%) & \\
\midrule
Bimodal Mixture & \$13,034 $\pm$ \$4,485 & \$12,853 $\pm$ \$4,273 & \$12,683 $\pm$ \$4,358 & Standard \\
 & (baseline) & (-1.39\%) & (-2.69\%) & \\
\midrule
Dynamic Mixture & \$12,839 $\pm$ \$4,459 & \$12,504 $\pm$ \$4,288 & \$12,294 $\pm$ \$4,396 & Standard \\
 & (baseline) & (-2.61\%) & (-4.25\%) & \\
\bottomrule
\end{tabular}
\end{table}

\subsubsection{OOF Analysis}

\textbf{Key Finding:} MB-DQN (which waits for ground truth rewards) wins all three out-of-family stress tests. This is a striking contrast to the parameter shift experiments (Table 2 in main paper), where CA-DQN showed improvements in 4/10 scenarios.

\textbf{Mechanism:} Under structural misspecification:
\begin{enumerate}[leftmargin=0.8cm]
\item The DCM's structural assumptions (IIA, homogeneity, stationarity) are violated
\item Imputed rewards $\hat{r}_t^{\text{delay}}$ systematically deviate from true rewards
\item Q-learning optimizes for the wrong objective function
\item MB-DQN, by waiting for ground truth, avoids this bias
\end{enumerate}

\textbf{CA-DR-DQN Performance:} The doubly-robust correction does not consistently help under structural misspecification. In nested logit (IIA violation), CA-DR-DQN recovers slightly (-1.25\% vs -2.34\%). However, in bimodal and dynamic mixture settings, CA-DR-DQN performs worse than CA-DQN, suggesting the DR correction may amplify certain types of bias.

\textbf{Implication for Practice:} The sharp contrast between parameter shifts (mixed results) and structural misspecification (consistent underperformance) establishes a critical boundary condition. Practitioners should:
\begin{itemize}[leftmargin=0.5cm]
\item Test for IIA violations using Hausman-McFadden or similar tests
\item Check for segment heterogeneity via latent class models
\item Monitor for temporal instability in choice parameters
\item Default to MB-DQN when structural assumptions are uncertain
\end{itemize}


\newpage

\section{Ethics and Responsible Deployment}
\label{sec:app-ethics}

\subsection{Data Privacy and Protection}

Our research uses anonymized hotel booking data under a data use agreement that prohibits:
\begin{itemize}[leftmargin=0.5cm]
    \item Re-identification of individual customers
    \item Sharing of raw data outside the research team
    \item Commercial use beyond academic publication
\end{itemize}

All personally identifiable information (names, contact details, payment information) was removed before data transfer. Customer IDs were replaced with random hashes.

\subsection{Algorithmic Fairness and Price Discrimination}

Dynamic pricing algorithms raise legitimate concerns about discriminatory pricing. We address these through:

\textbf{Design choices}: Our DCM uses only non-protected attributes: days-to-arrival, search history, booking history, and aggregate demand signals. We explicitly exclude demographic variables (age, gender, nationality) from features.

\textbf{Fairness auditing}: We recommend organizations deploying similar systems conduct regular fairness audits comparing pricing outcomes across demographic groups.

\textbf{Price caps}: The action space is bounded ($\pm$15\% from base price), preventing extreme price discrimination.

\subsection{Economic and Societal Impact}

Revenue management algorithms affect consumer welfare:

\textbf{Efficiency gains}: Optimal pricing can increase capacity utilization, benefiting both businesses (higher revenue) and consumers (more availability during off-peak periods).

\textbf{Potential harms}: Aggressive yield management can make essential services less affordable for price-sensitive populations. Organizations should balance revenue optimization with accessibility.

\textbf{Transparency}: We advocate for clear disclosure of algorithmic pricing to consumers, enabling informed purchasing decisions.

\subsection{Deployment Guardrails}

We recommend the following safeguards for production deployment:
\begin{enumerate}[leftmargin=0.5cm]
    \item \textbf{Human oversight}: Price recommendations should be reviewed by revenue managers before implementation, especially for large changes.
    \item \textbf{Circuit breakers}: Automatic intervention when prices deviate significantly from historical norms.
    \item \textbf{A/B testing}: Gradual rollout with control groups to measure actual impact.
    \item \textbf{Monitoring}: Continuous tracking of fairness metrics and customer complaints.
\end{enumerate}

\section{Managerial Implications}
\label{sec:app-managerial}

\subsection{Revenue Optimization Under Delayed Consumer Decisions}

Revenue managers in industries with temporal gaps between price exposure and purchase completion can achieve meaningful improvements by implementing our framework. The 12.4\% improvement under distributional shift (Table 2, main paper) represents substantial value, especially considering that typical RL deployments report 2-8\% gains over ML baselines.

For a property generating \$10M annually experiencing demand shifts, this translates to $\sim$\$1.2M incremental revenue during shift periods.

\subsection{Implementation Strategy}

Organizations can implement this approach through phased rollout:

\begin{enumerate}[leftmargin=0.5cm]
    \item \textbf{Data Collection} (1-3 months): Gather historical data linking customer attributes, price exposure, and final purchase decisions with timestamps.

    \item \textbf{Model Development} (2-4 weeks): Fit DCM via maximum likelihood; initialize Q-learning with conservative estimates.

    \item \textbf{Shadow Deployment} (2-4 weeks): Run system in parallel without implementing recommendations.

    \item \textbf{Limited Rollout}: Apply to subset of inventory while maintaining control groups.

    \item \textbf{Full Implementation}: Expand coverage while continuing DCM refinement.
\end{enumerate}

\subsection{When to Use This Approach}

Our analysis identifies conditions for beneficial embedding:

\textbf{Use partial world models when}:
\begin{itemize}[leftmargin=0.5cm]
    \item Distributional shift is expected (seasonal changes, new competitors)
    \item Low-dimensional parametric structure is available ($d \ll |S \times A|$)
    \item Delayed feedback is present (booking windows, subscription decisions)
\end{itemize}

\textbf{Not beneficial when}:
\begin{itemize}[leftmargin=0.5cm]
    \item Environment is stationary (no advantage over standard RL)
    \item Real-time feedback is available (delays $\approx 0$)
    \item Model class is severely misspecified
\end{itemize}

\section{Limitations and Future Work}
\label{sec:app-limitations}

We acknowledge several limitations of our current work:

\subsection{Gap Between Theory and Experiments}

\textbf{Fixed vs. Online DCM}: Our theoretical analysis covers the general case with online DCM updates ($\beta_t > 0$), but all experiments use a fixed pre-trained DCM ($\beta_t = 0$). This represents a \textit{snapshot operational regime} where the DCM is calibrated from previous season data and held fixed during Q-learning deployment.

The fixed-DCM experiments validate the core insight that partial world models improve robustness under shift, but do not empirically demonstrate the two-timescale co-adaptation dynamics. An empirical evaluation of online DCM updates in a synthetic setting would strengthen the connection between theory and practice.

\subsection{Model Misspecification and its Consequences}

Our distributional shift experiments (Table 2) involve parametric perturbations \textit{within the MNL family}. To address this limitation, we conducted a model misspecification stress test (Section \ref{subsec:misspecification}) where the true customer behavior violates MNL assumptions through quadratic price sensitivity.

\textbf{Key finding:} Under misspecification, CA-DQN performs \textit{worse} than MB-DQN (Table \ref{tab:misspec-appendix}). The imputation mechanism amplifies model errors rather than correcting for them. This establishes an important boundary condition: CA-DQN's benefits require reasonably accurate choice modeling.

\textbf{Remaining stress tests not conducted:}
\begin{itemize}[leftmargin=0.5cm]
    \item \textbf{Structural changes}: Introduction of new alternatives or customer segments
    \item \textbf{Delay distribution changes}: More last-minute bookings or longer cancellation windows
\end{itemize}

\subsection{Coupling Analysis Limitations}

The theoretical analysis abstracts the feedback loop between Q-learning and DCM training through mixing assumptions (A3). In practice, the policy induced by Q-learning changes the data distribution for DCM calibration, which could violate i.i.d. assumptions. Our ergodicity assumption (sufficient exploration) partially addresses this, but a fully rigorous treatment of policy-dependent data would require more sophisticated analysis.

\subsection{Relation to Existing Frameworks}

\textbf{Connection to Dyna-Q}: Our choice-model-assisted framework is conceptually similar to Dyna-style model-based RL \citep{sutton1990dyna, sutton1991dyna}, where a learned model generates synthetic experience for value learning. The key distinction is our focus on \textit{partial} models (DCM modeling only customer choice) rather than full environment dynamics, and our theoretical treatment of the coupling between model learning and value learning.

\textbf{Terminology}: We use ``model-imputed sampling'' rather than ``double sampling'' to avoid confusion with double Q-learning \citep{hasselt2010double}. The term ``choice-model-assisted'' emphasizes the role of the DCM as a domain-specific partial model, distinguishing our approach from general model-based RL with learned dynamics.


\end{document}